\documentclass[twoside]{article}

\usepackage[accepted]{aistats2026}

\usepackage[round]{natbib}
\usepackage{titletoc}

\usepackage[utf8]{inputenc} 
\usepackage[T1]{fontenc}    
\usepackage[
    linkcolor=blue,     
    urlcolor=cyan       
]{hyperref}
\usepackage{url}            
\usepackage{booktabs}       
\usepackage{amsfonts}       
\usepackage{amsmath}
\usepackage{amssymb}
\usepackage{nicefrac}       
\usepackage{microtype}      
\usepackage{xcolor}         
\usepackage{mathtools}
\usepackage{dsfont}
\usepackage{amsthm}
\usepackage{pifont}

\usepackage[ruled,linesnumbered]{algorithm2e}

\usepackage{hyperref}
\usepackage{cleveref}

\usepackage{graphicx}
\usepackage{microtype}

\newtheorem{theorem}{Theorem}[section]
\newtheorem{assumption}{Assumption}[section]
\newtheorem{proposition}[theorem]{Proposition}
\newtheorem{lemma}[theorem]{Lemma}
\newtheorem*{lemma*}{Lemma}

\newtheorem{definition}[theorem]{Definition}

\crefname{theorem}{theorem}{theorems}
\Crefname{theorem}{Theorem}{Theorems}

\crefname{assumption}{assumption}{assumptions}
\Crefname{assumption}{Assumption}{Assumptions}

\crefname{proposition}{proposition}{propositions}
\Crefname{proposition}{Proposition}{Propositions}

\crefname{lemma}{lemma}{lemmas} 
\Crefname{lemma}{Lemma}{Lemmas} 

\crefname{corollary}{corollary}{corollaries}
\Crefname{corollary}{Corollary}{Corollaries}

\crefname{definition}{definition}{definitions}
\Crefname{definition}{Definition}{Definitions}

\newcommand{\cC}{\ensuremath{\mathcal{C}}}

\newcommand{\cG}{\ensuremath{\mathcal{G}}}

\newcommand{\cH}{\ensuremath{\mathcal{H}}}
\newcommand{\cK}{\ensuremath{\mathcal{K}}}

\newcommand{\cX}{\ensuremath{\mathcal{X}}}

\newcommand{\cY}{\ensuremath{\mathcal{Y}}}

\newcommand{\bA}{\ensuremath{\mathbf{A}}}
\newcommand{\bB}{\ensuremath{\mathbf{B}}}

\newcommand{\bG}{\ensuremath{\mathbf{G}}}

\newcommand{\bK}{\ensuremath{\mathbf{K}}}

\newcommand{\bV}{\ensuremath{\mathbf{V}}}

\newcommand{\bX}{\ensuremath{\mathbf{X}}}
\newcommand{\bY}{\ensuremath{\mathbf{Y}}}
\newcommand{\bZ}{\ensuremath{\mathbf{Z}}}

\renewcommand{\epsilon}{\varepsilon}

\newcommand{\expected}[2]{\mathbb{E}_{#1}\left[ #2 \right]}
\newcommand{\MMD}{\ensuremath{\text{MMD}}}

\newcommand{\Fuse}{\text{Fuse}} 
\SetKwInput{KwIn}{Input}
\SetKwInput{KwOut}{Output}
\SetKwComment{Comment}{\tcp*}{}

\usepackage{amsmath,amsfonts,bm}

\def\eqref#1{equation~\ref{#1}}

\def\1{\bm{1}}

\DeclareMathAlphabet{\mathsfit}{\encodingdefault}{\sfdefault}{m}{sl}
\SetMathAlphabet{\mathsfit}{bold}{\encodingdefault}{\sfdefault}{bx}{n}

\DeclareMathOperator*{\argmax}{arg\,max}

\DeclareMathOperator{\sign}{sign}

\usepackage{multirow}

\begin{document}

%

%
\runningauthor{M\'onica Ribero, Antonin Schrab, Arthur Gretton}

\twocolumn[

\aistatstitle{Regularized $f$-Divergence Kernel Tests}

\aistatsauthor{  M\'onica Ribero$^*$  \And  Antonin Schrab$^*$ \And Arthur Gretton }

\aistatsaddress{   Google Research \And University of Cambridge \And Google DeepMind} ]


\begin{abstract}
We propose a framework to construct practical kernel-based two-sample tests from the family of $f$-divergences. The test statistic is computed from the witness function of a regularized variational representation of the divergence, which we estimate using kernel methods. The proposed test is adaptive over hyperparameters such as the kernel bandwidth and the regularization parameter. We provide theoretical guarantees for statistical test power across our family of $f$-divergence estimates. While our test covers a variety of $f$-divergences, we bring particular focus to the Hockey-Stick divergence, motivated by its applications to differential privacy auditing and machine unlearning evaluation. For two-sample testing, experiments demonstrate that different $f$-divergences are sensitive to different localized differences, illustrating the importance of leveraging diverse statistics. For machine unlearning, we propose a relative test that distinguishes true unlearning failures from safe distributional variations.
\end{abstract}

\section{Introduction}
\label{sec:intro}

Two-sample testing is a fundamental task in statistics, with broad applications in machine learning (e.g., MMD, \citealp{GBRS+12}). 
These tests can confidently determine if two sets of observations come from different underlying distributions, applicable to diverse contexts, from auditing the privacy of machine learning models \citep{KMRS24} and verifying the efficacy of machine unlearning techniques \citep{TKP+24}, to validating models in the physical sciences \citep{2019learningNewPhysics, grosso2025multiple}.

Two-sample tests require the choice of a distance or divergence to quantify the difference between the distributions, and different applications require different notions of distance. Some tests are concerned with smooth variations, for which the Maximum Mean Discrepancy (MMD, \citealp{GBRS+12}) is a suitable choice \citep{SKA+23}. Others, like outlier detection in physical models, require sensitivity to subtle and/or localized differences that MMD might miss. Other applications, like privacy or unlearning, require divergences parameterized by a budget $\varepsilon$ to define an acceptable degree of separation. This highlights the challenge that no single two-sample test is universally optimal or suitable across all scenarios.

\paragraph{Contributions. } We introduce a general framework for constructing two-sample tests from the family of $f$-divergences, unifying specific instances that have been studied by \citet{NWJ10, eric2007testing, GAG21, grosso2025multiple}. This framework builds on the variational formulation for $f$-divergences. A key challenge is that  the optimizer of this problem, also known as the witness function, is often non-smooth and hard to learn efficiently. However, the witness function can always be expressed as a functional of the likelihood ratio. 
We leverage kernel-based methods to estimate the ratio and $\ell_2$-regularization, trading off some accuracy for tractability (\Cref{sec:consistent-estimation}). 
While more specialized estimators may exist in the literature for certain pairs of distributions and divergences, the novel proposed framework provides a general and practical approach for constructing a test applicable to any $f$-divergence. \looseness=-1

On the technical side, we  instantiate a likelihood ratio estimator proposed by \cite{chen2024regularized}. We prove in \Cref{res:witness_convergence} that, under smoothness assumptions (i.e., $\mu_P-\mu_Q\in\mathrm{Ran}(\Sigma_Q^\theta)$ for some $\theta\in (1,2]$ (see \Cref{subsec:kernel}) the estimator converges to the true ratio at a rate of 
$N^{-\frac{\theta-1}{2(\theta+1)}}$
with high probability, where $N$ represents  the minimum number of samples from distributions $P$ and $Q$. Subsequently, we leverage the variational formulation of $f$-divergences, expressing the corresponding witness function as a functional of the likelihood ratio. This yields an  $f$-divergence estimator which, as proved in \Cref{res:fdiv_convergence}, converges to the true divergence value at a rate of 
$
N^{-\frac{\theta-1}{2(\theta+1)}}
+ \widetilde{N}^{-1/2}
$
with high probability,
where the data has been split between $N$ samples (used for estimating the ratio) and $\widetilde N$ samples (used for estimating the $f$-divergence given the estimated ratio). 

We analyze the properties of the permutation hypothesis tests using these likelihood-based $f$-divergence statistics. 
First, these naturally control the type I error at the desired level $\alpha\in(0,1)$ for any sample size \citep[Lemma 1]{romano2005exact}.
Then, we prove that in \Cref{res:asymptotic_power} these tests are consistent in that they are asymptotically powerful, that is, their type II error probability converges to 0 as the sample size tends to infinity. 
Moreover, for the finite-sample regime, we establish a non-asymptotic power guarantee: the test controls the type II error by $\beta\in(0,1)$ for any pair of distributions separated with respect to the $f$-divergence $D_f$ as
\looseness=-1
$$
D_{f}(P \| Q) 
\gtrsim
\left(
N^{-\frac{\theta-1}{2(\theta+1)}}
+ \widetilde{N}^{-\frac{1}{2}}
\right)
\sqrt{\ln\!\bigg(\frac{1}{\min\{\alpha,\beta\}}\bigg)}.
$$
This reinforces the idea that tests based on different $f$-divergences capture different types of departures from the null.

Our experimental results corroborate that no single test consistently outperforms the others, across a variety of settings. However, we find that in most cases, at least one of the considered tests exhibits strong performance. We show that the recently developed aggregation method of \citet{SKA+23} can aggregate these different $f$-divergence tests into a single one, which we call $f$-Agg, and which can powefully detect a wide range of alternatives (see details in \Cref{alg:meta-test} in \Cref{sec:adaptive}).
Our results provide valuable intuition on which statistic is best suited for a given task. Furthermore, we propose a new, more precise evaluation framework for machine unlearning based on our theoretical and empirical findings, detailed below.

Of specific interest is the test based on the Hockey-Stick divergence, motivated by its relevance to privacy and unlearning. We demonstrate how the Hockey-Stick divergence can be used to audit pure differential privacy and, more broadly, to evaluate machine unlearning algorithms, a task for which finding appropriate evaluation metrics has been a significant challenge. We then address a flaw in current unlearning definitions: absolute distance tests, comparing an unlearned model distribution to a retrained one without the removed data, are insufficient. Such tests can fail even if unlearning was successful, as they cannot distinguish between a genuine failure to forget and benign variations introduced by the unlearning algorithm. Further, recent work by \cite{yu2025impossibility} theoretically and empirically shows that equivalence to a retrained model is impossible for gradient-based unlearning algorithms that are oblivious to the learning trajectory. We resolve this ambiguity by proposing a  relative distance test, which measures whether the unlearned model is distributionally closer to a retrained model or to the original, compromised one. \looseness=-1

\paragraph{Related Work.}
A popular approach to non-parametric two-sample testing relies on the MMD \citep{GBRS+12}, which measures distances between distributions in a Reproducing Kernel Hilbert Space (RKHS), see \Cref{subsec:mmd} for details. The power of MMD tests is highly sensitive to the choice of kernel and bandwidth, which has motivated various lines of work including learning the kernel \citep{sutherland2016generative,liu2020learning}, constructing adaptive tests by combining multiple kernels' results \citep{SKA+23, BSG23}, and leveraging AutoML methods \citep{kubler2022automl}. Separately, $f$-divergences are widely used in machine learning and have been recently used to train machine learning models via the variational forms of these divergences \citep{NCT16, mescheder2017adversarial, arbel2020generalized}, as well as for model evaluation \citep{pillutla2023mauve}, and to define gradient flows on probability measures \citep{GAG21, CMGK+24}. 
Application-specific two-sample tests have also been developed and tested, such as the tests based on the Hockey-Stick divergence for privacy \citep{KMRS24} and for unlearning \citep{TKP+24}, as well as deep-learning and KL-based tests for new physics \citep{2019learningNewPhysics, grosso2025multiple}.

The estimation of $f$-divergences is a well-studied problem with several established approaches. One popular family of non-parametric techniques relies on nearest neighbor methods \citep{wang2009divergence, singh2016analysis}. Other non-parametric approaches use space partitioning estimates, however, these tend to perform poorly in high dimensions \citep{BG05}. A separate line of work, relevant to ours, focuses on estimating the likelihood ratio directly through its variational formulation. Within this framework, estimators have been proposed over various function classes, including kernel-based classes using convex relaxations \citep{NWJ10}, generalized linear models approximations \cite{sugiyama2008direct,kanamori2009least,kanamori2012statistical} or direct optimization \citep{chen2024regularized}, and neural network classes \citep{NCT16}. \looseness=-1

This work unifies these threads building upon the variational representation of $f$-divergences but adapts it for hypothesis testing, developing and analyzing a practical framework for constructing a  family of tractable two-sample tests with applications to diverse settings.

The paper is organized as follows. In \Cref{sec:consistent-estimation} we introduce a method for consistently estimating $f$-divergences. We then provide asymptotic and non-asymptotic power guarantees for the permutation tests based on these estimates in \Cref{sec:testfdiv}.  In \Cref{sec:statistics}, we instantiate this framework with the Hockey-Stick divergence. Finally, we present numerical results in \Cref{sec:experiments}. \looseness=-1

\paragraph{Notation.}
We write $a\lesssim b$ if there exists some constant $C>0$ (in our setting independent of the sample sizes) such that $a\leq C b$. We write $a\asymp b$ if $a\lesssim b$ and $a \gtrsim b$.
For a distribution $Q$, we let $\mu_Q$ denote the kernel mean embedding, and let $\Sigma_Q$ denote the covariance operator (see details in \Cref{subsec:kernel}). We use bold notation to denote vectors of random variables, e.g. $\bX=(X_1, ..., X_n)$. 

\section{Estimating Any $f$-Divergence Consistently}
\label{sec:consistent-estimation}
For probability distributions which are absolutely continuous \(P \ll Q\) on a space $\mathcal Z$,
an $f$-divergence \citep{csiszar1967information} is defined as
\begin{equation}
    \label{def:direct-f-divergence}
D_f(P \;\|\; Q) 
= \int_{\mathcal{Z}} f\!\left(\frac{dP}{dQ}\right) \, dQ,
\end{equation}
for some convex function \(f \colon (0,\infty) \to \mathbb{R}\) which satisfies \(f(1) = 0\), and is lower semi-continuous implying $f(0) = \lim_{t\to0^{+}}f(t)$.
The variational representation (See, e.g., Theorem 7.26 in \citealp{PW22}) of the $f$-divergence is
\begin{equation}\label{eq:variational_f_divergence}
D_f(P \;\|\; Q) 
= \sup_{g:\mathcal{Z}\to\mathbb{R}} 
\mathbb{E}_P[g(X)] - \mathbb{E}_Q[f^*(g(Y))] ,
\end{equation}
where \(f^*(u) = \sup_{t\in\mathbb R} \{ ut - f(t) \}\) is the convex conjugate of \(f\), and where the supremum is taken over functions which output values for which $f^*$ is well-defined and finite.
The witness function $g^\star$ is, if it exists, the function at which the supremum is attained.
As can easily be derived \citep[Lemma 1]{NWJ10},
the witness function is equal to\footnote{With some care needed when $f$ is not differentiable everywhere (e.g., Total-Variation, Hockey-Stick).} 
$$
g^\star = f'\!\left(\frac{dP}{dQ}\right)
$$
where $\frac{dP}{dQ}$ denotes the Radon-Nikodym derivative. 
Therefore, the $f$-divergence is equal to\begin{align*}
&D_f(P \;\|\; Q) 
= 
\mathbb{E}_P[g^\star(X)] - \mathbb{E}_Q[f^*(g^\star(Y))]\\
&= 
\mathbb{E}_P\left[f'\!\left(\frac{dP}{dQ}(X)\right)\right] - \mathbb{E}_Q\left[f^*\!\left(f'\!\left(\frac{dP}{dQ}(Y)\right)\right)\right].
\end{align*}
Given samples 
$X_1,\dots,X_m,\widetilde{X}_1\dots,\widetilde{X}_{\widetilde m}$ i.i.d.\,\,from $P$, 
and samples
$Y_1,\dots,Y_n,\widetilde{Y}_1,\dots,\widetilde{Y}_{\widetilde n}$ i.i.d.\,\,from $Q$,
we use $X_1,\dots,X_m$ and $Y_1,\dots,Y_n$ to construct an estimator $\widehat{r}_\lambda(\cdot)$ of $\frac{dP}{dQ}(\cdot)$ (see \Cref{eq:r_lambda}), and the samples $\widetilde{X}_{1},\dots,\widetilde{X}_{\widetilde m}$ and $\widetilde{Y}_{1},\dots,\widetilde{Y}_{\widetilde n}$ are used to estimate the expectations as
\begin{equation}
\label{eq:Df_estimator}
\widehat{D}_{f,\lambda} =
\frac{1}{\widetilde m}\sum_{i=1}^{\widetilde m} f'\!\big(\widehat{r}_\lambda(\widetilde{X}_i)\big) - \frac{1}{\widetilde n}\sum_{j=1}^{\widetilde n} f^*\!\big(f'\!\big(\widehat{r}_\lambda(\widetilde{Y}_j)\big)\big).
\end{equation}
Hence, with \Cref{eq:Df_estimator}, we are able to construct estimators for any $f$-divergence.

We leverage the variational formulation, instead of the direct estimator in \Cref{def:direct-f-divergence} for its flexibility in handling regularized $f$-divergences (see \Cref{sec:regularized-f-divergences}) but also because of its higher empirical performance in detection power (see \Cref{sec:direct-vs-variational}).

There are various ways of estimating $\frac{dP}{dQ} $ \citep{NWJ10,chen2024regularized,NCT16}. We focus on the method of \citet[Proposition 6.1]{chen2024regularized} which provides a kernel-based closed form expression for $\widehat{r}_\lambda$,  a $\lambda$-regularized version of $\frac{dP}{dQ}$, using a kernel $k\colon \mathcal{Z}\times\mathcal{Z}\to\mathbb{R}$ in an RKHS $\mathcal{H}$.
For any $u\in\mathcal Z$ and $\lambda>0$, it is defined as
\begin{align}
\label{eq:r_lambda}
&\widehat{r}_\lambda(u)
=~
\frac{1}{n\lambda}k_{u\mathbf{Y}}\mathds{1}_n
-
\frac{1}{m\lambda}k_{u\mathbf{X}}\mathds{1}_m\\
&-
\frac{1}{n\lambda}k_{u\mathbf{X}}L_{\mathbf{X}\mathbf{X},\lambda}^{-1}k_{\mathbf{X}\mathbf{Y}}\mathds{1}_n
+
\frac{1}{m\lambda}k_{u\mathbf{X}}L_{\mathbf{X}\mathbf{X},\lambda}^{-1}k_{\mathbf{X}\mathbf{X}}\mathds{1}_m
+1\nonumber
\end{align}
where we write
$k_{u\mathbf{X}} = \big(k(u,X_{i})\big)_{i=1}^{m} \in\mathbb{R}^{1\times m}$,
$k_{u\mathbf{Y}} = \big(k(u,Y_{i})\big)_{i=1}^{n} \in\mathbb{R}^{1\times n}$, 
$k_{\mathbf{X}\mathbf{X}} = \big(k(X_{i},X_{j})\big)_{1\leq i,j \leq m} \in\mathbb{R}^{m\times m}$,
$k_{\mathbf{X}\mathbf{Y}} = \big(k(X_{i},Y_{j})\big)_{1\leq i \leq m, 1\leq j \leq n} \in\mathbb{R}^{m\times n}$,
and
$L_{\mathbf{X}\mathbf{X},\lambda}= m\lambda I + k_{\mathbf{X}\mathbf{X}}$.
Interestingly, the quantity $2(\widehat{r}_\lambda-1)$ is shown by \cite{chen2024regularized} to be the witness function for DrMMD, an MMD with a perturbed (deregularized) kernel. 

First, we prove that the estimator of \Cref{eq:r_lambda} indeed tends to $\frac{dP}{dQ}$ and derive a rate of convergence.

\begin{assumption}
\label{assump1}
The sample sizes are balanced in the sense that $m\asymp n$ and $\widetilde{m}\asymp \widetilde{n}$.
The kernel is bounded by $K$ everywhere. 
We have $\mu_P-\mu_Q\in\mathrm{Ran}(\Sigma_Q^\theta)$ for some $\theta\in (1,2]$.
Let $\eta\in(0,e^{-1})$, $N=\min(m,n)$, $\widetilde{N}=\min(\widetilde{m},\widetilde n)$ and $\lambda_{N\!,\theta} = N^{-1/2(\theta+1)}$.
\end{assumption}

\begin{lemma}[Witness convergence]
\label{res:witness_convergence}
Under \Cref{assump1}, it holds with probability at least $1-\eta$ that
\begin{equation}
\label{eq:witness_convergence}
\left\|
\widehat{r}_{\lambda_{N\!,\theta}}
-
\frac{dP}{dQ}
\right\|_{\mathcal H}
\leq
\mathsf{C}_1
N^{-\frac{\theta-1}{2(\theta+1)}}
\sqrt{\ln(1/\eta)}
\end{equation}
for some universal constant $\mathsf{C}_1>0$.
\end{lemma}

Similarly, under permutation of the samples, we show in \Cref{lastprop} in \Cref{proof3} that for some $\mathsf{C}_0>0$ we have
$\big\|\widehat{r}_{\lambda_{N\!,\theta}}-1\big\|_{\mathcal H} \leq
\mathsf{C}_0
N^{-\frac{\theta-1}{2(\theta+1)}} 
\sqrt{\ln(1/\eta)}
$ holding with probability at least $1-\eta$.
We define 
\begin{equation}
\label{C}
\mathsf{C}=\max\{\mathsf{C}_0,\mathsf{C}_1\}. 
\end{equation}

The result of \Cref{res:witness_convergence}  then allows us to prove that the estimator of \Cref{eq:Df_estimator} is consistent in that it converges to the true $f$-divergence ${D}_f$ as $m,\widetilde m,n,\widetilde n$ all tend to infinity.
Before presenting this result, we first introduce some additional necessary assumptions.

\begin{assumption}
\label{assump2}
For $f$-divergences with $f$ twice continuously differentiable on $(0,\infty)$, assume that $\frac{dP}{dQ}$ belongs to $[c,C]$ for some $0<c<1<C<\infty$, and that $\mathsf{C}N^{-\frac{\theta-1}{2(\theta+1)}} \!\sqrt{\ln(1/\eta)} \leq c/2$.
For the Total-Variation ($\gamma=1$) and Hockey-Stick ($\gamma>0$) $f$-divergences, for $X\sim P$ and $Y\sim Q$, assume that the densities of $\widehat{r}_{\lambda_{N\!,\theta}}(X)$ and $\widehat{r}_{\lambda_{N\!,\theta}}(Y)$ exist and are bounded on $[\gamma/2,3\gamma/2]$, and that $\mathsf{C}N^{-\frac{\theta-1}{2(\theta+1)}} \!\sqrt{\ln(1/\eta)} \leq \gamma/2$.
\end{assumption}

We remark that it is possible to weaken the assumption, imposing instead that $\frac{dP}{dQ}$ belongs to $[c_{\eta},C_{\eta}]$ only with probability at least $1-\eta/2$ under both $P$ and $Q$. This allows for more a wider class of alternatives to be considered, however, it comes at the cost of losing the explicit dependence on $\eta,\alpha,\beta$ in the rates (\Cref{res:fdiv_convergence} and \Cref{res:non_asymptotic_power}).

\begin{lemma}[$f$-Divergence convergence]
\label{res:fdiv_convergence}
Under \Cref{assump1,assump2}, with probability at least $1-\eta$, it holds that
$$
\big|
\widehat{D}_{f,\lambda_{N\!,\theta}}-D_f
\big|
\lesssim 
\left(
N^{-\frac{\theta-1}{2(\theta+1)}}
+ \widetilde{N}^{-1/2}
\right)
\sqrt{\ln(1/\eta)}.
$$
\end{lemma}
Hence, via \Cref{eq:Df_estimator}, we are able to construct a consistent estimator for any $f$-divergence.
We stress that the conditions of \Cref{res:fdiv_convergence} cover all commonly-used $f$-divergence such as all the ones presented in \Cref{tab:fdiv}. \looseness=-1

\section{Testing Using Any $f$-Divergence With High Power}
\label{sec:testfdiv}

Given i.i.d.~samples from $P$ and $Q$, the goal of two-sample testing is to determine whether the data provides sufficient evidence to reject the null hypothesis that $H_0\colon P=Q$.

A hypothesis test, which controls the probability of type I error (rejecting the null under $H_0$) by $\alpha\in (0,1)$, can easily be constructed using any statistic by relying on the permutation method (\citealp{romano2005exact}, see also \Cref{subsec:permutation}).
As such, we can construct a permutation-test using the estimator of \Cref{eq:Df_estimator} for any $f$-divergence. That is, the test rejects the null hypothesis when
\begin{equation}
\label{eq:permutation_test}
    \widehat{D}_{f,\lambda_{N\!,\theta}} > \widehat{q}_{1-\alpha}
\end{equation}
where $\widehat{q}_{1-\alpha}$ is a permutation quantile as explained in \Cref{subsec:permutation}.

We prove that the resulting test is consistent, in the sense that, its probability of type II error (failing to reject the null when $H_0$ does not hold) converges to zero as the sample sizes tend to infinity.
Equivalently, we say that the test power (i.e., one minus the type II error) converges to 1.
We emphasize that this is not a simple implication of \Cref{res:fdiv_convergence} as a thorough analysis of the behavior of the test statistic under permutation is also required (see \Cref{lastprop} in \Cref{proof3}).

\begin{theorem}[Asymptotic Power]
\label{res:asymptotic_power}
Under \Cref{assump1,assump2}, the test defined in \Cref{eq:permutation_test} is consistent in that its power converges to 1 as the sample sizes tend to infinity.
For any fixed $P\neq Q$, we have
$$
\mathbb{P}(\mathrm{reject }~ H_0)\to 1
\textrm{ as } N,\widetilde{N}\to\infty.
$$
\end{theorem}
This asymptotic power guarantee is a desired property for any hypothesis test.
Nonetheless, a much more challenging result to obtain is to characterize a set of alternatives which can be detected with high probability at any fixed sample size.
\begin{theorem}[Non-asymptotic Power]
\label{res:non_asymptotic_power}
Under \Cref{assump1,assump2}, the test defined in \Cref{eq:permutation_test} with level $\alpha$ can achieve power at least $1-\beta$ for any distributions $P$ and $Q$ satisfying
$$
D_{f}(P \| Q) 
\gtrsim
\left(
N^{-\frac{\theta-1}{2(\theta+1)}}
+ \widetilde{N}^{-\frac{1}{2}}
\right)
\sqrt{\ln\!\bigg(\frac{1}{\min\{\alpha,\beta\}}\bigg)}.
$$
\end{theorem}
We stress that the same power guarantees also hold for the permutation test based on the regularized $\chi^2$-divergence (DrMMD, \citealp{chen2024regularized}), as presented in \Cref{drmmd_power}.
\Cref{res:non_asymptotic_power} characterizes the set of alternatives detectable by our $f$-divergence permutation tests. 
In particular, any distributions $P$ and $Q$ for which the $f$-divergence is greater than some rate decreasing the sample sizes, can correctly be detected with high accuracy.

We note that the two-sample problem is inherently symmetric, while $f$-divergences are not. To address this issue, it is possible to instead work with either the average or the maximum of 
${D}_{f}(P \;\|\; Q)$
and
${D}_{f}(Q \;\|\; P)$.
An important problem is the one of the choice of hyperparameters such as the kernel bandwidth and the regularization parameter $\lambda$. To address this issue, we construct an adaptive permutation test (detailed in \Cref{alg:perm_fuse} in \Cref{sec:adaptive}) which adaptively combines results from multiple hyperparameter configurations \citep{BSG23}.
Finally, building on \citet{SKA+23}, we also construct a test, which we call $f$-Agg, which aggregates these adaptive permutation tests over a collection of $f$-divergences to be powerful against a wide range of alternatives (detailed in \Cref{alg:meta-test} in \Cref{sec:adaptive}). \looseness=-1

\section{Hockey-Stick $f$-Divergence}
\label{sec:statistics}

Motivated by its practical relevance for privacy and unlearning, we focus on the Hockey-Stick divergence with parameter $\gamma$.
It is an $f$-divergence with $f(t)= \max(t-\gamma,0)$, and its variational form is
\begin{align*}
\mathrm{HS}_{\gamma}(P||Q) &\coloneqq \sup_{0 \leq g \leq 1}  \mathbb{E}_{X \sim P}[g(X)] 
- \gamma \mathbb{E}_{X \sim Q}[g(X)].
\end{align*}
The Hockey-Stick witness function is equal to
\begin{equation}
\label{eq:hswitness}
g^\star_{HS}(x)=\mathds{1}\left(\frac{dP}{dQ}(x)\geq \gamma\right)
\end{equation}
As in explained \Cref{sec:consistent-estimation}, by estimating $\frac{dP}{dQ}$ with $\widehat{r}_\lambda$ from \Cref{eq:r_lambda}, we obtain the witness estimator
$
\mathds{1}({\widehat{r}_\lambda(x) \geq \gamma}).
$
Noting that \(f^*(u) = \sup_{t\in\mathbb R} \{ ut - f(t) \}=\gamma u\) for all $u\in[0,1]$, the Hockey-Stick estimator equivalent of \Cref{eq:Df_estimator} is
\begin{equation}
\label{eq:hs_estimator}
\widehat{\mathrm{HS}}_{\gamma,\lambda} =
\frac{1}{\widetilde m}
\sum_{i=1}^{\widetilde m} 
\mathds{1}({\widehat{r}_\lambda(\widetilde{X}_i) \geq \gamma})
- 
\frac{\gamma}{\widetilde n}
\sum_{j=1}^{\widetilde n} 
\mathds{1}({\widehat{r}_\lambda(\widetilde{Y}_j) \geq \gamma}).
\end{equation}

A permutation two-sample test can then be constructed using this Hockey-Stick estimator as explained and analysed in \Cref{sec:testfdiv}. 
See also \Cref{alg:perm_fuse} in \Cref{sec:adaptive} for details on the adaptivity over the kernel and the regularization parameter $\lambda$.

\section{Regularized $f$-divergences}
\label{sec:regularized-f-divergences}
So far we have studied $f$-divergence based tests that leverage a regularized likelihood ratio $\hat{r}$  to compute the divergence via substitution into the $f$-divergence variational definition, i.e., through \Cref{eq:Df_estimator}.  Alternatively, one could consider estimating the regularized $f$-divergence by solving its corresponding regularized variational formulation directly, as follows.

\begin{definition}{(Regularized $f$-divergence.)}
\label{def:reg-f-divergence}
Let $D_f(\cdot || \cdot)$ be an $f$-divergence,  $\cH$ be a RKHS, and $\lambda>0$. For two distributions $P$ and $Q$ over a set $\cX$ we define the regularized $f$-divergence $D^{\cH,\tau}_f(P||Q)$ as
$$
\sup_{g:\cX \to \mathrm{Dom}(f^*)} \expected{P}{g(x)}-\expected{Q}{f^*(g(x))}-\frac{\tau}{2}\|g \|^2_{\cH}.
$$
\end{definition}

While the concept of regularized $f$-divergences  has appeared before, prior work focused on specific divergences, such as the regularized KL \citep{NWJ10, NCT16, arbel2020generalized} or KALE, and regularized-$\chi^2$ or DrMMD \citep{eric2007testing, chen2024regularized}. These previous works focused on specific applications, leveraging these statistics for asymptotic tests \citep{eric2007testing}, and particle descent and gradient flows as measures for optimization and modeling \citep{arbel2020generalized, GAG21, chen2024regularized}, while the focus of the present work is in constructing permutation-based non-asymptotic tests with adaptation over the kernel and regularization parameter. \looseness=-1

 Testing with regularized $f$-divergences is challenging, as it requires specific optimization techniques for each $f$-divergence, based on the objective and constraints in \Cref{def:reg-f-divergence}. For example, the  KALE (regularized KL) statistic formulation leads  to a convex optimization problem that can be solved efficiently, although without a closed-form solution. In contrast, the DrMMD (regularized $\chi^2$) has a closed-form expression that can be derived using tools from functional analysis. We provide their detailed formulations in \Cref{sec:appendix-kale-drmmd} and below present an approach for testing with a regularized hockey-stick. 

\paragraph{Regularized hockey-stick divergence} 

The regularized Hockey-Stick divergence is defined by
\begin{equation}
\label{eq:hstau}
\begin{aligned}
    \mathrm{HS}_{\gamma}^{\tau}(P||Q) &:= \sup_{g \in \mathcal{H}, \, 0 \leq g \leq 1} \big( \mathbb{E}_{X \sim P}[g(X)] .\\
&\qquad \quad  - \gamma \mathbb{E}_{X \sim Q}[g(X)] - \tau \|g\|_{\mathcal{H}}^2 \big),
\end{aligned}
\end{equation}
where $\tau > 0$ controls the regularization strength.
We emphasize that this $\tau$-regularization of the $f$-divergence itself is different from the $\lambda$-regularization of the $\frac{dP}{dQ}$ estimator used in \Cref{eq:Df_estimator}.

In \Cref{sec:alternatives-hs}, we provide a theoretical framework for solving the optimization problem in \Cref{eq:hstau} using the plug-in estimator and techniques from optimization with non-negative constraints. The problem turns into a semidefinite program, and can be efficiently solved using accelerated proximal splitting methods like FISTA \citep{beck2009fast}. 

In practice, we do not use this approach, since the RKHS function class employed in \Cref{eq:hstau} is not suited to representing discontinuous functions  of the form of \Cref{eq:hswitness}, and empirical performance in testing is consequently poor (see \Cref{fig:witness-direct} in the appendix for an illustration). We nonetheless present this direct approximation as it will be of interest to revisit with more suitable function classes in future work. 
\section{Experiments}
\label{sec:experiments}

We begin by evaluating the performance of the tests proposed in \Cref{sec:testfdiv} on two synthetic benchmarks, 
illustrating that different alternatives are more effectively captured by different $f$-divergence tests.
We then test on practical applications from differential privacy auditing and unlearning evaluation. All our code is publicly available at \url{https://github.com/google-research/google-research/tree/master/f_divergence_tests}

\paragraph{Experimental details.} 
For our two-sample $f$-divergence tests, we focus on the Hockey-Stick test presented in \Cref{sec:statistics}. 
In \Cref{sec:test-all-f-divergences}, we also present results for various other $f$-divergences and regularized $f$-divergences. However, in the main body we restrict ourselves to the ones with highest performance. 

We use the state-of-the-art Maximum Mean Discrepancy (MMD, \citealp{GBRS+12}, see \Cref{subsec:mmd}) as our primary baseline. 
We also compare against two novel permutation tests based on regularized $f$-divergence statistics: DrMMD (regularized $\chi^2$, \citealp{chen2024regularized}) and KALE (regularized KL, \citealp{arbel2020generalized,GAG21, grosso2025multiple}), detailed in \Cref{sec:appendix-kale-drmmd}.
Additionally, we include $f$-Agg, a test which leverages the aggregation method of \cite{SKA+23} to perform multiple testing across all divergences considered (\Cref{alg:meta-test}). For privacy experiments, we also report the power of the Hockey-Stick tester from the DP-Auditorium library \citep{KMRS24}.

All experiments use 500 samples from each distribution unless otherwise noted, with the exception of the Expo-1D experiment, which uses 1000 samples.
All tests are performed at significance level $\alpha=0.05$. For each test, we report the empirical power calculated over 100 repetitions, along with Clopper-Pearson confidence intervals. All experiments were run on NVIDIA A100 GPUs.

\subsection{Perturbed Uniforms} \label{sec:experiments-perturbed}This benchmark was introduced by \cite{SKA+23}. We test power over detection of perturbations on one and two dimensional uniform distributions, varying the amplitude of the perturbation from $a=0$ (no perturbation, corresponding to the null $P=Q$) to $a=1$ (maximum value of the perturbation for the density to remain non-negative). Consistent with previous experiments, we confirm in \Cref{fig:perturbed-uniform} that MMD is optimal at detecting these smooth differences. This is unsurprising as MMD is proved to be minimax optimal in this setting \citep{SKA+23}. Nonetheless, even in this setting optimized for MMD, we observe that the $f$-divergence tests perform well with power relatively close to MMD. 
We note that the adaptive $f$-Agg test achieves the highest power after the MMD test, showing the benefit of aggregating over multiple statistics. 
\begin{figure}
    \centering

    \includegraphics[width=0.7\linewidth]{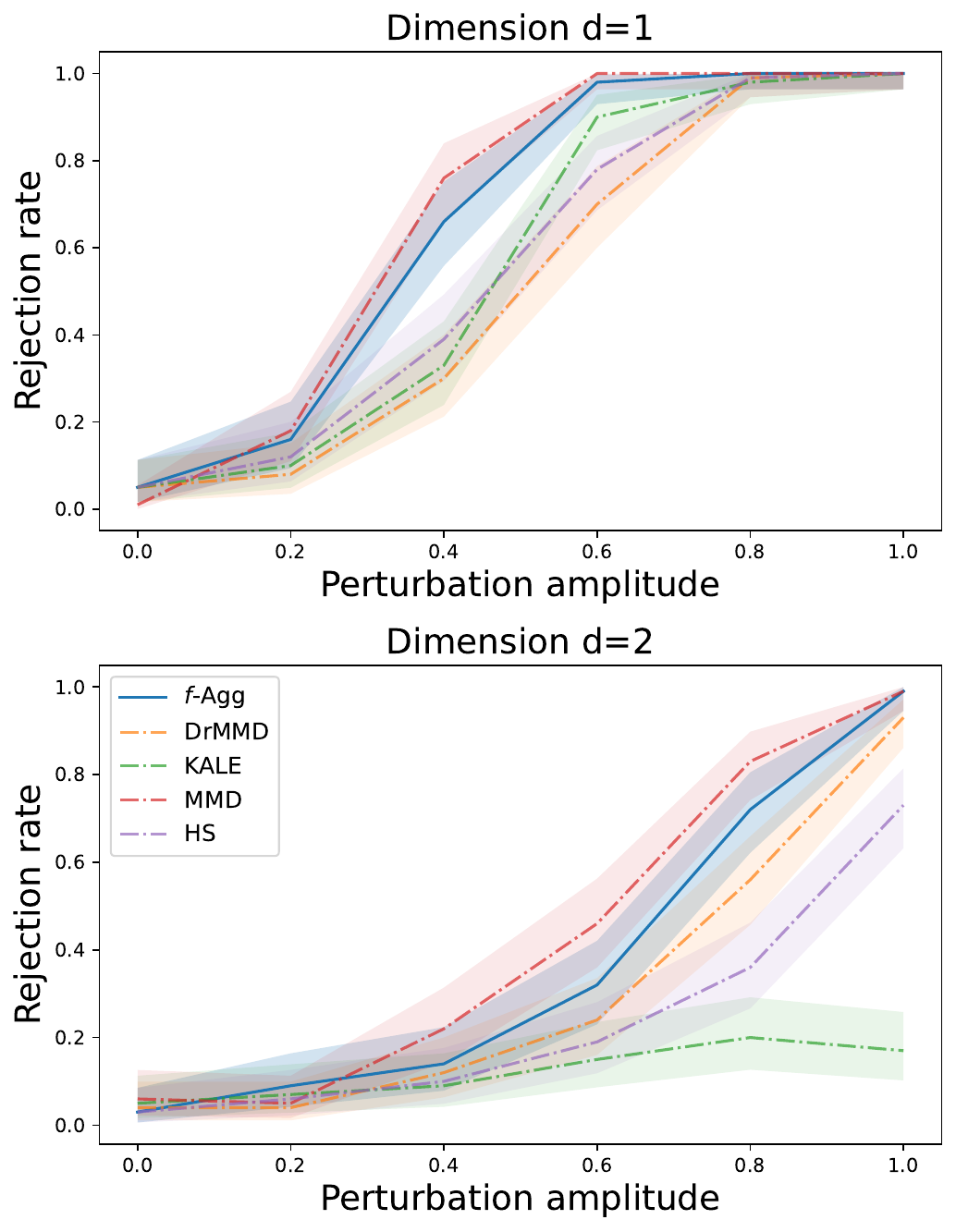}
    \caption{Performance on perturbed $d$-dimensional uniform alternatives with varying perturbation amplitudes. As the amplitude increases, the deviation from uniformity grows, leading to higher statistical power. Both MMD and DrMMD perform well in this setting, and the performance decreases for all methods in two dimensions. The $f$-Agg test maintains almost the same performance as the best method.
}
    \label{fig:perturbed-uniform}
\end{figure}
\subsection{Outlier Detection on Physics Datasets}
\label{sec:experiments-expo1d}
New physics address limitations within the current standard model of particle physics. In this context, two-sample tests are often used to identify anomalies or  validate new models, among other tasks \citep{grosso2025multiple}. 
We consider the Expo-1D benchmark, first introduced in \cite{2019learningNewPhysics}, and then considered in various works \citep{letizia2022learning,grosso2025multiple,grosso2024goodness,grosso2025sparse}. 
In this dataset, there is a reference model given by an energy spectrum that decays exponentially, described by the unnormalized density ${n_0(x) = 2000e^{-x}}$. The alternative hypothesis is given by the density $n_0(x)+k\frac{1}{\sigma\sqrt{2\pi}}\exp\left({-\frac{x-\mu^2}{2\sigma^2}}\right)$. 
In our experiments of \Cref{fig:expo1d}, we vary the values of the multiplier $k$ while fixing $\mu=4$ and $\sigma=0.01$. 
For convenience, we plot the perturbed densities in \Cref{fig:explo1d-densities} for $k\in\{0,20,40,60, 80, 100\}$.

\begin{figure}
    \centering
    \includegraphics[width=0.75\linewidth]{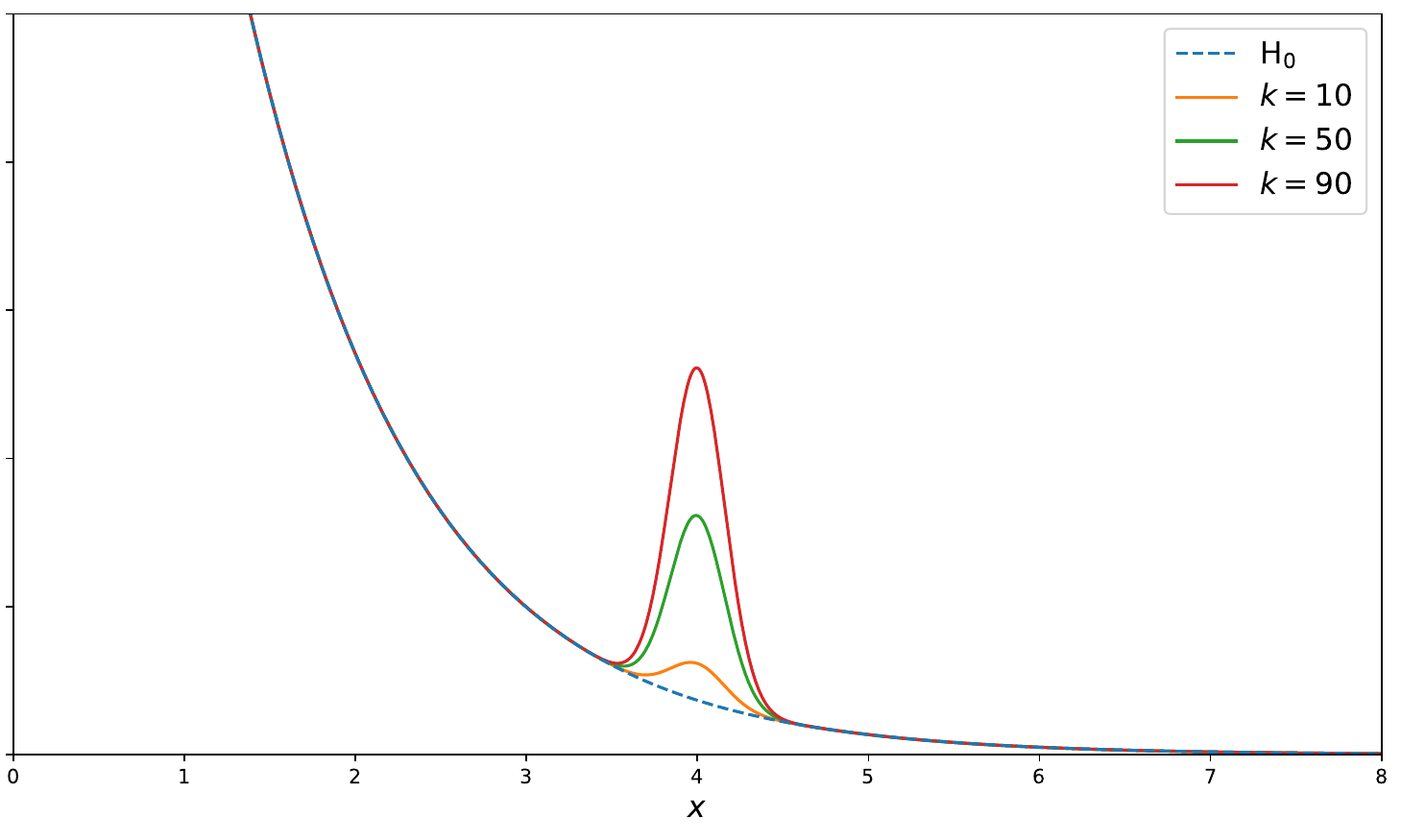}
    \caption{Expo-1D null ($H_0$) and alternative hypothesis density functions varying the multiplier $k$.}
    \label{fig:explo1d-densities}
\end{figure}
\begin{figure}
    \centering
    \includegraphics[width=0.75\linewidth]{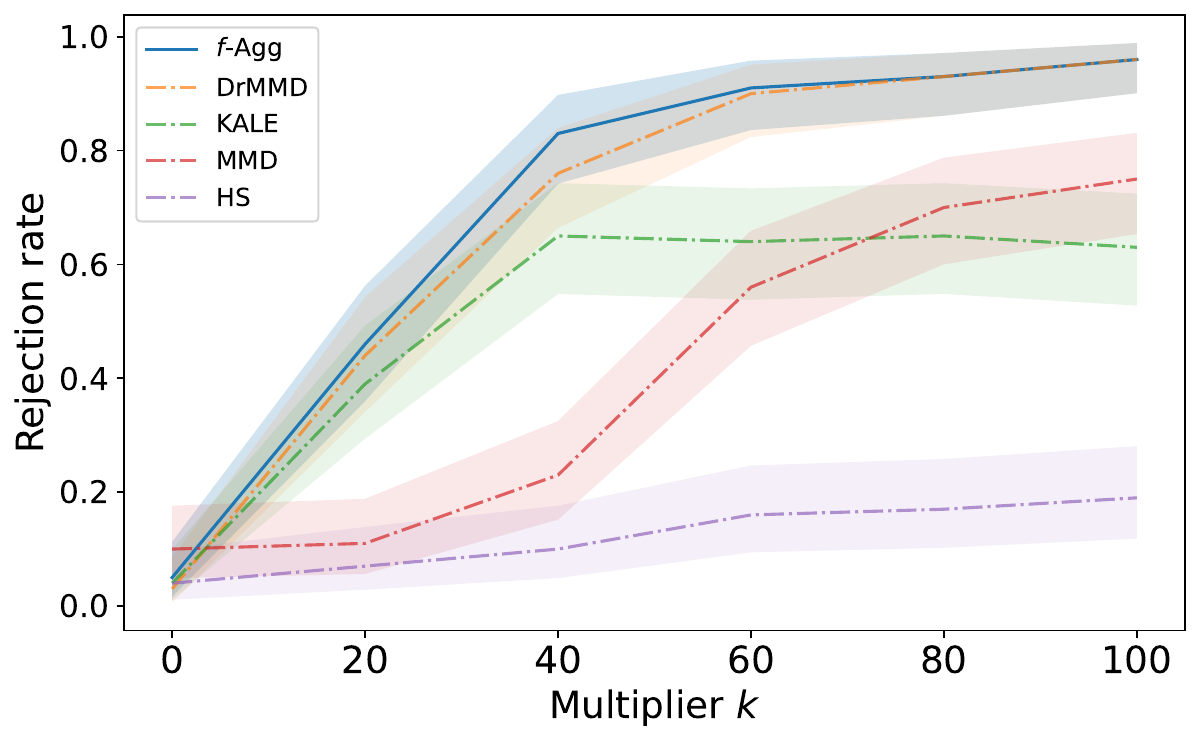}
    \caption{ Detection rate comparison on the Expo-1D task, varying the multiplier $k$ of the perturbation. The DrMMD test achieves the highest power across all values. 
}
    \label{fig:expo1d}
\end{figure}

In this setting, \Cref{fig:expo1d} shows that, excluding the adaptive $f$-Agg test, the DrMMD test achieves the highest statistical power across all values of the multiplier $k$. In contrast, MMD exhibits low power for small $k$, indicating that deregularization makes the resulting statistic more sensitive to this type of departure from the null, compared to the baseline MMD.
 The KALE estimator is also effective at small $k$, but its power stagnates as $k$ increases to greater values than 40. This can occur because the slow growth of the log-likelihood ratio as a function of $k$ poses a greater challenge to log-ratio estimators like KALE than to direct-ratio estimators like DrMMD. We illustrate this phenomenon in \Cref{app:details-expo1d}. \looseness=-1

The superior performance of DrMMD, which is equivalent to MMD with a deregularized kernel, shows that MMD's power can be significantly boosted with a sufficiently expressive kernel. 
Our results show that DrMMD achieves a power increasing with $k$, up to 0.96 in two dimensions, using only 1000 samples from each distribution. This represents a major improvement in both statistical power and sample efficiency.

We  observe that the HS-sigmoid method outperforms the hard-threshold version HS. Still, both Hockey-Stick instantiations are suboptimal for this set of alternatives.
For generic two-sample testing, we do not recommend using the Hockey-Stick divergence over other $f$-divergences.
However, there are settings, such as auditing differential privacy auditing (\Cref{subsec:audit}) and machine unlearning \citep{sekhari2021remember}, where the use of the Hockey-Stick divergence is necessarily required.

\subsection{Auditing Pure Differential Privacy}
\label{subsec:audit}

Differential privacy \citep{dwork2006calibrating} provides a rigorous framework for protecting user data by introducing calibrated noise to query responses, bounding the influence of any single individual on a query response. 
Formally, a randomized mechanism $M$ with outputs in $\mathbb{R}^p$ is $\varepsilon$-differentially private ($\varepsilon$-DP) if and only if, for any  $S \subseteq \mathbb{R}^p$ and any pair of datasets $D$ and $D'$ differing in one record, the probability of obtaining any output in $S$ is nearly the same, that is 
$$P(M(D) \in S) \leq e^{\varepsilon}P(M(D') \in S). $$
Privacy guarantees rely on the correct implementation of the  mechanism $M$. We can audit these guarantees by framing the problem as a two-sample hypothesis test, based on Lemma 5 from \cite{zhu2022optimal}. Specifically, for any correctly implemented $\varepsilon$-DP random mechanism $M$, the Hockey-Stick divergence (with parameter $e^\varepsilon$) between the outputs from two adjacent datasets $D$ and $D'$ must be zero, that is 
$$\mathrm{HS}_{e^\varepsilon\!}(M(D)||M(D')) = \mathrm{HS}_{e^\varepsilon\!}(M(D')||M(D))= 0.$$ 
 
It  follows directly that if our proposed test (\Cref{alg:perm_fuse}, \Cref{sec:testfdiv,sec:statistics}) rejects the null, $\mathrm{HS}_{e^\varepsilon\!}(M(D)||M(D'))=0$, then the mechanism $M$ cannot be private for the given $\varepsilon$. 
This method builds upon methodologies for auditing privacy mechanisms \citep{gilbert2018property, bichsel2021dp, KMRS24}. \looseness=-1

We focus on testing six non-private mechanisms previously evaluated in the literature.  The first four mechanisms, SVT3--SVT6, are incorrect instances of the sparse-vector-technique \citep{LSL17}. The methods mean-1 and mean-2 are incorrect implementations of the Laplace mechanism for the average. As in previous work, we fix $D$ and $D'$ for each mechanism. Details on the audited mechanisms and the baselines can be found in \Cref{app:details-privacy}. 
The results are reported in \Cref{tab:dp-results}. \looseness=-1

Our \Cref{alg:perm_fuse} achieves higher rejection rates for privacy violations while requiring significantly fewer samples than the baseline Hockey-Stick tester from DP-Auditorium, improving the overall sample efficiency of the auditing process. 

Achieving a rejection rate greater than level values for the SVT3 mechanism is noteworthy. While previous techniques detected some privacy violations for SVT3, they required millions of samples to approximate the likelihood ratio \citep{bichsel2021dp}. 

\begin{table}
\tiny
    \centering
    \caption{Privacy violations detection rate for several non-private mechanisms.  HS improves detection rate over DP-auditorium which requires hundreds of thousand of samples to detect some of the mechanisms below.}
    \label{tab:dp-results}
\smallskip
\resizebox{\columnwidth}{!}{
\begin{tabular}{llrrrrrr}
\toprule
 & Test & SVT3 & SVT4 & SVT5 & SVT6 & mean1 & mean2 \\
\midrule
\multirow{2}{*}{$n=500$}
 & HS & 0.095 & 0.070 & 0.785 & 0.085 & 0.970 & 1.0 \\
 & DP-Aud & 0.0 & 0.0 & 0.0 & 0.0 & 0.768 & 0.071 \\
\midrule
\multirow{2}{*}{$n=5000$}
 & HS & 0.175 & 0.090 & 0.995 & 0.165 & 1.0 & 1.0 \\
 & DP-Aud & 0.0 & 0.0 & 0.010 & 0.0 & 1.0 & 0.717 \\
\bottomrule
\vspace{-0.5cm}
\end{tabular}
}
\end{table}

\subsection{Evaluating Machine Unlearning}
\label{subsec:unlearning}
The goal of unlearning is to remove the influence of a specific subset of data $D_f \subseteq D$ often called the forget set, from a trained model $M$. A widely used definition \citep{sekhari2021remember} considers unlearning successful if the resulting model, $M_u$, is statistically indistinguishable  from a model $M_r$ that was retrained from scratch on the remaining data $D\setminus D_f$. This common definition of unlearning quantifies indistinguishability using the Hockey-Stick divergence with parameter $e^{\varepsilon}$. The parameter $\varepsilon$ controls the degree of statistical indistinguishability. Larger $\varepsilon$ allows for greater divergence between the two model distributions, while $\varepsilon=0$ enforces exact unlearning and requires the distributions to be identical.

Evaluating this definition can be framed as a two-sample hypothesis test where the null hypothesis is that the Hockey-Stick divergence of order $e^{\varepsilon}$ between unlearned and retrained models is zero, $\mathrm{HS}_{e^\varepsilon\!}(M_u||M_r)=\mathrm{HS}_{e^\varepsilon\!}(M_r||M_u)=0$, where the randomness comes from the training procedure (e.g., batch-sampling, dropout, etc.). \looseness=-1

This strict definition has the following flaw. Different training procedures (e.g., number of training steps) mean that even two models $M_{r}$ and $M_{r'}$ retrained from scratch on the exact same data will likely produce slightly different parameter distributions. A sensitive two-sample test will correctly detect this difference between $M_r$ and $M_{r'}$, and reject the null hypothesis $\mathrm{HS}_{e^\varepsilon\!}(M_{r'}||M_r)=\mathrm{HS}_{e^\varepsilon\!}(M_r||M_{r'})=0$, even though the unlearning process of retraining from scratch with different parameters is entirely valid.

To confirm that this issue arises in practice,  we trained a CNN on the CIFAR-10  dataset \cite{krizhevsky2009learning} and then applied several simplified unlearning techniques to forget a set of ten randomly selected images. For details on training procedures and unlearning algorithms, see \Cref{app:details-unlearning}.

This flaw is demonstrated in \Cref{tab:unlearning-two-sample-test-exact}. We tested several unlearning methods against a retrained model at levels $\varepsilon=0$ and $\varepsilon=0.1$. 
For $\varepsilon=0.1$, we can only test with Hockey-Stick statistic to quantify a positive separation. 
For $\varepsilon=0$, or exact unlearning, we test with all the statistics. 

For  $\varepsilon=0$ (\Cref{tab:unlearning-two-sample-test-exact}),  all approximate unlearning methods produce a different distribution than the retrained one. 
However, even for $\varepsilon>0$ the test detects the difference between the retrained model and models retrained from scratch on retain data using a different batch size and half as many steps (i.e., `retrain-batch' and `retrain-less').
This shows that the two-sample test, while statistically correct, is not practically useful for verifying that a model does not hold information about specific forget examples. \looseness=-1

We emphasize that this flaw does not represent an issue with our tests which correctly distinguish the difference in distributions of $M_u$ and $M_r$.
Instead, it highlights that the null hypothesis that corresponds to this unlearning definition, $H_0\colon \mathrm{HS}_{e^{\varepsilon\!}}(M_u||M_r)=\mathrm{HS}_{e^{\varepsilon\!}}(M_r||M_u)=0$, is not necessarily equivalent to the null hypothesis that the unlearning process is `successful', as one may intuitively reason about it.
A potentially better-suited null hypothesis to characterize successful unlearning could be the composite goodness-of-fit null hypothesis (e.g., \citealp{key2025composite}) that $M_u$ is identically distributed as some element of the set $\mathcal{M}_r$ consisting of all models retrained from scratch on $D\setminus D_f$ with arbitrary training procedures. 
However, constructing a practical test for this composite null is challenging and remains an open problem. 
In the next section, we instead propose a different approach to fixing this flaw based on three-sample hypothesis testing.
\looseness=-1

We note that the KALE-based method performs poorly on this task, potentially for two reasons. First, the high dimensionality of the sample space ($d=10$) may impact the optimization required to calculate the KALE statistic. This claim is supported by a similar performance drop observed in the two-dimensional perturbed uniforms experiment in \Cref{fig:perturbed-uniform}. Second, the log-likelihood ratio between the distributions may be ill-conditioned. Characterizing the set of alternative distributions for which KALE is optimal is an important direction for future work.

\begin{table}
\tiny
\centering
\caption{Two-sample test power results against the null hypothesis $H_0\colon \mathrm{HS}_{e^{\varepsilon\!}}(M_u||M_r)=\mathrm{HS}_{e^{\varepsilon\!}}(M_r||M_u)=0$, with $\varepsilon=0$ and $\varepsilon=0.1$. A rejection rate of 1.0 means the test detects a distributional difference. 
}
\label{tab:unlearning-two-sample-test-exact}
\smallskip
\begin{tabular}{lcccc|c}
\toprule
Unlearning& DrMMD & KALE & MMD & HS & HS \\
mechanism & \multicolumn{4}{c|}{\textbf{$\varepsilon=0$}}  & $\varepsilon=0.1$ \\
\midrule
Finetune & 1.0 & 0.2 & 1.0 & 1.0 & 1.0 \\
Pruning & 1.0 & 0.0 & 1.0 & 1.0 & 0.1 \\
SSD & 1.0 & 0.0 & 1.0 & 1.0 & 0.1 \\
Random label & 1.0 & 0.0 & 1.0 & 1.0 & 0.0 \\
\midrule
Retrained-less & 1.0 & 0.5 & 1.0 & 1.0 & 1.0 \\
Retrained-batch & 1.0 & 0.0 & 1.0 & 1.0 & 0.4 \\
\bottomrule
\vspace{-0.5cm}
\end{tabular}
\end{table}

\paragraph{Three-sample tests.}
To address the highlighted issue, we propose instead to use a {\em  three-sample test}. Instead of checking whether the unlearned model is indistinguishable to a retrained one, this procedure reframes the problem as comparing three distributions: the original model, the unlearned model, and the retrained model. The null hypothesis becomes $H_0: $ Unlearning was effective, measured by closeness of the unlearned model distribution to the retrained model distribution, relative to the distance to the original model distribution. Rejecting the null provides evidence that unlearning was ineffective.

A three-sample test was proposed by \cite{bounliphone2015test} for the MMD. This test relies on the joint asymptotic distribution of {\em two} statistics, $d(M_u, M)$ and $d(M_u, M_r)$, which respectively measure the distance from the unlearned model $M_u$ to the retrained model $M_r$, and from $M_u$ to the original model $M$.   In the event that $M_u\neq M_r\neq M$, the joint distribution of the two quantities $\mathrm{\mathrm{MMD}}(M_u, M)$ and $\mathrm{\mathrm{MMD}}(M_u, M_r)$ is asymptotically Gaussian, yielding a straightforward test with 
$
H_0: 
\mathrm{\mathrm{MMD}}(M_u, M_r)
\leq 
\mathrm{\mathrm{MMD}}(M_u, M) 
$. Details can be found in \Cref{app:details-unlearning}. 

We therefore employ this three-sample MMD test to measure the success of unlearning algorithms.  While it will be of interest to generalize the two-sample $f$-divergence tests to the three-sample setting, this requires establishing asymptotics of the variational $f$-divergence form in \Cref{eq:variational_f_divergence} for divergences of interest (such as the KL and $\chi^2$), which  will be an interesting direction for future work.

The results of the three-sample MMD test are shown in \Cref{tab:three-sample-test}. The test correctly and consistently identifies that the retrained models are ``safe'' (rejection rate = 0.0),  overcoming the flaw of the two-sample approach. Among the approximate methods, only the Random Label technique passed the evaluation, while others like Finetuning, Pruning, and SSD were found to be ineffective (rejection rate = 1.0). We emphasize that our goal is the evaluation of unlearning methods, and not designing good algorithms for unlearning. Consequently, we used simplistic implementations of the unlearning procedures and a more rigorous implementation is needed for ranking the unlearning methods in practice. 
\looseness=-1

\begin{table}
\centering
\caption{Three-sample test. Rejection rate for the relative similarity test. 
A rejection rate of 1.0 means the test detects that the unlearned model is closer to the original model than to the retrained model, a proxy for evidence that the unlearning method is ineffective. 
The test correctly identifies retraining from scratch on retain data as a safe procedure (rejection rate = 0.0). }
\label{tab:three-sample-test}
\smallskip
\tiny
\begin{tabular}{lr}
\toprule
 Method & {Power of relative test} \\
\midrule
Finetune & 1.0 \\
Pruning & 1.0  \\
SSD & 1.0  \\
Random-label & 0.0  \\
\midrule
Retrained-less & 0.0 \\
Retrained-batch & 0.0  \\
\bottomrule
\vspace{-0.5cm}
\end{tabular}
\end{table}

\section{Conclusion}
We introduce a general framework for estimating $f$-divergences, and an accompanying two-sample permutation test. We establish finite sample bounds for the statistics and consistency of the tests, and characterize their sets of detectable alternatives. 
We illustrate the applicability of the Hockey-Stick divergence for differential privacy and unlearning problems. In experiments, we show that a  deregularized MMD (DrMMD) approximation to the $\chi^2$ divergence, and the KALE approximation to the KL divergence, both outperform classical MMD tests in certain domains, notably outlier detection. Finally, we demonstrate a limitation of two-sample tests in assessing unlearning algorithms, and show that this can be overcome with a three-sample test. Theoretically grounding these observations and characterizing optimal estimators for other divergences are interesting directions for future work.

\subsubsection*{Acknowledgments}

Antonin Schrab is supported by a UKRI Turing AI World-Leading Researcher Fellowship on Advancing Modern Data-Driven Robust Al with grant number G111021.
We thank Eleni Triantafillou for helpful guidance in the unlearning experiment.
We thank Pierre Glaser and Zonghao Chen for their help to understand details of their relevant respective works, \citet{GAG21} and \citet{chen2024regularized}.

\bibliographystyle{apalike}
\bibliography{references}

\section*{Checklist}



\begin{enumerate}

  \item For all models and algorithms presented, check if you include:
  \begin{enumerate}
    \item A clear description of the mathematical setting, assumptions, algorithm, and/or model. [Yes] All theorems include assumptions, and algorithms and models clear descriptions.
    \item An analysis of the properties and complexity (time, space, sample size) of any algorithm. [Yes] We include them in the supplementary.
    \item (Optional) Anonymized source code, with specification of all dependencies, including external libraries. [Yes/No/Not Applicable]
    [Yes] We include them with our final version.
  \end{enumerate}

  \item For any theoretical claim, check if you include:
  \begin{enumerate}
    \item Statements of the full set of assumptions of all theoretical results. [Yes] We include them in all theorem statements.
    \item Complete proofs of all theoretical results. [Yes] We include them in the supplementary.
    \item Clear explanations of any assumptions. [Yes] We include them when necessary. 
  \end{enumerate}

  \item For all figures and tables that present empirical results, check if you include:
  \begin{enumerate}
    \item The code, data, and instructions needed to reproduce the main experimental results (either in the supplemental material or as a URL). [Yes] We will include the url pointer to our code with the final version.
    \item All the training details (e.g., data splits, hyperparameters, how they were chosen). [Yes] We include them in the supplementary.
    \item A clear definition of the specific measure or statistics and error bars (e.g., with respect to the random seed after running experiments multiple times). [Yes] We include them in experiments section.
    \item A description of the computing infrastructure used. (e.g., type of GPUs, internal cluster, or cloud provider). [Yes] We include them in the experiments section.
  \end{enumerate}

  \item If you are using existing assets (e.g., code, data, models) or curating/releasing new assets, check if you include:
  \begin{enumerate}
    \item Citations of the creator If your work uses existing assets. [Yes] We include all necessary citations.
    \item The license information of the assets, if applicable. [Yes] We include them in the supplementary.
    \item New assets either in the supplemental material or as a URL, if applicable. [Not Applicable]
    \item Information about consent from data providers/curators. [Not Applicable]
    \item Discussion of sensible content if applicable, e.g., personally identifiable information or offensive content. [Not Applicable]
  \end{enumerate}

  \item If you used crowdsourcing or conducted research with human subjects, check if you include:
  \begin{enumerate}
    \item The full text of instructions given to participants and screenshots. [Not Applicable]
    \item Descriptions of potential participant risks, with links to Institutional Review Board (IRB) approvals if applicable. [Not Applicable]
    \item The estimated hourly wage paid to participants and the total amount spent on participant compensation. [Not Applicable]
  \end{enumerate}

\end{enumerate}

\clearpage
\startcontents[appendices]

\appendix
\thispagestyle{empty}

\onecolumn
\aistatstitle{Regularized $f$-Divergence Kernel Tests \\ Supplementary Material}
\printcontents[appendices]{l}{1}{\setcounter{tocdepth}{2}}
\clearpage

We give an overview of the supplementary material.
In \Cref{sec:background}, we provide additional background information on common $f$-divergences, on MMD, on permutation tests, on adaptive testing, on KALE, and on DrMMD.
In \Cref{sec:experimental-design}, we run additional experiments.
In \Cref{sec:exp_details}, we provide additional detials on the experimental settings.
In \Cref{sec:alternatives-hs}, we derive a direct optimization procedure for the computation of the Hockey-Stick witness function.
In \Cref{sec:proofs}, we formally prove all the statements presented in \Cref{sec:statistics,sec:testfdiv}.
In \Cref{drmmd_power}, we prove that the DrMMD permutation tests satisfies the same theoretical power guarantees as our proposed tests.

\section{Background}
\label{sec:background}

In this section, we provide background information 
on common $f$-divergences (\Cref{sec:table}), 
on kernel mean embeddings and covariance operators (\Cref{subsec:kernel}), 
on MMD (\Cref{subsec:mmd}), 
on permutation tests (\Cref{subsec:permutation}), 
on adaptive testing (\Cref{sec:adaptive}), 
on implementation details (\Cref{sec:fuse-details}),
and on KALE and DrMMD (\Cref{sec:appendix-kale-drmmd}).

\subsection{Table of $f$-Divergences}
\label{sec:table}

We provide a table of commonly-used $f$-divergences in \Cref{tab:fdiv} with the associated function $f$, the convex conjugate $f^*$, and the witness function $g^\star$.

\begin{table}[h]
\centering
\small
\renewcommand{\arraystretch}{1.35}
\caption{\(f\)-divergences: generator \(f(t)\), convex conjugate \(f^*(u)\), and derivative \(f'(r)\). Recall that the optimal witness \(g^\star(x)=f'(r)\) with $r=\tfrac{dP}{dQ}(x)$. 
We use the Lambert $W$ function.
}
\smallskip
\begin{tabular}{|l|c|c|c|}
\hline
\textbf{$f$-Divergence} & \(\;f(t)\;\) & \(\;f^*(u)\;\) & \(\;f'(r)\;\) \\
\hline
KL (Kullback--Leibler) & \(t\log t\) & \(e^{u-1}\) & \(1+\log r\) \\
\hline
Reverse KL & \(-\log t\) & \(-1-\log(-u),\; u<0\) & \(-\tfrac{1}{r}\) \\
\hline
Jeffreys (symmetrized KL) 
& \((t-1)\log t\) 
& \(W(e^{u-1})+\tfrac{1}{W(e^{u-1})}+u-2\)
& \(1+\log r-\tfrac{1}{r}\) \\
\hline
Jensen--Shannon & \(t\log t-(t+1)\log\frac{t+1}{2}\) & \(-\log(2-e^u), u<\ln(2)\) & \(\log\frac{2r}{r+1}\) \\
\hline
Total Variation & \(\tfrac{1}{2}|t-1|\) & \(u,\;|u|< {1}/{2}\) & \(\mathrm{sign}(r-1)/2\) \\
\hline
Pearson \(\chi^2\) & \((t-1)^2\) & \(u+{u^2}/{4}\) & \(2(r-1)\) \\
\hline
Neyman \(\chi^2\) & \(\tfrac{(1-t)^2}{t}\) & \(2\big(1-\sqrt{1-u}\big),\; u<1\) & \(1-\tfrac{1}{r^2}\) \\
\hline
Vincze--LeCam & \(\tfrac{(t-1)^2}{t+1}\) & \(4-u-4\sqrt{1-u},\; |u|< 1\) & \(1-\tfrac{4}{(r+1)^2}\) \\
\hline
Squared Hellinger & \((\sqrt{t}-1)^2\) & \(\tfrac{u}{1-u},\; u<1\) & \(1-\tfrac{1}{\sqrt{r}}\) \\
\hline
Hellinger Discrimination & \(1-\sqrt{t}\) & \(-1-\tfrac{1}{4u},\; u<0\) & \(-\tfrac{1}{2\sqrt{r}}\) \\
\hline
\(\alpha\)-Divergence $(\alpha)$ & \(\tfrac{4}{1-\alpha^2}\big(1-t^{(1+\alpha)/2}\big)\) 
& \(\tfrac{2}{1+\alpha} \big( \tfrac{\alpha-1}{2} u \big)^{{(1+\alpha)}/{(\alpha-1)}} - \frac{4}{1-\alpha^2}\)
& \(-\tfrac{2}{1-\alpha}r^{(\alpha-1)/2}\) \\
\hline
Cressie--Read (\(\lambda\)) & \(\tfrac{t^\lambda-\lambda(t-1)-1}{\lambda(\lambda-1)}\) 
& \(
\big(
(u(\lambda-1)+1)^{\lambda/(\lambda-1)}-1
\big)/\lambda
\) 
& \(\tfrac{r^{\lambda-1}-1}{\lambda-1}\) \\
\hline
Hockey-Stick (\(\gamma\)) & $\max\{t-\gamma,0\}$  & $\gamma u , \;   0 < u < 1 $  & $\mathds{1}\!\left(r \geq \gamma\right) $  \\
\hline
\end{tabular}
\label{tab:fdiv}
\end{table}

\subsection{Kernel Mean Embedding and Covariance Operator}
\label{subsec:kernel}
Recall that in an RKHS $\mathcal H$ with kernel $k$, the kernel reproducing property states that $\langle f, k(x,\cdot)\rangle = f(x)$ for all $x$.
Given a distribution $P$, the kernel mean embedding is an element $\mu_P \in \mathcal H$ such that $\langle f, \mu_P \rangle = \mathbb E_{X\sim P}[f(X)]$ for all $f\in \mathcal H$. 
It can be expressed as $\mu_P = \mathbb E_{P}[k(X,\cdot)]$ so that 
$$\langle f, \mathbb E_{P}[k(X,\cdot)] \rangle =\langle f, \mu_P \rangle =  \mathbb E_{X\sim P}[f(X)] = \mathbb E_{X\sim P}[\langle f, k(X,\cdot)\rangle]$$ 
for all $f\in \mathcal H$.
The (uncentered) covariance operator $\Sigma_{P}: \mathcal H \to \mathcal H$ is defined as the operator satisfying $\langle f, \Sigma_{P} g \rangle = \mathbb E_{X\sim P}[f(X)g(X)]$ for all $f,g \in \mathcal H$. 
It can be expressed as $\Sigma_{P} = \mathbb E_{P}[k(X,\cdot) \otimes k(X,\cdot)]$, where $\otimes$ is the tensor product operator, so that
\begin{align*}
\langle f, \mathbb E_{P}[k(X,\cdot) \otimes k(X,\cdot)] g \rangle 
&= \langle f, \Sigma_{P} g \rangle \\
&= \mathbb E_{X\sim P}[f(X)g(X)] \\
&= \mathbb E_{X\sim P}[\langle f, k(X,\cdot)\rangle \langle k(X,\cdot), g\rangle]\\ 
&= \mathbb E_{X\sim P}[\langle f\otimes g, k(X,\cdot)\otimes k(X,\cdot)\rangle_{\mathrm{HS}}]\\
&= \mathbb E_{X\sim P}[\langle f, \big(k(X,\cdot)\otimes k(X,\cdot)\big)g\rangle]
\end{align*}
where $\langle \cdot, \cdot \rangle_{\mathrm{HS}}$ is the Hilbert-Schmidt inner product.

\subsection{Maximum Mean Discrepancy}
\label{subsec:mmd}
A popular approach for two-sample tests involves leveraging kernel-based metrics to compare distributions. One such metric is the Maximum Mean Discrepancy (MMD), introduced in  \cite{GBRS+12} to perform a kernel based two-sample test. 
MMD measures the distance between the mean embeddings of two distributions in the feature space induced by a kernel. 

\begin{definition}
\label{def:mmd}
Let $\cH_k$ be a reproducing kernel Hilbert space (RKHS) with corresponding kernel $k(\cdot, \cdot)$. This is, $f(x) = \langle f, k(x, \cdot) \rangle_{\cH_k}$
Then, the MMD between two distributions $P$ and $Q$ is given by
\begin{equation}
    {\mathrm{MMD}}_k(P, Q):= \sup_{f\in \cH_k: \| f\|_{\cH_k}\leq 1} \expected{X \sim P}{f(X)} - \expected{Y \sim Q}{f(Y)}
\end{equation}
\end{definition}

Given a sample $\bZ=(X_1, ..., X_n, Y_1, ...Y_m)$, the minimum variance unbiased estimate for the $\mathrm{MMD}^2_k$ is given by
\begin{align}
    \widehat{\MMD}^2_k  = \frac{1}{n(n-1)} \sum_{i=1}^n \sum_{i'=1, i' \neq i}^n k(X_i, X_{i'})
     + \frac{1}{m(m-1)} \sum_{j=1}^m \sum_{j'=1, j' \neq j}^m k(Y_j, Y_{j'})
   - \frac{2}{mn} \sum_{i =1}^n \sum_{j=1}^m k(X_i, Y_{j}).
\end{align}

\subsection{Permutation Tests}
\label{subsec:permutation}

We restrict ourselves to the two-sample testing framework where both distributions $P$ and $Q$ are unknown and only sample-access is available ($X_1, ..., X_n \overset{\mathrm{i.i.d.}}{\sim} P$, $Y_1, ..., Y_m\overset{\mathrm{i.i.d.}}{\sim} Q$). Recall that we define $\bZ = (\bX, \bY)$. 
As explained in \citet[Chapter 15.2]{lehmann2005testing}, the problem of two-sample testing $H_0\colon \bX \overset{d}{=} \bY$ is equivalent to the problem of testing for exchangeability $H_0\colon \bZ \overset{d}{=} \sigma \bZ$ for permuted samples 
$\sigma \bZ = (Z_{\sigma(1)}, ..., Z_{\sigma(N)})$ with permutation $\sigma\in\cG_N$, where 
$\cG_N$ be the symmetric group of $N\coloneqq m+n$ elements $${\cG_N= \{\sigma: [N] \to [N] : \sigma \text{ is a permutation of } [N]\}}.$$

Let $\tau: \mathbb{Y}^{n + m} \to \mathbb{R}$ be the statistic of interest. \Cref{thm:permutations} uses the fact that under the null hypothesis ($P=Q$), the data vector $\bZ$ and any permuted version $\sigma \bZ$ are equal in distribution (since all samples are i.i.d.). This theorem provides a method to compute a threshold for the test statistic using the $(1-\alpha)$-quantile of statistics computed on the permuted samples (see theorem statement for details). 

\begin{theorem}[Theorem 2, \cite{HG18}]
\label{thm:permutations}
Let $\bG = (id, G_2, . . . , G_w )$ be the vector where $id \in \cG_N$ is the identity and $G_2, ..., G_w$ are elements from $\cG_N$ drawn either uniformly at random without replacement from $\cG_N \setminus \{id\}$ (with $w \leq |\cG_N|$), or drawn with replacement from $\cG$. Let $\tau^{(1)}, ..., \tau^{(w)}$ be the ordered statistics of the values $\tau(g\bZ)$ for $g\in \bG$, and let 
$$
q_{1-\alpha}\coloneqq \tau^{(\lceil (1-\alpha)w\rceil)}
$$ 
be the $(1-\alpha)$-quantile of these values.  
For the null hypothesis $H_0\colon \bZ =^d g\bZ$ for all $g\in \bG$, it holds that
\begin{equation*}
    \mathbb{P}_{H_0, \bG}\big(\tau(Z) >  q_{1-\alpha} \big) \leq \alpha.
\end{equation*}
\end{theorem}

\Cref{thm:permutations} provides a near exact test that controls type-I error at level $\alpha$. In this paper we focus on carefully designing statistics $\tau$ that can trade-off accuracy of estimating the likelihood ratio and the efficiency in terms of number of samples.
We review a specific type of aggregation for multiple testing that will allow us to circumvent hyper parameter tuning, specifically bandwidth and regularization values. 

While we focus on the permutation method in this paper, we emphasize that an alternative line of work relying on the wild bootstrap exists \citep{fromont2012kernels,fromont2013two,chwialkowski2014wild,SKA+23,pogodin2024practical}. 

\subsection{Adaptive Testing Methods}
\label{sec:adaptive}

It is often the case that several test statistics $\tau_1, ..., \tau_k $ are computed on a fixed sample $\bZ$, for example, using different kernels and kernel bandwidths. These are then combined to potentially achieve more powerful tests. In this paper, we leverage the fuse method introduced in \cite{BSG23} and the aggregation method of \cite{SKA+23}. While AutoML techniques \citep{kubler2022automl} have also been studied  for this purpose,  we focus mostly on the fuse method due to its superior performance (see \Cref{fig:perturbed-uniform}).

\paragraph{Fuse. } This method was introduced and studied by \cite{BSG23} for two-sample hypothesis tests with the MMD. Given a sample $\bZ$, this method, for each permutation $g$, takes a softmax over the values of the individual statistics $\tau_1(\sigma\bZ), ..., \tau_k(\sigma\bZ)$. 

\begin{definition}
Let $\kappa > 0$, and $\cK = \{ \tau: \cY^{n+m} \to \mathbb{R} \} $ be a family of test statistics. Let $\pi$ be a distribution over $\cK$ that is either fixed independently of the data or is a function of the data that remains invariant under permutations of the combined sample $\bZ$. This invariance property ensures that the fusing process respects the exchangeability of samples under the null hypothesis. We define the unnormalized and normalized fused statistics as
\begin{align*}
    \mathrm{FUSE}(\bZ)&:= \frac{1}{\kappa} \log \left( \expected{\tau \sim \pi}{\exp\left(\kappa \tau(\bZ) \right)}\right), \\
\end{align*}
where the statistics can potentially be normalized to account for differences in the scale or variance of different test statistics.
\end{definition}

Following previous work, we let $\pi$ be the uniform distribution over $\cK$, giving equal weight to each considered test statistic. 
In practice, we follow the choice of $\kappa = \sqrt{n(n-1)}$ as used in the original paper \citep{BSG23}.
We introduce in \Cref{alg:perm_fuse} the permutation $f$-divergence test fusing over hyperparameters such as the kernel bandwidths and the regularization values.

\begin{algorithm}[h]
\caption{Single $f$-divergence two-sample test}
\label{alg:perm_fuse}
\DontPrintSemicolon 
\SetAlgoLined 
\KwIn{Samples $\bX = (x_1, \dots, x_n) \overset{\mathrm{iid}}{\sim}P$,  $\bY = (y_1, \dots, y_m)\overset{\mathrm{iid}}{\sim}Q$, significance level $\alpha \in (0, 1)$, list of hyperparameter configurations $\cC$,  number of permutations $B$, $f$-divergence estimator $\tau: \cY^{m+n} \times \cC$ that receives a sample as the first argument and a hyperparameter configuration as the second argument, $\Fuse$ function.}
\KwOut{Boolean decision: 1 (rejecting $H_0$), 0 otherwise.}
\BlankLine 
$\bZ \leftarrow (\bX, \bY)$\tcp*{Combine samples.} \label{line:combine}
$\mathbf{T} \leftarrow \{\tau(\bZ, c) \mid c \in \cC\}$\tcp*{Compute true statistics.} \label{line:true-stats}
$T^{\text{fuse}} \leftarrow \Fuse(\mathbf{T})$\tcp*{Compute the fused observed statistic.} \label{line:fuse_obs}
 \label{line:init_perm_list}
 $\mathcal{G} \leftarrow \{\sigma_1, \dots, \sigma_B\}$ \tcp*{Generate permutations on $n+m$.}
 \For{$\sigma \in \mathcal{G}$}{
    \For{$c\in \cC$}{
        $T_{\sigma,c} \leftarrow \tau(\sigma\bZ, c)$\tcp*{Compute statistic on all configurations.} \label{line:permuted-stats}
        }
        $\mathbf{T}_{\sigma} \leftarrow \{T_{\sigma, c} \mid c \in \cC\}$ \tcp*{Collect statistics for permutation $\sigma$.}
        $T_{\sigma}^\mathrm{fuse} \leftarrow \Fuse( \mathbf{T}_{\sigma})$ \tcp*{Compute fused statistic for permutation $\sigma$.} \label{line:permuted-fuse}
    }
$p_{\text{val}} \leftarrow \big(1 + \big|\{ \sigma_i \in G : T_{\sigma}^\mathrm{fuse} \ge T^{\text{fuse}} \}\big|\big)/(B + 1)$ \tcp*{Compute $p$-value.}\label{line:pval_formula}
\Return{$\mathds{1}(p_{\mathrm{val}} \leq \alpha)$}
\end{algorithm}

\Cref{alg:perm_fuse} introduces a test that combines the fuse and permutations approaches described above and the statistics defined in this  section to compute a p-value for two-sample tests. The test receives two samples $\bX$ and $\bY$, a statistic function derived from an $f$-divergence $\tau(\cdot, \cdot)$,  that takes the concatenation of (potentially permuted) samples $\bZ$ and a hyperparameter configuration $c \in \cC$, where $\cC$ is a set of hyperparameter configurations, a significance value $\alpha$, and a number of permutations $B$. The test first calculates the true statistic for all hyperparameter configurations (\Cref{line:true-stats}), and fuses them into one statistic $T^{\text{fuse}}$ (\Cref{line:combine}). Similarly, computes the statistic on all permutations and all configurations, and fusing the statistics over configurations for all permutations (\Cref{line:permuted-stats}) and for each permutation fuses the statistics across different configurations (\Cref{line:permuted-fuse}). Finally, it uses \Cref{thm:permutations} in \Cref{line:pval_formula} to calculate the p-value and return a binary decision if the p-value is smaller than the significance value $\alpha$. 
Note that, if $\cC$ contains a single configuration, then \Cref{alg:perm_fuse} reduces to a standard permutation test based on the statistic $\tau$.

\paragraph{MMDAgg.}
Aggregating test statistics with different magnitude scales can be challenging, as methods that combine statistics directly (e.g., the soft-maximum in MMD-Fuse) can be dominated by the statistics with the largest scales. We instead the aggregation procedure introduced by \citet[Algorithm 1]{SKA+23}, that circumvents normalization by relying on multiple testing.

We emphasize that aggregated tests have been considered in various other contexts, including independence testing \citep{albert2019adaptive}, goodness-of-fit testing \citep{schrab2022ksd}, efficient testing \citep{schrab2022efficient}, robust testing \citep{schrab2024robust}, imprecise credal testing \citep{chau2024credal}, and learning kernels \citep{zhou2025dual}.
See \citet{schrab2025practical,schrab2025unified,schrab2025optimal} for an overview on these.

The procedure considers a finite set of hyperparameter configurations, $\mathcal{C} = \{c_1, \dots, c_{|\mathcal{C}|}\}$, and uses permutations. $f$-Agg leverages this method to aggregate over different $f$-divergence statistics. 
Essentially, a multiple test is run with the level of each single $f$-divergence having a correction applied to it.
This correction is estimated using the data and extra permutations in order for the test to be as powerful as possible.
The method is summarized in \Cref{alg:meta-test}.

\begin{algorithm}[h]
\caption{$f$-Agg: Aggregated $f$-divergence two-sample test}
\label{alg:meta-test}
\DontPrintSemicolon
\SetAlgoLined
\KwIn{Samples $\bX = (x_1, \dots, x_n) \overset{\mathrm{iid}}{\sim}P$,  $\bY = (y_1, \dots, y_m)\overset{\mathrm{iid}}{\sim}Q$, significance level $\alpha \in (0, 1)$, number of permutations $B$, list $\mathcal{F}$ of functions $f$ with associated $f$-divergence estimators $\tau_f: \cY^{m+n} \times \cC_f$ that receives a sample as the first argument and a hyperparameter configuration as the second argument, lists $\cC_f$ of hyperparameter configurations for each $f\in\mathcal{F}$, Fuse function.}
\KwOut{Boolean decision: 1 (rejecting $H_0$), 0 otherwise.}
$\bZ \leftarrow (\bX, \bY)$\tcp*{Combine samples.} 
$\sigma_0 \leftarrow \mathrm{Id}$ \tcp*{Set $\sigma_0$ to be the idendity permutation.}
 $\mathcal{G} \leftarrow \{\sigma_0,\sigma_1, \dots, \sigma_B\}$ \tcp*{Generate permutations on $n+m$.}
 \For{$f \in \mathcal{F}$}{\label{ff1}
 \For{$\sigma \in \mathcal{G}$}{
    \For{$c\in \cC_f$}{
        $T_{f, \sigma, c} \leftarrow \tau_f(\sigma\bZ, c)$\tcp*{Compute statistic on all configurations.} 
        }
        $\mathbf{T}_{f,\sigma} \leftarrow \{T_{f, \sigma,c} \mid c \in \cC_f\}$ \tcp*{Collect statistics for permutation $\sigma$.}
        $T_{f,\sigma}^\mathrm{fuse} \leftarrow \Fuse( \mathbf{T}_{f,\sigma})$ \tcp*{Compute fused statistic for permutation $\sigma$.} 
    }
$T_{f}^\mathrm{fuse} \leftarrow T_{f,\sigma_0}^\mathrm{fuse}$\tcp*{Collect observed fused statistics.}
}\label{ff2}
$\mathrm{Test}(u) \leftarrow \mathds{1}(T_{f}^\mathrm{fuse} > \mathrm{Quantile}_{1-u}\{T_{f,\sigma}\mid \sigma\in\mathcal{G}\} \textrm{ for some } f\in\mathcal{F})$ \tcp*{Define aggregated test.} 
$u_\alpha \leftarrow \sup\{u\in[0,1]: \mathbb{P}_{H_0}(\mathrm{Test}(u)=1)\leq \alpha\}$ \tcp*{Compute via bisection with extra permutations.} \label{line:compute-u}
\Return{$\mathrm{Test}(u_\alpha)$}
\end{algorithm}

To compute the threshold in \Cref{line:compute-u} of \Cref{alg:meta-test}, some additional permuted statistics are computed exactly as in \Cref{ff1} to \Cref{ff2} with a new set of permutations $\mathcal{G}_2=\{\sigma_{B+1},\dots,\sigma_{2B}\}$. Then, these are used to estimate the probability under the null by a Monte-Carlo approximation. 
Finally, a bisection method is performed to determine the largest threshold for which the estimated type I error remains controlled at level $\alpha$.
See \citet[Algorithm 1]{SKA+23} for details.

\subsection{Fuse implementation details}
\label{sec:fuse-details}

We now describe some specific implementation details of \Cref{alg:perm_fuse} that we use throughout the experimental section. 
\begin{enumerate}
    \item We use only perform tests with the gaussian kernel, $k(x,y) = \exp(\|x-y\|^2/\gamma^2)$ where $\gamma$ is the kernel bandwidth. We calculate five different bandwiths using five bandwidths chosen as the uniform discretizations of the interval between half the 5\% and twice the 95\% 
    quantiles of $\{\|z-z'\|_2 : z, z'\in \bZ \}$, as in the original MMD-Fuse paper \citep{BSG23}.
    \item We fuse over the regularization parameter $\lambda$, for $\lambda \in \{0.0001, 0.001, 0.01, 0.1, 1.0, 10.0\}.$
    \item  $f$-divergences can be asymmetric, i.e. $D_f(P||Q) \leq D_f(Q||P)$. Hence, for samples $\bX, \bY$ and configuration of hyperparameters $c\in \cC$,  we use the statistic $\tau_\mathrm{final}((\bX, \bY), c)  = \max \{\tau((\bX, \bY), c), \tau((\bY, \bX), c) \}$, where $\tau$ is an $f$-divergence estimator as described in \Cref{sec:consistent-estimation}.
\end{enumerate}

\paragraph{Complexity.}
Computing the likelihood ration $r(x)$ using expression \Cref{eq:r_lambda} has a time complexity of $O(m^3 + mn + n^2)$, coming from matrix inversion and matrix multiplication. For a given $f$-divergence, we compute this ratio for all hyperparameter configurations $c\in \cC_f$ and all permutations $\sigma \in \cG$, so  the time complexity of a test is given by $O\left(|\cG||\cC_f|(m^3 + mn + n^2)\right)$. 

\subsection{Formulation of KALE and DrMMD}
\label{sec:appendix-kale-drmmd}

The KALE statistic introduced first by \cite{NWJ10} and later studied by \cite{arbel2020generalized} for generative modelling. \cite{arbel2020generalized}. The KALE divergence  corresponds to the following  kernel regularized KL divergence
\begin{align*}
\text{KALE}(P,Q)
\coloneqq~&(1+\tau)\sup_{h \in \cH}\biggl\{ 1 + \int h dP - \int \exp(h)dQ - \frac{\tau}{2}\|h\|^2_{\cH} \biggr\}.
\\
=~&\min_{f>0} \int( f(\log f - 1) +1)dQ + \frac{1}{2\tau}\Big\|\!\int f(x)k(x, \cdot )dQ \Big\|^2_{\cH}.
\end{align*}
where the expression for the dual form in the second equation is from \citet[Equation 6]{GAG21}.

DrMMD, which is a regularized $\chi^2$ $f$-divergence (or de-regularized MMD),  was first proposed by \cite{eric2007testing} in a paper proposing an asymptotic test.
It can be expressed either in its regularized variational form, or by its closed form expression 
\begin{align*}
    \text{DrMMD}(P,Q) &:=  (1 + \tau)\|(\Sigma_{Q} + \tau I)^{-\nicefrac{1}{2}}(\mu_P - \mu_Q) \|^2_{\cH} \\
    & = (1+\tau) \sup_{h \in \cH} \biggl\{ \int h dP - \int \left(\frac{h^2}{4}+h \right)dQ - \frac{\tau}{4}\|h\|^2_{\cH} \biggr\}.
\end{align*}

See \Cref{sec:regularized-f-divergences} for details on general regularized $f$-divergences.

 While DrMMD and KALE have also been studied in the context of gradient flows and particle descent by \cite{chen2024regularized} and \cite{GAG21}, these papers focus on the DrMMD/KALE divergences themselves as measures for optimization and modeling.

In our work, we shift the focus of these divergences to hypothesis testing. Specifically, we use these divergences to construct permutation-based DrMMD/KALE non-asymptotic tests, with adaptation over the kernel and regularization parameter (see \Cref{subsec:permutation}). These adaptive, non-asymptotic tests have not previously been considered in the literature, and we compare them against our proposed non-regularized $f$-divergence tests in our experiments.

For a fixed kernel and fixed regularization parameter, the DrMMD and KALE discrepancies  are computed exactly as in their original papers. Our methodological novelty is first in developing non-asymptotic, permutation tests based on these discrepancies, contrasting with the original asymptotic test proposed for DrMMD, and second, leveraging theoretical insights and techniques from our proposed method to construct adaptive DrMMD/KALE statistics that are robust to the choice of kernel and regularization.

\section{Experimental Design Considerations}
\label{sec:experimental-design}
We begin by presenting experiments that evaluate the performance of various $f$-divergences and regularized $f$-divergences, illustrating the flexibility of the proposed framework. Next, we present an experiment to build intuition about the regularization parameter $\lambda$ for $f$-divergence testing with estimator in \Cref{eq:Df_estimator}. We also show the advantage of leveraging the variational formulation in \Cref{eq:Df_estimator}, even when an estimator for the likelihood ratio  $\frac{dP}{dQ}(x)$ is readily available. 

\subsection{Testing with Various $f$-Divergences}
\label{sec:test-all-f-divergences}

To showcase the flexibility of the proposed method to accommodate regularised and non-regularised $f$-divergences we ran the Expo-1D tests (\Cref{sec:experiments-expo1d}) and perturbed uniform tests (\Cref{sec:experiments-perturbed}) using different regularized and non-regularized $f$-divergence based statistics. 

In \Cref{fig:all-f-divergences-perturbed} we observe that MMD still outperforms all divergences in this setting in dimensions $d=1,2$. The performance decreases for all methods in two dimensions. However, the rejection rate of all but the reversed KL converges to 1 as the perturbation amplitude increases in dimension $d=1$. 

In \Cref{fig:all-f-divergences-expo1d} we see that the KL-based test leveraging the likelihood ratio estimator in \Cref{eq:r_lambda} has similar performance to the DrMMD based test, and these two outperform all other $f$-divergence based tests.

\begin{figure}[h]
    \centering
    \includegraphics[width=0.9\linewidth]{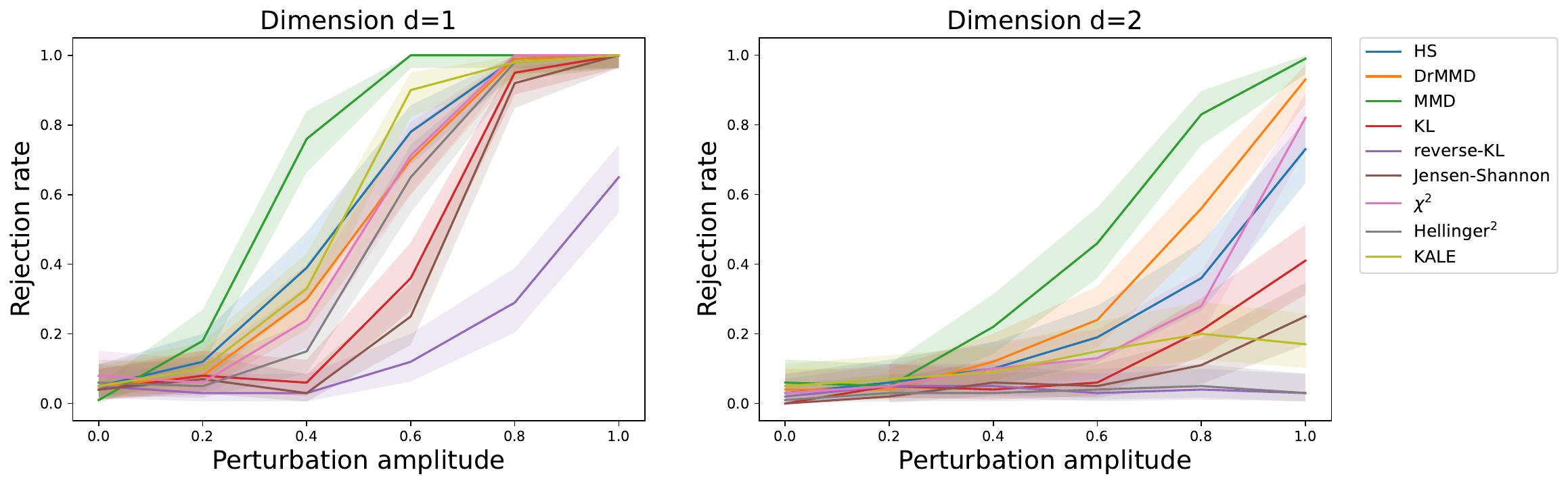}
    \caption{Performance on perturbed $d$-dimensional uniform alternatives with varying perturbation amplitudes for a diverse set of $f$-divergences. MMD outperforms all divergences in this setting. The performance decreases for all methods in two dimensions.}
    \label{fig:all-f-divergences-perturbed}
\end{figure}

\begin{figure}[h]
    \centering
    \includegraphics[width=0.5\linewidth]{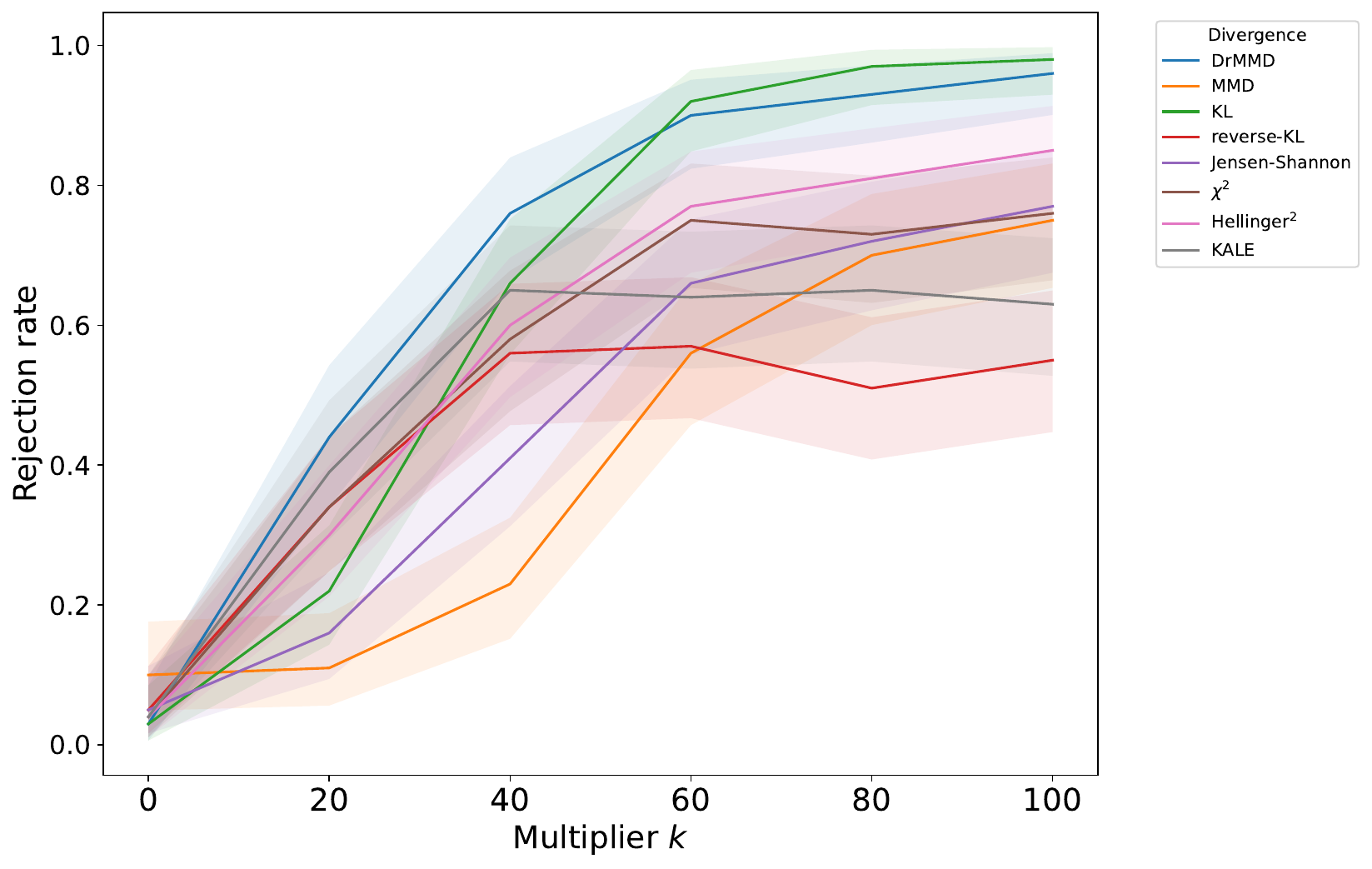}
    \caption{DrMMD (Regularized $\chi^2$) and KL are the top-performing divergences overall in this setting. MMD improves in smooth settings (large $k$). For smaller $k$, the regularized estimators, KALE (regularized log-ratio) and DrMMD (regularized ratio) have superior performance. }
    \label{fig:all-f-divergences-expo1d}
\end{figure}

\subsection{Effect of regularization parameter $\lambda$ and the importance of adaptation}
\label{sec:effect-regularization}
By design, larger values of regularization make the estimation of the (regularized) likelihood more tractable but introduce bias. However, in our tests we do not need to compromise accuracy nor smoothness since we adapt over the parameter $\lambda$ as detailed in \Cref{sec:adaptive}. Therefore, all the experiments in the manuscript are currently run with adaptation over $\lambda$, explaining the high power of our proposed method.

 To showcase the importance of varying and adapting over the  regularisation parameter, we ran the proposed test for the Expo-1D experiment to compare rejection rates across diferent $\lambda$ values for the KL and $\chi^2$ divergence based tests. In this experiment, we still adapt over five kernel bandwidths but do not adapt over the regularization parameter $\lambda$.

 \begin{figure}
     \centering
     \includegraphics[width=1.0\linewidth]{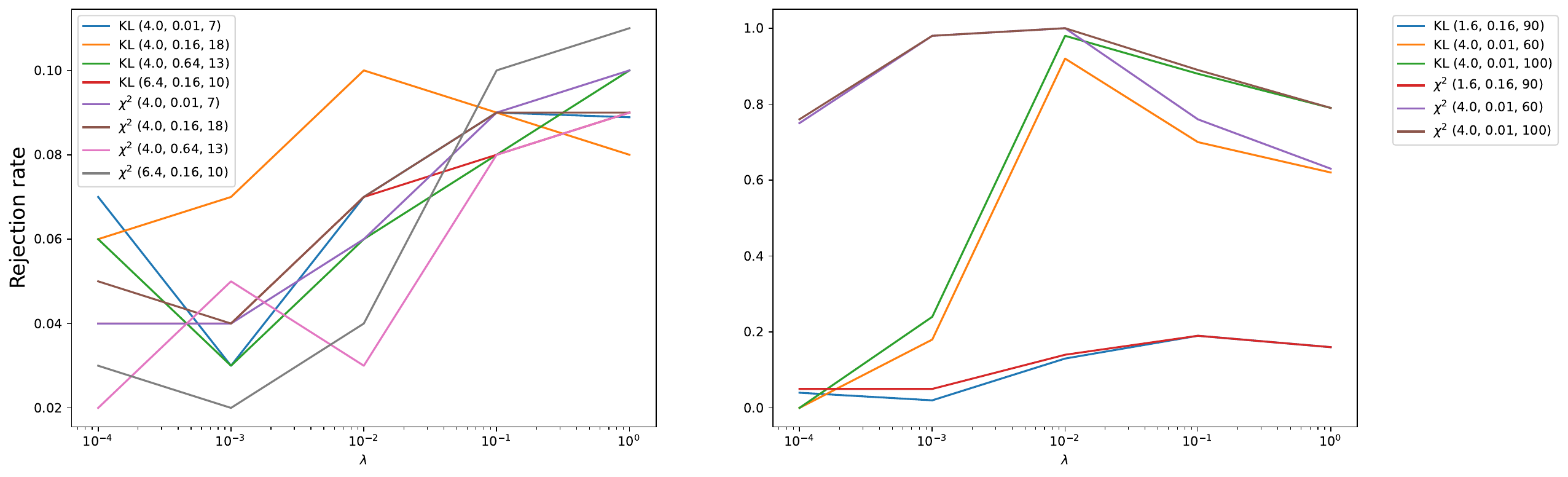}
     \caption{Rejection rate of \Cref{alg:meta-test} for different alternatives, adapting only over the kernel bandwidth. Each line corresponds to a specific $f$-divergence on a given Expo-1D alternative, with hyperparameters $(\mu, \sigma, k)$ in the legend. For details see \Cref{sec:experiments-expo1d}. Alternatives closer to the null (left) achieve higher power with stronger regularisation ($\lambda=1.0$), while larger departures from the null (right) require lower regularisation. }
     \label{fig:sweep-lambda}
 \end{figure}

 We show the results in \Cref{fig:sweep-lambda}. This experiment highlights that optimal detection rates depend on matching the regularization value to the specific alternative. Note that 'easier' cases ($k=100$, right panel) require lower regularization ($\lambda \approx 0.001$), while harder alternatives ($k=10$, left panel) achieve higher power with stronger regularization ($\lambda=1.0$). These results justify our method's need for adaptivity.

\subsection{Comparison of Direct $f$-divergence estimator vs. Variational Estimators}
\label{sec:direct-vs-variational}
In \Cref{sec:consistent-estimation} we introduce an estimator based the variational formulation of $f$-divergences for its flexibility in handling regularized $f$-divergences. However, a direct plug-in approach in \Cref{def:direct-f-divergence} is a natural baseline for standard $f$-divergences.

To address this, we ran new experiments to compare the test power of the direct estimator against our variational approach. 
For clarity, let us introduce:

\begin{align*}
    \text{Direct: } \hspace{1cm} \widehat{D}_d(P||Q) &= \mathbb{E}_Q[f(dP/dQ)]\\
    \text{Variational: } \hspace{1cm} \widehat{D}_v(P||Q) &= \mathbb{E}_P[f’(dP/dQ)]-\mathbb{E}_Q[f^*(f’(dP/dQ)] \\
\end{align*}

\paragraph{Setting.} For testing, we used the Expo-1D dataset as in \Cref{sec:experiments-expo1d}, varying the parameter $k$ to control the departure from the null hypothesis ($k=0$ is the null; $k=100$ is the maximum departure). We compute the test statistic based on $\max(\widehat{D}(P||Q), \widehat{D}(Q||P))$ using either the direct or variational estimators. We run the experiments with $n=1000$ samples from $P$ and $Q$.

\begin{figure}
    \centering
    \includegraphics[width=0.7\linewidth]{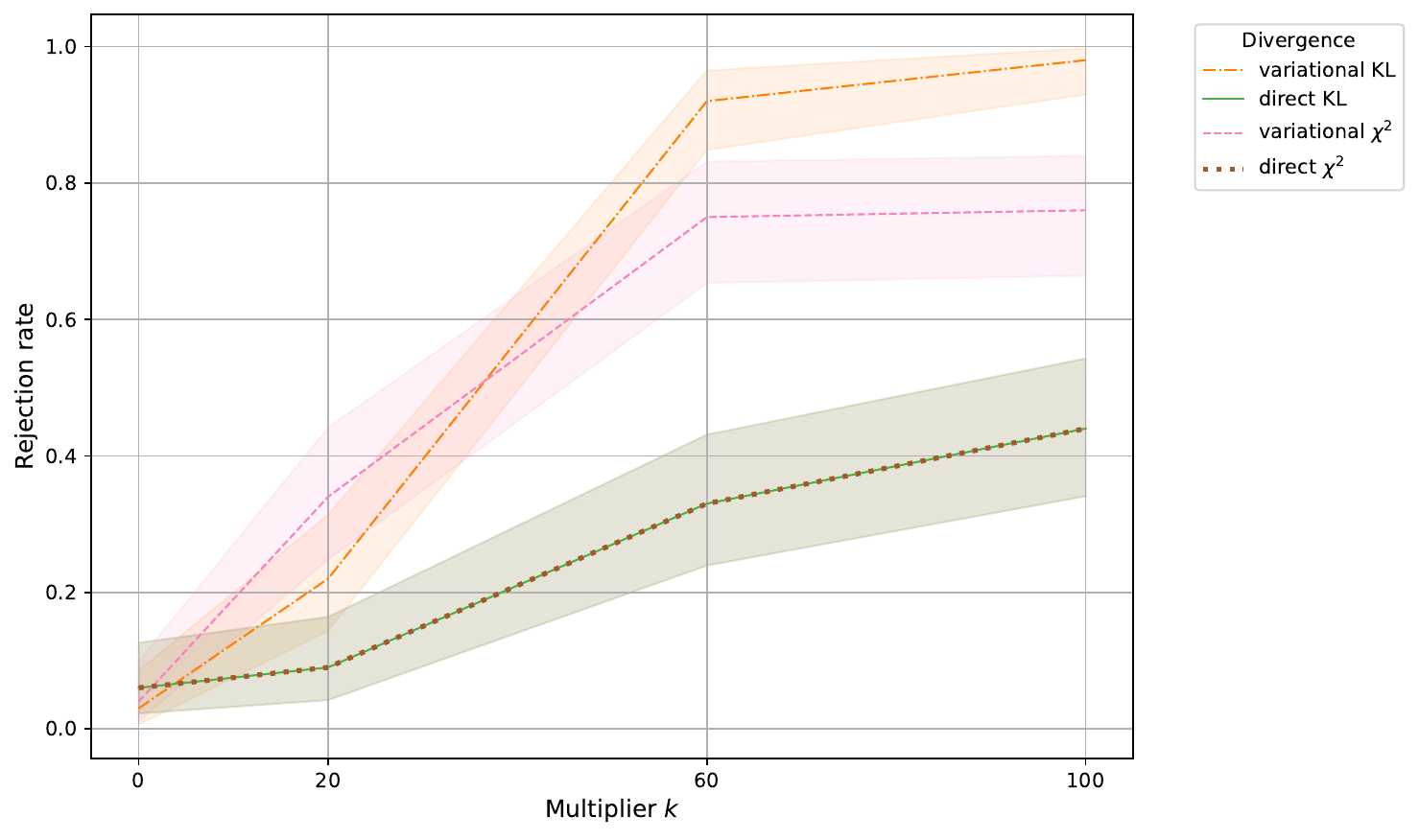}
    \caption{Rejection rate for direct $f$-divergence estimator vs. variational $f$-divergence estimator based two sample tests on the Expo-1D dataset.}
    \label{fig:direct_vs_variational}
\end{figure}

We observe in \Cref{fig:direct_vs_variational} that the variational approach achieves much higher power than the direct approach, which justifies the use of the proposed variational method. 

\section{Experimental Details}
\label{sec:exp_details}

This section provides additional details on the experimental settings of \Cref{sec:experiments}: 
 the Expo-1D experiment (\Cref{app:details-expo1d}),
 differential privacy experiments (\Cref{app:details-privacy}), and machine unlearning (\Cref{app:details-unlearning}).

\subsection{Expo-1D}
\label{app:details-expo1d}
\paragraph{KALE performance in \Cref{fig:expo1d}.}
In \Cref{sec:experiments} we noted that KALE outperforms MMD-based methods for small values of $k$ but after $k=40$ its performance grows more slowly than DrMMD and MMD. \Cref{fig:log-ratio} below shows how the log-ratio of the specific tested densities grows more slowly than the ratio, potentially making it harder for methods that rely on the ratio to continue improving without more samples. 

\begin{figure}[h]
    \centering
    \includegraphics[width=0.7\linewidth]{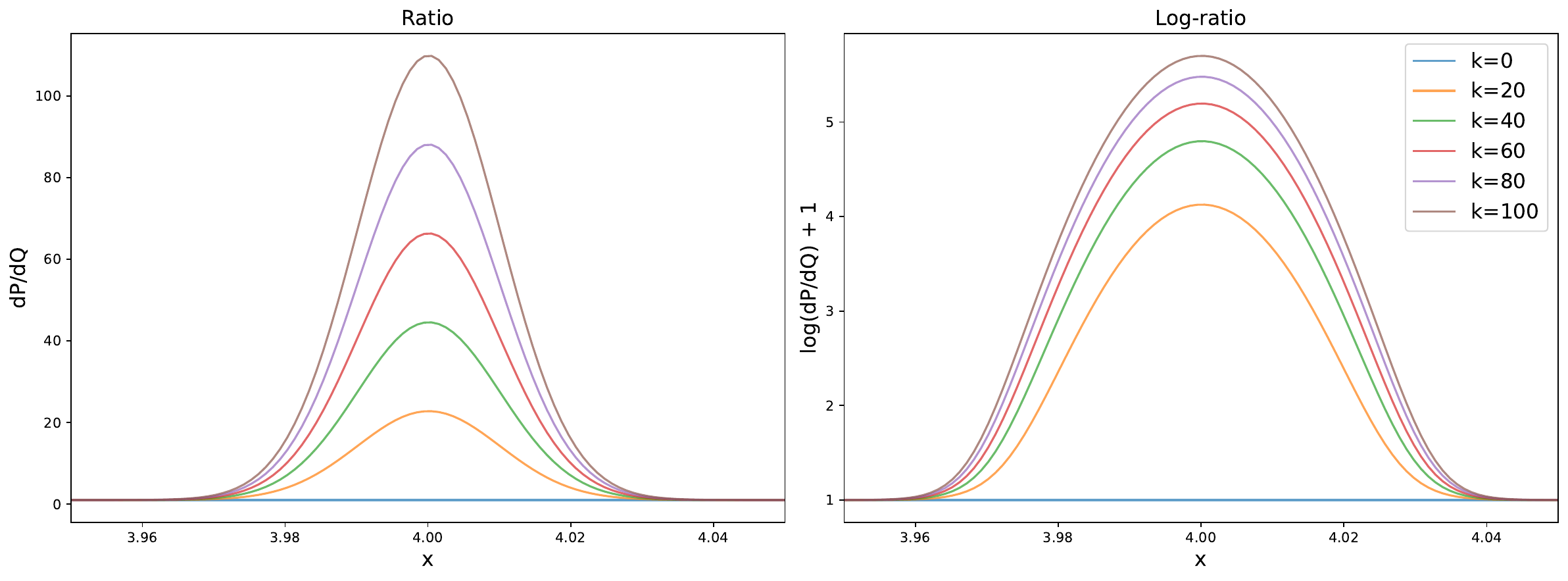}
    \caption{Ratio (left) and Log-Ratio (right) of densities involved in the Expo-1D test. }
    \label{fig:log-ratio}
\end{figure}

\paragraph{Comparison to \cite{grosso2025multiple}.}
In \Cref{tab:divergence_pivot}, we compare the proposed $f$-divergence tests against the results from \cite{grosso2025multiple} using their experimental configuration at a significance level of $\alpha = 0.023$.

The methodology in the prior work involved computing the KALE statistic over 4000 trials, using 200'000 samples from the null distribution; the number of samples from the alternative distribution was not specified. For our $f$-divergence tests, we use 5000 samples from each distribution and report the average statistical power and confidence intervals over 100 trials. We compare our results against the best-performing aggregation method reported by \cite{grosso2025multiple}.

The comparison shows that an $f$-divergence statistic paired with the fuse aggregation method (smax-t in their work) achieves greater statistical power in nearly all settings, using significantly fewer samples. 

These results also align with the behavior illustrated in \Cref{fig:log-ratio}, where KALE performs best for small density ratios where the log-ratio is large, while MMD becomes more powerful as the ratio increases (e.g., for multiplier $k=90$).

\begin{table}[h]
\centering
\caption{Performance of different methods on the Expo-1D dataset on the same parameter configuration as previous work. Each column presents a parameter configuration $(\mu, \sigma, k)$ that configures the alternative $n_a = n_0(x)+k\frac{1}{\sigma\sqrt{2\pi}}\exp\left({-\frac{x-\mu^2}{2\sigma^2}}\right)$. We confirm the intuition in \Cref{fig:log-ratio} that KALE performs better when the density ratio is small, but once the multiplier is large $(k=90)$ MMD is a better statistic.}
\label{tab:divergence_pivot}
\tiny
\smallskip
\begin{tabular}{llllll}
\toprule
$(\mu, \sigma, k)$ & (1.6, 0.16, 90.0) & (4.0, 0.01, 7.0) & (4.0, 0.16, 18.0) & (4.0, 0.64, 13.0) & (6.4, 0.16, 10.0) \\
\midrule
HS & $0.250 \ [0.169, 0.347]$  & $0.020 \ [0.002, 0.070]$  & $0.030 \ [0.006, 0.085]$  & $0.030 \ [0.006, 0.085]$  & $0.030 \ [0.006, 0.085]$  \\
HS-sig & $0.390 \ [0.294, 0.493]$  & $0.020 \ [0.002, 0.070]$  & $0.030 \ [0.006, 0.085]$  & $0.010 \ [0.000, 0.054]$  & $0.040 \ [0.011, 0.099]$  \\
DrMMD & $0.110 \ [0.056, 0.188]$  & $0.220 \ [0.143, 0.314]$  & $0.190 \ [0.118, 0.281]$  & $0.040 \ [0.011, 0.099]$  & $0.320 \ [0.230, 0.421]$  \\
KALE & $0.320 \ [0.230, 0.421]$  & $\mathbf{0.330} \ [0.239, 0.431]$  & $\mathbf{0.420} \ [0.322, 0.523]$  & $\mathbf{0.470} \ [0.369, 0.572]$  & $0.540 \ [0.437, 0.640]$  \\
MMD & $\mathbf{0.620} \ [0.517, 0.715]$  & $0.030 \ [0.006, 0.085]$  & $0.030 \ [0.006, 0.085]$  & $0.030 \ [0.006, 0.085]$  & $0.040 \ [0.011, 0.099]$  \\
\midrule
\cite{grosso2025multiple} & 0.008 $\pm 0.001$ & 0.103 $\pm 0.007$ & 0.012 $\pm 0.002$ & 0.32 $\pm 0.01$ & $\mathbf{0.66} \pm 0.01$ \\
\bottomrule
\end{tabular}
\end{table}

\subsection{Differential Privacy}
\label{app:details-privacy}
\paragraph{Audited mechanisms.}
We consider the 4 non-private variants of the sparse vector technique (SVT) introduced by \cite{LSL17} as algorithms 3--6. This mechanism for releasing a stream of queries compares each query value against a threshold. Each algorithm returns certain outputs for a maximum number of queries $c$.  SVT4 satisfies $(\frac{1+6c}{4})$-DP, and SVT3, SVT5, and SVT6 do not satisfy $\varepsilon$-DP for any finite $\varepsilon$.

The mechanisms mean1 and mean2 were introduced in \cite{KMRS24} and are non-private mechanisms that take the average of $n$ numbers (for any finite $\varepsilon$) and add Laplace noise. Mean1 violates the guarantee by accessing the private number of points $n$, and mean2 one privatizes the number of points to estimate the scale of the noise added to the mean statistic but the mean itself is computed using the non-private number of points.

For each mechanism we fixed the neighboring test datasets specified in \Cref{tab:adjacent_datasets}.

\begin{table}[h]
\centering
\caption{Adjacent Datasets $D$ and $D'$ used to test differential privacy guarantees.}
\label{tab:adjacent_datasets}
\begin{tabular}{lcc}
\toprule
Mechanism & $D$ & $D'$ \\
\midrule
Mean mechanisms & $\{ 1 \}$  & $\{ 1 , 0\}$ \\
SVT mechanisms & $\{ 0, 1  \}$ &  $\{ 0,1,0 \}$ \\
\bottomrule
\end{tabular}
\end{table}

\subsection{Machine Unlearning}
\label{app:details-unlearning}

\paragraph{Model training and data.} For this experiment we used the CIFAR10 dataset \citep{krizhevsky2009learning}. We train a simple convolutional neural network with two convolutional layers, with 16, and 32 features respectively. This is followed by two dense layers that output logits for 10 categories. The model is trained to minimize the cross entropy loss. We train for 10 epochs with a batch size of size 128 on 90\% of the data, and leave 10\% for validation. We target the last two dense layers for modification when applying unlearning techniques. The forget set is constructed by sampling 10 different images at random after shuffling the dataset. 

\paragraph{Unlearning methods in \Cref{sec:experiments}.} There is a variety of unlearning methods with varying characteristics and different applications. They use different techniques that include gradient ascent in the forget set, finetuning in the retain set, pruning, saliency or data correlation metrics. We consider the following algorithms which we believe capture a large portion of these techniques. By no means we optimize the hyperparameter or replicate exactly their algorithms. 

\begin{enumerate}
    \item Selective Synaptic Dampening (SSD) by \cite{foster2024fast}: The core idea of SSD is to  "dampen" the parameters that are most important for the forgotten data, pulling them back toward their original state before training. It avoids damaging the model's overall performance by protecting parameters that are also important for the retain data. It achieves this by:
\begin{itemize}
    \item Calculating parameter importance for both the forget and retain sets (approximated by the squared gradient of the loss)
    \item Creating a "dampening factor" for each parameter that is high if it's important for the forget set but not the retain set.
    \item Minimizing a loss function that penalizes changes from the original model's weights, weighted by these dampening factors.
\end{itemize}

    \item Pruning: Pruning methods address unlearning by pruning the most influential neurons for the forget set, measured by mean activation magnitude when evaluated on the forget set.
    \item Random Label: These methods erase influence of the forget set by training on randomly labeled forget data, and then finetuning on retain data.
    \item Finetuning: these models perform finetuning on the retain set. 
    \item Retraining for a different number of iterations.
    \item Retrain from scratch: Retrains the model on retain data with the exact same train configuration as the original model. 
    \item Retrain less: Retrains the model from scratch on retain data for only 5 epochs, other parameters are fixed. 
    \item Retrain batch: Retrains the model from scratch on retain data with a batch size of 256 (instead of 128). 
\end{enumerate}

To test the unlearning capability of the above methods we first generate 5000 models from each distribution that we split in 10 groups of 500 samples each, for 10 repetitions of each test.

\paragraph{Three sample unlearning-evaluation test. }
We use the following relative similarity test introduced in \cite{bounliphone2015test} in the context of model selection of generative models. 

Let  $P$, $Q$, and $R$ be three distributions over the same domain. Given samples $\bX = (x_1, \dots, x_m) \overset{\mathrm{iid}}{\sim}P$,  $\bY = (y_1, \dots, y_n)\overset{\mathrm{iid}}{\sim}Q$, and $\mathbf{W} = (w_1, \dots, w_r)\overset{\mathrm{iid}}{\sim}R$ such that $P\neq Q$, $P\neq R$, we consider the hypothesis test with null hypothesis $H_0: MMD(P,Q) \leq \mathrm{\mathrm{MMD}}_k(P, R)$ against  the alternative $H_1:=\mathrm{\mathrm{MMD}}_k(P,Q) > \mathrm{\mathrm{MMD}}_k(P, R)$ at significance level $\alpha$, where $\mathrm{\mathrm{MMD}}_k$ is the population MMD for the RKHS with corresponding kernel $k$. 

Theorem 2 in \cite{bounliphone2015test} establishes that 
if  $\mathbb{E}[k(x_i,x_j)] < \infty$ and  $\mathbb{E}[k(y_i,y_j)] < \infty$ and $\mathbb{E}[k(x_i,y_j)] < \infty$, then the joint distribution of the unbiased empirical estimators of $\mathrm{\mathrm{MMD}}_k(P,Q)$ and $\mathrm{\mathrm{MMD}}_k(P,Q)$, ($\widehat{\mathrm{MMD}}_k(P,Q)$ and $\widehat{\mathrm{MMD}}_k(P,Q)$ respectively) satisfy: 
\begin{equation}
\sqrt{m} 
\left( \begin{pmatrix}
\widehat{\mathrm{MMD}}^2_k({\bX}, {\bY}) \\ 
\widehat{\mathrm{MMD}}^2_k({\bX}, {\mathbf{W}})
\end{pmatrix}
-
\begin{pmatrix}
\mathrm{MMD}^2_k({P},{Q}) \\ 
\mathrm{MMD}^2_k({P}, {R})
\end{pmatrix}  
\right)
\overset{d}{\longrightarrow} \mathcal{N} 
\left( 
\begin{pmatrix}
0 \\ 
0
\end{pmatrix},
\begin{pmatrix} 
\sigma_{PQ}^2 & \sigma_{PQ, PR} \\ 
\sigma_{PQ, PR} & \sigma_{PR}^2 
\end{pmatrix} 
\right)
\label{eq:joint_asymptotic_MMD}
\end{equation}

The variance terms $\sigma^2_{PQ}$, $\sigma^2_{PR}$, $\sigma^2_{PQ, PR}$ can be estimated using samples through Equations (2) and (7) in the original paper by \cite{bounliphone2015test}. Consequently, the $p$-value $p$ for the test $H_0$ against $H_1$ can be calculated   using the following one-sided inequality:

\begin{equation}
p \leq \Phi \left( -\frac{\widehat{\mathrm{MMD}}^2({\bX},{\bY})-\widehat{\mathrm{MMD}}^2({\bX},{\mathbf{W}})}{\sqrt{\sigma_{PQ}^2 + \sigma_{PR}^2 - 2 \sigma_{PQ, PR}}} \right), 
\end{equation}

where $\Phi$ is the cumulative distribution function of the standard normal distribution. The test is performed by comparing $p$ to the significance level $\alpha$.

\section{Direct Optimization of the Regularized Hockey-Stick Witness Function} 
\label{sec:alternatives-hs}

In this section, we derive a direct optimization of the witness function of the regularized Hockey-Stick divergence.

\paragraph{Optimization via semidefinite programming}
Estimating \Cref{eq:hswitness} is hard in practice, when only finite samples $X_1, ..., X_n \sim P$ and $Y_1, ..., Y_n \sim Q$ from the distributions $P$ and $Q$ are available

The primary challenge lies in the constraints $0\leq g(x)\leq1$, which must hold for all $x$. since enforcing these constraints over an infinite set of points is intractable. Simply enforcing it on the sample points $\{X_i\}_i$ and $\{Y_i\}_i$ provides no guarantee that the constraint holds elsewhere.

 Several alternatives present their own challenges: 
\begin{itemize}
    \item Reducing the class of functions $g$ to generalized linear models (GLMs) (e.g. restricting $g$ to the form $g(x) = \langle \phi(x), g \rangle$), introduces a non-convex loss function to map a solution $g$'s outputs to the bounded range $[0,1]$, and making it hard to assert convergence guarantees about the solution.
    \item Replacing the original range constraint with a finite number of constraints over finite samples leads to either unfeasible solutions that violate the range constraint outside of the samples, or computationally  impractical optimization problems for large sample sizes. 
\end{itemize}

Instead, inspired by work on non-negative optimization \citep{MBR20} we first formulate an optimization problem with non-negativity constraints, we reparameterize the witness function as $g(x) = \phi(x)^T A \phi(x),$ where $\bA$ is a hermitian positive semidefinite linear operator $(A\succcurlyeq 0)$, $\phi(x)$ is the empirical feature map associated with the kernel $k(\cdot, \cdot)$, data $z_i = \begin{cases} x_i & \text{ if } i=, 1,..., n\\
y_{i-n} & \text{otherwise}.
\end{cases}$. 
Recall that the empirical feature map for $x \in \cX$ is given by $\phi(x) = \bV^\top v(x)$,  $v(x) = (k(x, z_i))^n_{i=1}$, and $\bV$ is the Cholesky decomposition of $\bK$ , i.e., $\bK = \bV^{\top}\bV.$. This formulation inherently ensures $g(x)\geq 0$ for all $x$ and enables efficient optimization.

\begin{theorem}
\label{thm:primal-hs}

The semidefinite program defined by

\begin{align*}
  \textbf{Primal 1: } \quad  \min_{\bA \succcurlyeq 0} \quad & \frac{e^{\varepsilon}}{n} \sum_{i=1}^n \phi(y_i)^TA\phi(y_i)
  - \frac{1}{n}\sum_{i=1}^n \phi(x_i)^TA\phi(x_i) + \tau\|A\|_F
\end{align*}

 has a unique solution $A^*$ that can be written as 

\begin{equation*}
    A^* = \sum_{i=1}^{2n} B_{ij}\phi(z_i)\phi(z_j)^T, \quad \text{ for some matrix } \quad B \in \mathbb{R}^{2n\times 2n}
\end{equation*}
where $\phi(x) = \bV^\top v(x)$,  $\bV$ is the Cholesky decomposition of $\bK$ , i.e., $\bK = \bV^{\top}\bV$, and $v(x) = (k(x, z_i))^n_{i=1}$ is the empirical feature map and $$z_i = \begin{cases} x_i & \text{ if } i=, 1,..., n\\
y_{i-n} & \text{otherwise}.
\end{cases}$$

\end{theorem}

Further, this problem can be efficiently solved because as we show below, the dual has only $2n$ variables, instead of $4n^2$, and for this specific case can be optimized using accelerated proximal splitting methods as FISTA.

\begin{lemma}
\begin{itemize}
\item[(a)] Let $\bK \in \mathbb{R}^{n\times n}$ be the kernel matrix, i.e., $\bK_{i,j}=k(z_i, z_j)$, where $z_i$ is defined in \Cref{thm:primal-hs}. Let $\bV$ be the Cholesky decomposition of $\bK$ , i.e., $\bK = \bV^{\top}\bV.$

Then, the dual of Primal 1 is given by 

\begin{equation}
    \label{eq:dual}
    \inf_{\alpha_x, \alpha_y \in \mathbb{R}^n} M (\|\alpha_x + \mathds{1}/n\|_1 + \|\alpha_y - \frac{e^{\varepsilon}}{n}\mathds{1}\|_1) + \frac{1}{2\tau} \|[\bV \text{Diag}(\alpha)\bV^{\top} ]_{-}\|^2_{F}.
\end{equation}

\item[(b)] If $\alpha^{*}$ is a solution of \Cref{eq:dual}, then 

\begin{equation}
    \bB = \tau^{-1}\bV^{-1}[\bV\text{Diag}(\alpha^{*})\bV^{\top}]_{-}\bV^{-\top}
\end{equation}

is a solution for the problem \textbf{Primal 1}. 
 \end{itemize}
\end{lemma}

\begin{proof}
   Define the loss function $L$ on $2n$ variables$$L(x_1, ..., x_n, y_1, ..., y_n) = \frac{e^{\varepsilon}}{n}\sum_{i=1}^ng(y_i) - \frac{1}{n}\sum_{i=1}^ng(x_i).$$ It follows from Theorem 2 in \cite{MBR20} that the dual problem of Primal 1 with objective  $$\min_{\bA \succcurlyeq 0} \quad  L(x_1, ...x_n, y_1, ..., y_n) + \tau\|A\|_F,$$ is defined by 

   $$\sup_{\alpha \in \mathbb{R}^n} - L^*(\alpha) - \frac{1}{2\tau} \|[\bV \text{Diag}(\alpha)\bV^{\top} ]_{-}\|^2_{F}.$$

   replacing $L^*$ by the conjugate in \Cref{lem:conjugate-hs} yields the desired result. 
\end{proof}    

\begin{lemma}
\label{lem:conjugate-hs}
Let $L:[-M, M]^{2n} \to \mathbb{R}$ be a loss function defined as 

\begin{equation}
    L(x_1, ..., x_n, y_1, ..., y_n) = \frac{e^{\varepsilon}}{n}\sum_{i=1}^ng(y_i) - \frac{1}{n}\sum_{i=1}^ng(x_i).
\end{equation}

Then its convex conjugate $L^{*}:\mathbb{R}^{2n}(\alpha) \to \mathbb{R}$ is defined as 

\begin{equation}
    \label{eq:loss-conjugate}
    L^*(\alpha_x, \alpha_y) = M (\|\alpha_x + \mathds{1}/n\|_1 + \|\alpha_y - \frac{e^{\varepsilon}}{n}\mathds{1}\|_1)
\end{equation}
\end{lemma}

\begin{proof}

The convex conjugate $L^*(\alpha)$ for function $L(x)$ is defined as:
\begin{equation*}
    L^*(\alpha) = \sup_{x \in [-M, M]^{2n}} \left( \langle \alpha, x \rangle - L(x) \right)
\end{equation*}

Substituting $L(x)$ into the definition and grouping terms by each $x_i$:
\begin{align*}
    L^*(\alpha) &= \sup_{x \in [-M, M]^{2n}} \left( \sum_{i=1}^{2n} \alpha_i x_i - \left( \frac{e^{\varepsilon}}{n} \sum_{i=1}^{n} x_i - \frac{1}{n} \sum_{i=n+1}^{2n} x_i \right) \right) \\
    &= \sup_{x \in [-M, M]^{2n}} \left( \sum_{i=1}^{n} \alpha_i x_i - \frac{e^{\varepsilon}}{n} \sum_{i=1}^{n} x_i + \sum_{i=n+1}^{2n} \alpha_i x_i + \frac{1}{n} \sum_{i=n+1}^{2n} x_i \right) \\
    &= \sup_{x \in [-M, M]^{2n}} \left( \sum_{i=1}^{n} \left(\alpha_i - \frac{e^{\varepsilon}}{n}\right)x_i + \sum_{i=n+1}^{2n} \left(\alpha_i + \frac{1}{n}\right)x_i \right)
\end{align*}

Since the objective is a sum of separable terms, the supremum of the sum is the sum of the suprema:
\begin{equation*}
    L^*(\alpha) = \sum_{i=1}^{n} \sup_{x_i \in [-M, M]} \left(\alpha_i - \frac{e^{\varepsilon}}{n}\right)x_i + \sum_{i=n+1}^{2n} \sup_{x_i \in [-M, M]} \left(\alpha_i + \frac{1}{n}\right)x_i
\end{equation*}
For any constant $c$, the solution to $\sup_{z \in [-M, M]} (c \cdot z)$ is $M \cdot |c|$. Applying this, the supremum for the first term is $M \left|\alpha_i - \frac{e^{\varepsilon}}{n}\right|$ and for the second is $M \left|\alpha_i + \frac{1}{n}\right|$.

Substituting these back and simplifying gives the final result in both summation and vector notation:
\begin{align*}
    L^*(\alpha) &= \sum_{i=1}^{n} M \left|\alpha_i - \frac{e^{\varepsilon}}{n}\right| + \sum_{i=n+1}^{2n} M \left|\alpha_i + \frac{1}{n}\right| \\
    &= M \left( \sum_{i=1}^{n} \left|\alpha_i - \frac{e^{\varepsilon}}{n}\right| + \sum_{i=n+1}^{2n} \left|\alpha_i + \frac{1}{n}\right| \right) \\
    &= M \left( \left\|\alpha_y - \frac{e^{\varepsilon}}{n}\mathds{1}\right\|_1 + \left\|\alpha_x + \frac{1}{n}\mathds{1}\right\|_1 \right).
\end{align*}

\end{proof}

 \paragraph{Limitations of the Direct Optimization Approach. } 
 
 In practice, we do not use the direct optimization approach described above. The RKHS function class it employs (see \Cref{eq:hstau}) is not suited for representing the potentially discontinuous Hockey-Stick witness function from \Cref{eq:hswitness}.

\Cref{fig:witness-direct} empirically illustrates this limitation using two Gaussian distributions, $P = \mathcal{N}(0,1)$ and $Q = \mathcal{N}(3,1)$. The figure shows that the direct method fails to capture the sharp, non-linear structure of the true witness function, yielding an overly smooth approximation. In contrast, the threshold-based method from \Cref{sec:statistics} provides a much better fit. We present this direct method nonetheless, as it could be promising for future work exploring more suitable function classes.

\bigskip

\begin{figure}[h]
    \centering    \includegraphics[width=0.7\linewidth]{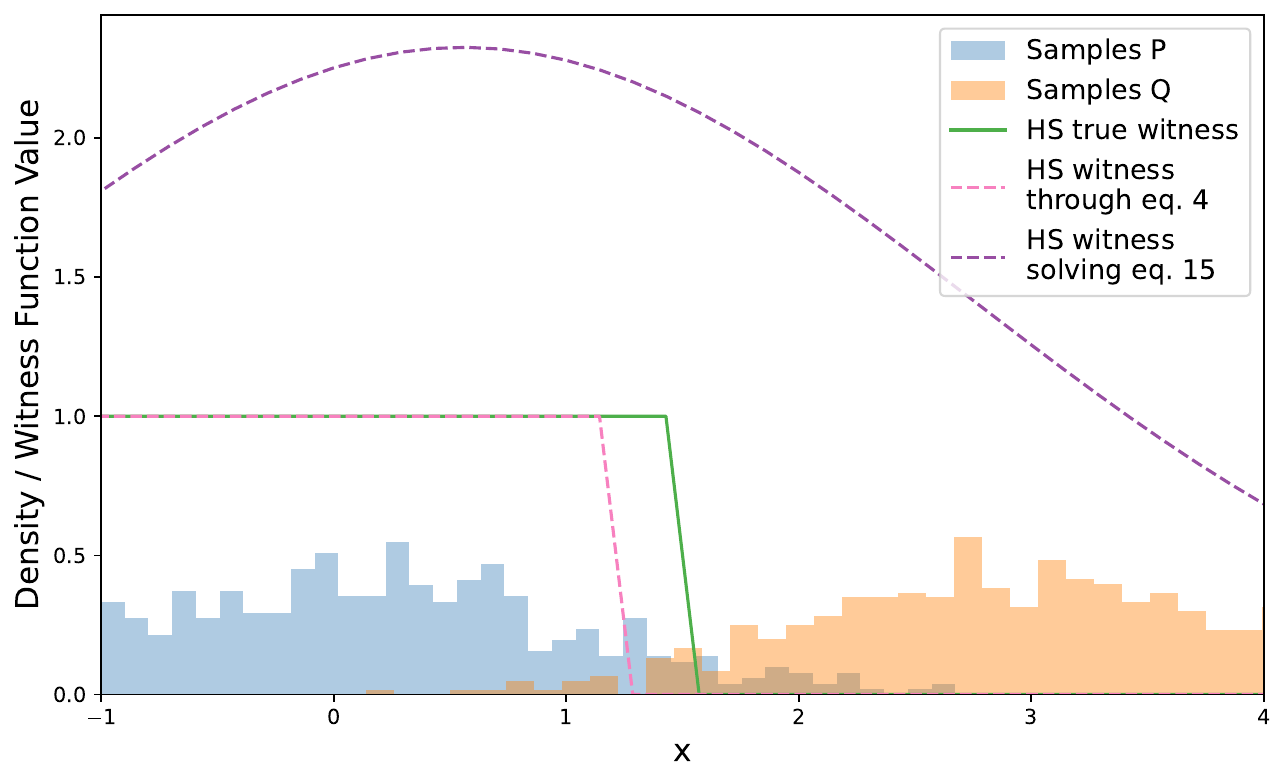}
    \caption{Hockey-Stick witness function plots for two Gaussian distributions, $P=\mathcal{N}(0,1)$ and $Q=\mathcal{N}(3,1)$. The true witness $g^\star(x) = \mathds{1}(\frac{dP}{dQ}(x)\geq \gamma)$ is plotted in green (HS true witness). The witness estimated via our proposed methods, i.e., $\mathds{1}({\widehat{r}_\lambda(x) \geq \gamma})$ as in \Cref{eq:hs_estimator} is plotted in pink using the estimator of \Cref{eq:r_lambda}. The witness obtained via the direct optimization framework of \Cref{sec:alternatives-hs} (\Cref{eq:dual}) is plotted in purple. HS witness estimation solving \Cref{eq:dual} fails to capture the sharp discontinuity due to the smoothing effect of the RKHS function class. }
    \label{fig:witness-direct}
\end{figure}

\clearpage

\section{Proofs}
\label{sec:proofs}

In this section, we provide the proofs of all results presented in \Cref{sec:statistics,sec:testfdiv}: 
\begin{itemize}
\item \Cref{proof1}: proof of \Cref{res:witness_convergence},
\item \Cref{proof2}: proof of \Cref{res:fdiv_convergence},
\item \Cref{proof3}: proof of \Cref{res:asymptotic_power},
\item \Cref{proof4}: proof of \Cref{res:non_asymptotic_power}.
\end{itemize}

We start by recalling some notation.
We write $a\lesssim b$ if there exists some constant $C>0$ (in our setting independent of the sample sizes) such that $a\leq C b$.
We note that for any constant $a\geq 1$ and $\eta\in(0,e^{-1})$, we have
$$
\ln(a/\eta) 
= 
\ln(a) + \ln(1/\eta)
\leq (\ln(a) + 1) \ln(1/\eta)
\lesssim
\ln(1/\eta).
$$
As such, events holding with probability $1-\eta/a$ for any $a\geq 1$ result in the same rate $\ln(1/\eta)$.
Rigorously, if we have $K$ events each holding with probability $1-\eta/K$, then they all hold simultaneously with probability $1-\eta$ and the rate for each is 
$
\ln(K/\eta)
\lesssim 
\ln(1/\eta)
$.
Knowing this, for simplicity, we simply write that each event holds \emph{with probability $1-\eta$} instead of $1-\eta/K$ for a specified number of events $K$. This allows to be flexible in the number of events that are considered (which changes often), and it results in the same final rate in any case.

We emphasize the relevance of the work of \citet{xu2022importance}, who establish convergence results for density ratio estimation relying on the kernel-based unconstrained least-squares importance fitting method of \citet{kanamori2012statistical}.
While their results are of great independent interest, they are not applicable to our setting, hence, justifying the need for our thorough convergence analysis. 
Their rates are derived under regularity assumptions which differ from ours, and hence, both rates are not directly comparable.

\subsection{Proof of \Cref{res:witness_convergence}}
\label{proof1}

Recall that \Cref{res:witness_convergence} states that, 
under \Cref{assump1},
for $\lambda_{N\!,\theta} = N^{-1/2(\theta+1) }$ where $N=\min(m,n)$, it holds with probability at least $1-\eta$ that
\begin{equation}
\label{dpdq-r1}
\left\|
\frac{dP}{dQ}-\widehat{r}_{\lambda_{N\!,\theta}}
\right\|_{\mathcal H}
\lesssim 
N^{-\frac{\theta-1}{2(\theta+1)}}
\sqrt{\ln(1/\eta)}
\end{equation}
for $\eta\in(0,e^{-1})$, with the assumption that the kernel is bounded by $K$ everywhere and that $\mu_P-\mu_Q\in\mathrm{Ran}(\Sigma_Q^\theta )$ for some $\theta\in (1,2]$.

\begin{proof}[Proof of \Cref{res:witness_convergence}]
Recall that the kernel mean embeddings $\mu_P$ and $\mu_Q$, as well as the covariance operator $\Sigma_Q$, are defined in \Cref{subsec:kernel}.

We follow the work of \citet{chen2024regularized}.
Let
$$
\mathcal F(h) 
\coloneqq 
\int h \,\mathrm{d}P
-\int \left(h+\frac{h^2}{4}\right) \,\mathrm{d}Q
$$
and
$$
\widehat{\mathcal F}(h) 
\coloneqq 
\frac{1}{m}\sum_{i=1}^m h(X_i) 
- 
\frac{1}{n}\sum_{i=1}^n \left(h(Y_i)+\frac{h(Y_i)^2}{4}\right).
$$
Let $N\coloneqq \min(m,n)$.
Then, let
\begin{align*}
h_0 
&\coloneqq 
\argmax_{h\in\mathcal H}
\mathcal F(h) 
= 2\left(\frac{\mathrm{d}P}{\mathrm{d}Q}-1\right)
\overset{(\star)}{=} 2\Sigma_Q^{-1}(\mu_P-\mu_Q),\\
h_\lambda 
&\coloneqq 
\argmax_{h\in\mathcal H}
\mathcal F(h)
- \frac{\lambda}{4}\|h\|_{\mathcal H}^2
= 2\big(\Sigma_Q+\lambda I\big)^{-1}(\mu_P-\mu_Q),\\
\widehat{h}_\lambda 
&\coloneqq 
\argmax_{h\in\mathcal H}
\widehat{\mathcal F}(h) 
- \frac{\lambda}{4}\|h\|_{\mathcal H}^2
= 2\big(\Sigma_{\widehat Q}+\lambda I\big)^{-1}(\mu_{\widehat P}-\mu_{\widehat Q}),
\end{align*}
where $(\star)$ holds since, for all $f\in\mathcal H$, it holds that
$$
\left\langle f, \Sigma_Q \left( \frac{\mathrm{d}P}{\mathrm{d}Q} - 1 \right)\right\rangle
=
\int f  \left[ \frac{\mathrm{d}P}{\mathrm{d}Q} - 1 \right]  \,\mathrm{d}Q
=
\int f  \,\mathrm{d}P - \int f  \,\mathrm{d}Q
=
\langle f, \mu_P-\mu_Q\rangle,
$$
implying that $\Sigma_Q \!\left( \frac{\mathrm{d}P}{\mathrm{d}Q} - 1 \right)=\mu_P-\mu_Q$.
By \citet[Proposition 6.1]{chen2024regularized}, we have
$$
\widehat{h}_\lambda
= 
2\left(\widehat{r}_\lambda-1\right).
$$
Hence, proving \Cref{dpdq-r1}, is equivalent to proving that, with probability at least $1-\eta$, it holds that
\begin{equation}
\label{eq:witness_convergence_2}
\big\|
h_0-\widehat{h}_{\lambda_{N\!,\theta}}
\big\|_{\mathcal H}
\lesssim 
N^{-\frac{\theta-1}{2(\theta+1)}} 
\sqrt{\ln(1/\eta)}
\end{equation}
where
$\lambda_{N\!,\theta} = N^{-1/2(\theta+1) }$.
We start with the decomposition
$$
\big\|\widehat{h}_{\lambda} - h_0\big\|_{\mathcal H}
\leq 
\big\|{h}_{\lambda} - h_0\big\|_{\mathcal H}
+
\big\|\widehat{h}_{\lambda} - h_{\lambda}\big\|_{\mathcal H}.
$$
Assume that $\mu_P-\mu_Q\in\mathrm{Ran}(\Sigma_Q^\theta )$ for some $\theta\in(1,2]$.
Then, applying \Cref{lem_h_1,lem_h_2}, we obtain
\begin{equation}
\label{aa1}
\big\|\widehat{h}_{\lambda} - h_0\big\|_{\mathcal H}
\lesssim
\lambda^{\theta-1}
+
\frac{1}{\lambda^2}\sqrt{\frac{{\ln(1/\eta)}}{{N}}}.
\end{equation}
holding with proability at least $1-\eta$.
Letting $\lambda_{N\!,\theta}=N^{-1/2(\theta+1) }$ in order to equate both terms, we get
$$
\big\|\widehat{h}_{\lambda_{N\!,\theta}} - h_0\big\|_{\mathcal H}
\lesssim 
N^{-\frac{\theta-1}{2(\theta+1)}} 
\sqrt{\ln(1/\eta)}.
$$
also holding with proability at least $1-\eta$.
This concludes the proof of \Cref{res:witness_convergence}.
\end{proof}

\begin{proposition}
\label{lem_h_1}
Assume that the kernel is bounded by $K$ and that $\mu_P-\mu_Q\in\mathrm{Ran}(\Sigma_Q^\theta )$ for some $\theta\in(1,2]$.
Then, we have
$$
\big\|{h}_{\lambda} - h_0\big\|_{\mathcal H}
\lesssim \lambda^{\theta-1}.
$$
\end{proposition}

\begin{proof}[Proof of \Cref{lem_h_1}]
First, note that
$$
\big\|{h}_{\lambda} - h_0\big\|_{\mathcal H}
= 2\left\|\left(\big(\Sigma_{Q}+\lambda I\big)^{-1}-\Sigma_{Q}^{-1}\right)(\mu_P-\mu_Q)\right\|_{\mathcal H}
$$
does not necessarily converge to 0 as $\lambda$ tends to 0 because the operator $\Sigma_Q^{-1}$ is unbounded.
Now, assuming that $\mu_P-\mu_Q\in\mathrm{Ran}(\Sigma_Q^\theta )$ for some $\theta\in(1,2]$, there exists $w\in\mathcal{H}$ such that $\mu_P-\mu_Q = \Sigma_Q^\theta  w$.
Let $\{\sigma_i,e_i\}_{i=1}^\infty$ be an orthonormal eigenbasis of $\Sigma_Q$ with $(\sigma_i)_{i=1}^\infty$ all non-negative.
Then, writing $w_i=\langle w, e_i\rangle$, we have
$$
\mu_P-\mu_Q
=
\sum_{i=1}^\infty
\sigma_i^\theta w_i e_i
$$
where $\sum_{i=1}^\infty w_i^2 <\infty$.
We obtain
\begin{align*}
\frac{1}{2}
\|{h}_{\lambda} - h_0\|_{\mathcal H}
&\leq 
\left\|\left(\big(\Sigma_{Q}+\lambda I\big)^{-1}-\Sigma_{Q}^{-1}\right)(\mu_P-\mu_Q)\right\|_{\mathcal H}\\ 
&=
\left\|
\sum_{i=1}^\infty
\left(
\frac{1}{\sigma_i+\lambda}-\frac{1}{\sigma_i}
\right)
\sigma_i^\theta w_i e_i
\right\|_{\mathcal H}\\
&\leq 
\left\|
\sum_{i=1}^\infty
\frac{\lambda}{(\sigma_i+\lambda)\sigma_i}
\sigma_i^\theta w_i e_i
\right\|_{\mathcal H}\\
&=
\left\|
\sum_{i=1}^\infty
\frac{\lambda^{2-\theta}\sigma_i^{\theta-1}}{(\sigma_i+\lambda)}
\lambda^{\theta-1} w_i e_i
\right\|_{\mathcal H}\\
&\leq
\lambda^{\theta-1} 
\left\|
\sum_{i=1}^\infty
w_i e_i
\right\|_{\mathcal H}\\
&=
\lambda^{\theta-1} 
\left\|
w
\right\|_{\mathcal H}\\
&\lesssim \lambda^{\theta-1}
\end{align*}
using the AM-GM inequality to get that 
${\lambda^{2-\theta}\sigma_i^{\theta-1}}\leq (2-\theta)\lambda + (\theta-1)\sigma_i\leq \lambda+\sigma_i$ for $\theta\in(1,2]$. This completes the proof.
\end{proof}

\begin{proposition}
\label{lem_h_2}
Let $N=\min(m,n)$, $\eta\in(0,e^{-1})$ and $\lambda\leq 1$.
With probability at least $1-\eta$, it holds that
$$
\big\|h_{\lambda}- \widehat{h}_{\lambda}\big\|_{\mathcal H}
\lesssim
\frac{1}{\lambda^2}\sqrt{\frac{{\ln(1/\eta)}}{{N}}}.
$$
\end{proposition}

\begin{proof}[Proof of \Cref{lem_h_2}]
We use the decomposition
\begin{align*}
    \frac{1}{2}\big\|h_{\lambda}- \widehat{h}_{\lambda}\big\|_{\mathcal H}
=~
&\big\|
\big(\Sigma_Q+\lambda I\big)^{-1}(\mu_P-\mu_Q)
-
\big(\Sigma_{\widehat Q}+\lambda I\big)^{-1}(\mu_{\widehat P}-\mu_{\widehat Q})
\big\|_{\mathcal H}
\\
\leq~
&\big\|
\big(\Sigma_Q+\lambda I\big)^{-1}(\mu_P-\mu_Q)
-
\big(\Sigma_{\widehat Q}+\lambda I\big)^{-1}(\mu_P-\mu_Q)
\big\|_{\mathcal H}\\
&+
\big\|
\big(\Sigma_{\widehat Q}+\lambda I\big)^{-1}(\mu_P-\mu_Q)
-
\big(\Sigma_{\widehat Q}+\lambda I\big)^{-1}(\mu_{\widehat P}-\mu_{\widehat Q})
\big\|_{\mathcal H}.
\end{align*}
The two terms can be bounded in a similar fashion to \citet[Equations 70 and 71]{chen2024regularized}.
The first term is bounded as
\begin{align*}
&\big\|
\big(\Sigma_Q+\lambda I\big)^{-1}(\mu_P-\mu_Q)
-
\big(\Sigma_{\widehat Q}+\lambda I\big)^{-1}(\mu_P-\mu_Q)
\big\|_{\mathcal H}\\
\leq~ &\Big\|
    \big(\Sigma_Q+\lambda I\big)^{-1}
    -
    \big(\Sigma_{\widehat Q}+\lambda I\big)^{-1}
\Big\|_{\mathrm{HS}}
\big\|\mu_P-\mu_Q
\big\|_{\mathcal H}\\
\lesssim~ &\Big\|
    \big(\Sigma_Q+\lambda I\big)^{-1}
    \left(
        \big(\Sigma_{\widehat Q}+\lambda I\big)
        -
        \big(\Sigma_{Q}+\lambda I\big)
    \right)
    \big(\Sigma_{\widehat Q}+\lambda I\big)^{-1}
\Big\|_{\mathrm{HS}}\\
\lesssim~ &\Big\|
    \big(\Sigma_Q+\lambda I\big)^{-1}
\Big\|_{\mathrm{op}}
\Big\|
    \big(\Sigma_{\widehat Q}+\lambda I\big)^{-1}
\Big\|_{\mathrm{op}}
\Big\|
    \Sigma_{\widehat Q}-\Sigma_Q
\Big\|_{\mathrm{HS}}\\
\lesssim~ &\frac{1}{\lambda^2}
\Big\|
    \Sigma_{\widehat Q}-\Sigma_Q
\Big\|_{\mathrm{HS}}\\
\lesssim~ &\frac{1}{\lambda^2}
\sqrt{\frac{\ln(1/\eta)}{N}}
\end{align*}
using the fact that $\big\|\mu_P-\mu_Q
\big\|_{\mathcal H}\leq 2\sqrt{K}$, and
where the last equality holds with probability at least $1-\eta/2$ by \citet[Lemma B.8]{chen2024regularized}.
Using that same reference, we can bound the second term as
\begin{align*}
&\big\|
\big(\Sigma_{\widehat Q}+\lambda I\big)^{-1}(\mu_P-\mu_Q)
-
\big(\Sigma_{\widehat Q}+\lambda I\big)^{-1}(\mu_{\widehat P}-\mu_{\widehat Q})
\big\|_{\mathcal H}\\
\leq~ &\Big\|
    \big(\Sigma_{\widehat Q}+\lambda I\big)^{-1}
\Big\|_{\mathrm{op}}
\big\|
(\mu_{P}-\mu_{Q})
-
(\mu_{\widehat P}-\mu_{\widehat Q})
\big\|_{\mathcal H}\\
\lesssim~ &\frac{1}{\lambda}
\left(
    \big\|
    \mu_{P}-\mu_{\widehat P}
    \big\|_{\mathcal H}
    +
    \big\|
    \mu_{Q}-\mu_{\widehat Q}
    \big\|_{\mathcal H}
\right)\\
\lesssim~ &\frac{1}{\lambda}
\sqrt{\frac{\ln(1/\eta)}{N}}
\end{align*}
where the last equality holds with probability at least $1-\eta/2$. 
We conclude that, with probability at least $1-\eta$, it holds that
\begin{align*}
\big\|h_{\lambda}- \widehat{h}_{\lambda}\big\|_{\mathcal H}
\lesssim 
\left(
\frac{1}{\lambda^2}
+\frac{1}{\lambda}
\right)
\sqrt{\frac{\ln(1/\eta)}{N}}
\lesssim \frac{1}{\lambda^2}
\sqrt{\frac{\ln(1/\eta)}{N}}
\end{align*}
since $\lambda\leq 1$.
\end{proof}

\subsection{Proof of \Cref{res:fdiv_convergence}}
\label{proof2}

Recall that \Cref{res:fdiv_convergence} states that,
under \Cref{assump1,assump2},
for $\lambda_{N\!,\theta} = N^{-1/2(\theta+1) }$, we have
$$
\big|
\widehat{D}_{f,\lambda_{N\!,\theta}}-D_f
\big|
\lesssim 
\left(
N^{-\frac{\theta-1}{2(\theta+1)}} 
+ \widetilde{N}^{-1/2}
\right)
\sqrt{\ln(1/\eta)}
$$
holding with probability at least $1-\eta$ over the $N$ samples used for estimating the ratio $\frac{dP}{dQ}$.

\begin{proof}[Proof of \Cref{res:fdiv_convergence} for $f$-divergences with $f$ twice continuously differentiable on $(0,\infty)$]
    In that setting, we assume that $\frac{dP}{dQ}$ belongs to $[c,C]$, for some $0<c<1<C<\infty$ and that $\mathsf{C}N^{-\frac{\theta-1}{2(\theta+1)}} \!\sqrt{\ln(1/\eta)} \leq c/2$ with $\mathsf{C}$ as in \Cref{C}.
Using \Cref{res:witness_convergence}, we can then ensure that, with probability at least $1-\eta$, for any $u\in\mathcal Z$, we have $\frac{dP}{dQ}(u)$ and $\widehat{r}_{\lambda_{N\!,\theta}}(u)$ belonging to $[\overline{c},\overline{C}]$ for $\overline{c}=c/2$ and $\overline{C}=C+c/2$.

First, we prove that $f'$ and $f^*\circ f'$ are both Lipschitz on $[\overline{c},\overline{C}]$.
Since $f'$ is continuously differentiable, it is also Lipschitz by the Mean Value Theorem, as for all $x,y\in[\overline{c},\overline{C}]$, we have
\begin{equation}
\label{lip}
|f'(x)-f'(y)|
~\leq~
|x-y| \sup_{\xi\in[\overline{c},\overline{C}]}|f''(\xi)|
~\lesssim ~
|x-y|
\end{equation}
since the supremum is finite as $f''$ is continuous on $[\overline{c},\overline{C}]$.
From the fact that $f$ is twice continuously differentiable and convex (the function defining an $f$-divergence is convex by definition), we deduce that the conjugate $f^*$ is continuously differentiable and strictly convex on the range of $f'$ with $(f^*)'(f'(x)) = x$ for all $x$.
Applying the Mean Value Theorem, we have
\begin{equation}
\label{lip2}
|f^*(f'(x))-f^*(f'(y))|
~\leq~
|f'(x)-f'(y)| \sup_{\xi\in\mathrm{Im}\big(f'\!\big|_{[\overline{c},\overline{C}]}\big)
}|(f^*)'(\xi)|~\lesssim ~
|x-y|
\end{equation}
as the supremum is equal to $\overline{C}$ and using \Cref{lip}.
This proves that $f'$ and $f^*\circ f'$ are both Lipschitz on $[\overline{c},\overline{C}]$.

We then have
\begin{equation}
\label{b1}
\begin{aligned}
&\big|
\widehat{D}_{f,\lambda_{N\!,\theta}}-D_f
\big| \\
\leq~
&\Bigg|
\frac{1}{\widetilde m}\sum_{i=1}^{\widetilde m} f'\!\big(\widehat{r}_{\lambda_{N\!,\theta}}(\widetilde{X}_i)\big) 
-
\mathbb{E}_{Z\sim P}\!\left[f'\!\left(\frac{dP}{dQ}(Z)\right)\right]
\Bigg|
+ 
\Bigg|
\frac{1}{\widetilde n}\sum_{j=1}^{\widetilde n} f^*\!\big(f'\!\big(\widehat{r}_{\lambda_{N\!,\theta}}(\widetilde{Y}_j)\big)\big)
-
\mathbb{E}_{W\sim Q}\!\left[f^*\!\left(f'\!\left(\frac{dP}{dQ}(W)\right)\right)\right]
\Bigg|.
\end{aligned}
\end{equation}
Since $f'$ and $f^*\circ f'$ are both Lipschitz, the two terms can be bounded similarly, hence, we focus on the first one, which is
\begin{align}
&\Bigg|
\frac{1}{\widetilde m}\sum_{i=1}^{\widetilde m} f'\!\big(\widehat{r}_{\lambda_{N\!,\theta}}(\widetilde{X}_i)\big) 
-
\mathbb{E}_{Z\sim P}\!\left[f'\!\left(\frac{dP}{dQ}(Z)\right)\right]
\Bigg|
\label{b2}
\\ 
\leq~
&
\frac{1}{\widetilde m}\sum_{i=1}^{\widetilde m} 
\Bigg|
f'\!\big(\widehat{r}_{\lambda_{N\!,\theta}}(\widetilde{X}_i)\big) 
-
f'\!\left(\frac{dP}{dQ}(\widetilde{X}_i)\right)
\Bigg|
+
\Bigg|
\frac{1}{\widetilde m}\sum_{i=1}^{\widetilde m} 
f'\!\left(\frac{dP}{dQ}(\widetilde{X}_i)\right)
-
\mathbb{E}_{Z\sim P}\!\left[f'\!\left(\frac{dP}{dQ}(Z)\right)\right]
\Bigg|
\label{b3}
\\ 
\leq~
&
\frac{1}{\widetilde m}\sum_{i=1}^{\widetilde m} 
\Bigg|
f'\!\big(\widehat{r}_{\lambda_{N\!,\theta}}(\widetilde{X}_i)\big) 
-
f'\!\left(\frac{dP}{dQ}(\widetilde{X}_i)\right)
\Bigg|
+
\sqrt{\frac{\ln(1/\eta)}{\widetilde N}}
\label{b4}
\\
\lesssim~
&\frac{1}{\widetilde m}\sum_{i=1}^{\widetilde m} 
\Bigg|
\widehat{r}_{\lambda_{N\!,\theta}}(\widetilde{X}_i)
-
\frac{dP}{dQ}(\widetilde{X}_i)
\Bigg|
+
\sqrt{\frac{\ln(1/\eta)}{\widetilde N}}
\\
\lesssim~
&
\left\|
\frac{dP}{dQ}-\widehat{r}_{\lambda_{N\!,\theta}}
\right\|_{\mathcal H}
+
\sqrt{\frac{\ln(1/\eta)}{\widetilde N}}\\
\lesssim~
&\left(
N^{-\frac{\theta-1}{2(\theta+1)}} 
+ \widetilde{N}^{-1/2}
\right)
\sqrt{\ln(1/\eta)}
\end{align}
where we have used Hoeffding's inequality and the fact that $f'$ takes bounded values on $[\overline{c},\overline{C}]$, as well as the rate of \Cref{dpdq-r1} from \Cref{res:witness_convergence} in the last inequality.
This gives the desired result.

\end{proof}

\begin{proof}[Proof of \Cref{res:fdiv_convergence} for the Total-Variation]
In that setting, for $X\sim P$ and $Y\sim Q$, we assume that the densities of $\widehat{r}_{\lambda_{N\!,\theta}}(X)$ and of $\widehat{r}_{\lambda_{N\!,\theta}}(Y)$ exist and are bounded on $[\gamma/2,3\gamma/2]$ by some $G>0$. We also assume that $\mathsf{C}N^{-\frac{\theta-1}{2(\theta+1)}} \!\sqrt{\ln(1/\eta)} \leq \gamma/2$ with $\mathsf{C}$ as in \Cref{C}.
For the Total-Variation, we have $\gamma=1$.

For the Total-Variation, we have $f(t) = \frac{1}{2}|t-1|$ which is not differentiable at $\gamma=1$. 
Everywhere else, we have $f'(t) = \frac{1}{2}\sign(t-1)$ and the optimal witness function indeed is $\frac{1}{2}\sign\!\left(\frac{dP}{dQ}-1\right)$. 
We choose to use the convention that $\sign(0) = 1$ and extend $f'(1)=\frac{1}{2}$. 
The convex conjugate $f^*$ is simply the identity on $\left\{-\frac{1}{2},\frac{1}{2}\right\}$.

The reasoning of \Cref{b1,b2,b3,b4} still holds, we are then left with the task of bounding the following term (for which we leverage \Cref{res:witness_convergence})
\begin{equation}
\label{c1}
\begin{aligned}
\frac{1}{\widetilde m}\sum_{i=1}^{\widetilde m} 
\Bigg|
f'\!\big(\widehat{r}_{\lambda_{N\!,\theta}}(\widetilde{X}_i)\big) 
-
f'\!\left(\frac{dP}{dQ}(\widetilde{X}_i)\right)
\Bigg|
&= 
\frac{1}{\widetilde m}\sum_{i=1}^{\widetilde m} 
\mathds{1}\!\left(
    f'\!\big(\widehat{r}_{\lambda_{N\!,\theta}}(\widetilde{X}_i)\big) 
\neq
f'\!\left(\frac{dP}{dQ}(\widetilde{X}_i)\right)
\right)\\ 
&=
\frac{1}{\widetilde m}\sum_{i=1}^{\widetilde m} 
\mathds{1}\!\left(
\gamma\in\left(
    \min\left\{\widehat{r}_{\lambda_{N\!,\theta}}(\widetilde{X}_i),\frac{dP}{dQ}(\widetilde{X}_i)\right\}
,
\max\left\{\widehat{r}_{\lambda_{N\!,\theta}}(\widetilde{X}_i),\frac{dP}{dQ}(\widetilde{X}_i)\right\}
\right]
\right)\\ 
&\leq
\frac{1}{\widetilde m}\sum_{i=1}^{\widetilde m} 
\mathds{1}\!\left(
    \big|\widehat{r}_{\lambda_{N\!,\theta}}(\widetilde{X}_i)-\gamma\big|
\leq 
\mathsf{C}N^{-\frac{\theta-1}{2(\theta+1)}}
\sqrt{\ln(1/\eta)}
\right)\\ 
&\leq
\mathbb{P}_{\widetilde{X}\sim P}\!\left(
\big|\widehat{r}_{\lambda_{N\!,\theta}}(\widetilde{X})-\gamma\big|
\leq
\mathsf{C}N^{-\frac{\theta-1}{2(\theta+1)}}
\sqrt{\ln(1/\eta)}
\right) + \sqrt{\frac{\ln(1/\eta)}{\widetilde N}}\\
&\leq 2G
\mathsf{C}N^{-\frac{\theta-1}{2(\theta+1)}}\sqrt{\ln(1/\eta)}
+ \sqrt{\frac{\ln(1/\eta)}{\widetilde N}}\\
&\lesssim
\left(
N^{-\frac{\theta-1}{2(\theta+1)}} 
+ \widetilde{N}^{-1/2}
\right)
\sqrt{\ln(1/\eta)}
\end{aligned}
\end{equation}
since $\mathsf{C}N^{-\frac{\theta-1}{2(\theta+1)}} \!\sqrt{\ln(1/\eta)} \leq \gamma/2$ implies that $\big[
\gamma-
\mathsf{C}N^{-\frac{\theta-1}{2(\theta+1)}}\sqrt{\ln(1/\eta)}
,
\gamma+
\mathsf{C}N^{-\frac{\theta-1}{2(\theta+1)}}\sqrt{\ln(1/\eta)}
\big]\subseteq [\gamma/2,3\gamma/2]$ on which the densities of $\widehat{r}_{\lambda_{N\!,\theta}}(X)$ and of $\widehat{r}_{\lambda_{N\!,\theta}}(Y)$ for $X\sim P$ and $Y\sim Q$,
are bounded by $G$.
This concludes the proof for the case of the Total-Variation.

\end{proof}

\begin{proof}[Proof of \Cref{res:fdiv_convergence} for the Hockey-Stick divergence]
In that setting, the assumptions are the same as for the Total-Variation (i.e., see \Cref{assump2}).

For the Hockey-Stick, we have $f(t) = \max\big(t-\gamma,0\big)$ which is not differentiable at $\gamma$. 
Everywhere else, we have $f'(t) = \mathds{1}(t\geq\gamma)$ and the optimal witness function indeed is $\mathds{1}\!\left(\frac{dP}{dQ}\geq\gamma\right)$.
We extend the definition of $f'$ to $f'(\gamma)=1$.
The convex conjugate is simply equal to $f^*(u)=\gamma u$. 

The exact same reasoning of \Cref{c1} holds using the $\gamma$ parameter of the Hockey-Stick divergence.
This concludes the proof.
\end{proof}

\subsection{Proof of \Cref{res:asymptotic_power}}
\label{proof3}

Recall that \Cref{res:asymptotic_power} claims that the proposed permutation test (\Cref{eq:permutation_test}), 
under \Cref{assump1,assump2},
is consistent in the sense that, for any fixed $P\neq Q$, we have
$$
\mathbb{P}(\mathrm{reject }~ H_0)\to 1
\textrm{ as } N,\widetilde{N}\to\infty.
$$

\begin{proof}[Proof of \Cref{res:asymptotic_power}]
Following \citet[Lemma 8 and Appendix E.6---see also \citealp{kim2025differentially}]{kim2023differentially}, it suffices to prove that the following two conditions are satisfied. 
First, for the fixed alternative $P\neq Q$, the estimator $\widehat{D}_{f,\lambda_{N\!,\theta}}$ needs to converge to some constant which is strictly positive.
This is a simple application of \Cref{res:fdiv_convergence} since it implies that $\widehat{D}_{f,\lambda_{N\!,\theta}}$ converges to $D_{f}(P\;\|\;Q)>0$.
The second condition is that, conditioned on the data, the permuted statistic converges to zero, where the randomness is taken with respect to the uniformly random draw of permutations. We prove this below.

Recall that we have access to samples 
$x_1,\dots,x_m,y_1,\dots,y_n$ 
used to estimate the ratio $\frac{dP}{dQ}$
and samples
$\widetilde{x}_{1},\dots,\widetilde{x}_{\widetilde m},\widetilde{y}_{1},\dots,\widetilde{y}_{\widetilde n}$ 
used to estimate the expectations in \Cref{eq:Df_estimator}.
A permutation $\sigma$ is then applied to all the data $\sigma(x_1,\dots,x_{m},y_1,\dots,y_{n},\widetilde{x}_1,\dots,\widetilde{x}_{\widetilde m},\widetilde{y}_1,\dots,\widetilde{y}_{\widetilde n}) = (z_1,\dots,z_{m+n},\widetilde{z}_1,\dots,\widetilde{z}_{\widetilde m+\widetilde n})$, the elements $z_1,\dots,z_{m+n}$ are used to estimate the ratio while the elements $\widetilde{z}_1,\dots,\widetilde{z}_{\widetilde m+\widetilde n}$ are used to estimate the expectations.

From \Cref{lastprop}, the RKHS norm of the witness function under permutations satisfies
\begin{equation*}
    \big\|\widehat{h}_{\lambda_{N\!,\theta}}\big\|_{\mathcal H}
\lesssim 
N^{-\frac{\theta-1}{2(\theta+1)}} 
\sqrt{\ln(1/\eta)}
\end{equation*}
with probability at least $1-\eta$.
We deduce that, 
with probability as least $1-\eta$,
for any $u\in\mathcal Z$, we have
$$
\widehat{h}_{\lambda_{N\!,\theta}}\!(u)\to0\textrm{ as }N\to\infty,
\qquad\textrm{ equivalently, }\qquad
\widehat{r}_{\lambda_{N\!,\theta}}\!(u)\to1\textrm{ as }N\to\infty.
$$
Hence, as $N$ tends to infinity, the permuted version of the estimator of \Cref{eq:Df_estimator} behaves as
\begin{align*}
\widehat{D}_{f,\lambda_{N\!,\theta}} 
&=
\frac{1}{\widetilde m}\sum_{i=1}^{\widetilde m} f'\!\big(\widehat{r}_{\lambda_{N\!,\theta}}(\widetilde{z}_i)\big) - \frac{1}{\widetilde n}\sum_{j=1}^{\widetilde n} f^*\!\big(f'\!\big(\widehat{r}_{\lambda_{N\!,\theta}}(\widetilde{z}_{m+j})\big)\big)\\
&\to 
\frac{1}{\widetilde m}\sum_{i=1}^{\widetilde m} f'\!\big(1\big) - \frac{1}{\widetilde n}\sum_{j=1}^{\widetilde n} f^*\!\big(f'\!\big(1\big)\big)\\
&= f'\!\big(1\big) - f^*\!\big(f'\!\big(1\big)\big)\\
&= f(1)\\ 
&=0
\end{align*}
holding with probability at least $1-\eta$, where we have applied the tight version of the Fenchel–Young inequality, and have leveraged the convergence results of \Cref{proof2}.
\end{proof}

\begin{proposition}
\label{lastprop}
Under permutation of the samples used to compute $\widehat{r}_{\lambda_{N\!,\theta}}$, with probability at least $1-\eta$, it holds that
\begin{equation*}
\big\|\widehat{r}_{\lambda_{N\!,\theta}}-1\big\|_{\mathcal H}
\lesssim 
N^{-\frac{\theta}{\theta+1}} 
\sqrt{\ln(1/\eta)}
\lesssim 
N^{-\frac{\theta-1}{2(\theta+1)}} 
\sqrt{\ln(1/\eta)}.
\end{equation*}
\end{proposition}

Before proceeding to the proof, we start by recalling some RKHS properties and notations.

\bigskip
\emph{RKHS properties.}
Let $\mathcal Z$ be a space.
Consider a reproducing kernel Hilbert space $\mathcal{H}$ (RKHS) consisting of functions $f\colon\mathcal{Z}\to\mathbb{R}$, with reproducing kernel $k:\mathcal Z \times \mathcal Z \to \mathbb{R}$ (which is bounded by $K$ everywhere by assumption).
By definition, the RKHS admits a feature map $\phi : \mathcal{Z} \to \mathcal{H}$ satisying the reproducing property $\langle f, \phi(x)\rangle_{\mathcal H} = f(x)$ for all $f\in\mathcal H$ and all $x\in\mathcal Z$, and $\langle \phi(x), \phi(y)\rangle_{\mathcal H} = k(x,y)$ for all $x,y\in\mathcal Z$.

For $\mathbf{x} = (x_1,\dots,x_m)\in\mathcal{Z}^m$, define the {feature operator}
$
\Phi_\mathbf{x} : \mathcal{H} \to \mathbb{R}^m
$
by
\[
(\Phi_\mathbf{x} f)_i = \big\langle f, \phi(x_i) \big\rangle_{\mathcal{H}} = f(x_i), \qquad i = 1, \dots, m
\]
for all $f\in\mathcal H$.
Then, its adjoint operator
$
\Phi_\mathbf{x}^* : \mathbb{R}^m \to \mathcal{H}
$
is given by
\[
\Phi_\mathbf{x}^* c = \sum_{i=1}^m c_i \, \phi(x_i)
\]
for all $c = (c_1, \dots, c_m)\in\mathbb{R}^m$.
The operator norms of $\Phi_{x}$ and $\Phi_{x}^*$ satisfy
\begin{equation}
\label{op1}
\big\|\Phi_{\mathbf x}\big\|_{\mathrm{op}}\leq \sqrt{m}\sqrt{K}
\qquad\textrm{ and }\qquad
\big\|\Phi_{\mathbf x}^*\big\|_{\mathrm{op}}\leq \sqrt{K}
\end{equation}
since
$$
\big\|\Phi_\mathbf{x}f\big\|_2
=
\sqrt{\sum_{i=1}^m\big|\big\langle f, \phi(x_i) \big\rangle_{\mathcal{H}}\big|^2}
\leq 
\big\|f\big\|_{\mathcal H}
\sqrt{\sum_{i=1}^m\big\|\phi(x_i)\big\|_{\mathcal H}^2}
=
\big\|f\big\|_{\mathcal H}
\sqrt{\sum_{i=1}^mk(x_i,x_i)}
\leq
\big\|f\big\|_{\mathcal H}\sqrt{Km} 
$$
and
$$
\big\|\Phi^*_{\mathbf{x}}c\big\|_{\mathcal H}
=
\left\|\sum_{i=1}^mc_i\phi(x_i)\right\|_{\mathcal H}
=
\sqrt{\sum_{i=1}^m\sum_{j=1}^mc_ic_jk(x_i,x_j)}
\leq 
\sqrt{K\left(\sum_{i=1}^mc_i\right)^2}
\leq 
\sqrt{K\sum_{i=1}^mc_i^2}
=
\sqrt{K}\|c\|_2.
$$
Consider samples $\mathbf{x}=(x_1,\dots,x_{m})$ and $\mathbf{y}=(y_{1},\dots,y_{n})$, and let $\mathds{1}_n$ denote an $(n\times 1)$ vector of ones.
We then have
\begin{equation*}
\Phi_\mathbf{x} \Phi_\mathbf{y}^*\mathds{1}_n
=
\Phi_\mathbf{x} 
\!\left(
\sum_{j=1}^n \phi(y_j)
\right),
\qquad
\big(\Phi_\mathbf{x} \Phi_\mathbf{y}^*\mathds{1}_n\big)_i 
=\left\langle 
\sum_{j=1}^n \phi(y_j),
\phi(x_i)
\right\rangle_{\!\!\mathcal H}
=
\sum_{j=1}^n
k(x_i,y_j)
=
\big(k_{\mathbf{x}\mathbf{y}}\mathds{1}_n\big)_i
\end{equation*}
for $i=1,\dots,m$, where
$k_{\mathbf{x}\mathbf{y}} = \big(k(x_i,y_{j})\big)_{1\leq i \leq m, 1\leq j \leq n} \in\mathbb{R}^{m\times n}$. We deduce that
\begin{equation}
\label{trick1}
\Phi_\mathbf{x} \Phi_\mathbf{y}^*\mathds{1}_n
=
k_{\mathbf{x}\mathbf{y}}\mathds{1}_n
\end{equation}
Similarly, considering some $u\in\mathcal Z$, and writing 
$k_{u\mathbf{x}} = \big(k(u,x_i)\big)_{i=1}^m \in\mathbb{R}^{1\times m}$, we obtain
\begin{equation}
\label{trick2}
\Phi_{u} \Phi_\mathbf{x}^*\mathds{1}_m
=
k_{{u}\mathbf{x}}\mathds{1}_m.
\end{equation}
Finally, when $m=1$, note that the following simple statement holds:
\begin{equation}
\label{equal}
\textrm{if } \quad
g(x) = \Phi_x f  
\quad
\textrm{ for all }\quad
x\in\mathcal Z
\quad
\textrm{ then }\quad
g = f.
\end{equation}
With these notions in place, we are now able to proceed with the proof of \Cref{lastprop}.

\begin{proof}[Proof of \Cref{lastprop}]
We show that the permuted estimated ratio $\widehat{r}_\lambda$ converges to 1 in RKHS norm, equivalently $\widehat{h}_\lambda$ converges to 0, under permutations.
For notation purposes, we let $\mathbf{z}_1=(z_1,\dots,z_{m})$ and $\mathbf{z}_2=(z_{m+1},\dots,z_{m+n})$.
For some $u\in\mathcal{Z}$, we also write
$k_{u\mathbf{z}_1} = \big(k(u,z_i)\big)_{i=1}^m \in\mathbb{R}^{1\times m}$,
$k_{u\mathbf{z}_2} = \big(k(u,z_{m+i})\big)_{i=1}^n \in\mathbb{R}^{1\times n}$, 
$k_{\mathbf{z}_1\!\mathbf{z}_1} = \big(k(z_i,z_{j})\big)_{1\leq i,j \leq m} \in\mathbb{R}^{m\times m}$,
and
$k_{\mathbf{z}_1\!\mathbf{z}_2} = \big(k(z_i,z_{m+j})\big)_{1\leq i \leq m, 1\leq j \leq n} \in\mathbb{R}^{m\times n}$,
.
For simplicity, we also write $L_{\mathbf{z}_1\!\mathbf{z}_2,\lambda}=m\lambda I + k_{\mathbf{z}_1\!\mathbf{z}_1}$.

\citet[Proposition 6.1]{chen2024regularized} provides a closed-form expression of $\widehat{h}_\lambda$, for any $u\in\mathcal Z$, it is equal to
\begin{equation*}
\widehat{h}_\lambda(u)
=
\frac{2}{n\lambda}k_{u\mathbf{z}_2}\mathds{1}_n
-
\frac{2}{m\lambda}k_{u\mathbf{z}_1}\mathds{1}_m
-
\frac{2}{n\lambda}k_{u\mathbf{z}_1}L_{\mathbf{z}_1\!\mathbf{z}_2,\lambda}^{-1}k_{\mathbf{z}_1\!\mathbf{z}_2}\mathds{1}_n
+
\frac{2}{m\lambda}k_{u\mathbf{z}_1}L_{\mathbf{z}_1\!\mathbf{z}_2,\lambda}^{-1}k_{\mathbf{z}_1\!\mathbf{z}_1}\mathds{1}_m.
\end{equation*}
With the notation introduced above, we obtain
\begin{equation*}
\widehat{h}_\lambda(u)
=
\Phi_{u}\left(
\frac{2}{n\lambda} \Phi_{\mathbf{z}_2}^*\mathds{1}_n
-
\frac{2}{m\lambda} \Phi_{\mathbf{z}_1}^*\mathds{1}_m
-
\frac{2}{n\lambda}\Phi_{\mathbf{z}_1}^*L_{\mathbf{z}_1\!\mathbf{z}_2,\lambda}^{-1}\Phi_{\mathbf{z}_1}\Phi_{\mathbf{z}_2}^*\mathds{1}_n
+
\frac{2}{m\lambda}\Phi_{\mathbf{z}_1}^*L_{\mathbf{z}_1\!\mathbf{z}_2,\lambda}^{-1}\Phi_{\mathbf{z}_1}\Phi_{\mathbf{z}_1}^*\mathds{1}_m
\right)
\end{equation*}
for all $u\in\mathcal Z$.
Applying \Cref{equal}, we deduce that
\begin{equation}
\label{h_lambda expression a}
\widehat{h}_\lambda
=
\frac{2}{n\lambda} \Phi_{\mathbf{z}_2}^*\mathds{1}_n
-
\frac{2}{m\lambda} \Phi_{\mathbf{z}_1}^*\mathds{1}_m
-
\frac{2}{n\lambda}\Phi_{\mathbf{z}_1}^*L_{\mathbf{z}_1\!\mathbf{z}_2,\lambda}^{-1}\Phi_{\mathbf{z}_1}\Phi_{\mathbf{z}_2}^*\mathds{1}_n
+
\frac{2}{m\lambda}\Phi_{\mathbf{z}_1}^*L_{\mathbf{z}_1\!\mathbf{z}_2,\lambda}^{-1}\Phi_{\mathbf{z}_1}\Phi_{\mathbf{z}_1}^*\mathds{1}_m
\end{equation}
Assuming that $\lambda\leq N^{-1/2}$, using \Cref{op1,trick1,trick2}, we can then bound the RKHS of the witness function under permutations as
\begin{equation}
\label{h_lambda norm bound}
\begin{aligned}
\big\|\widehat{h}_\lambda\big\|_\mathcal{H}
\lesssim~&
\left\|
\frac{1}{n\lambda} \Phi_{\mathbf{z}_2}^*\mathds{1}_n
-
\frac{1}{m\lambda} \Phi_{\mathbf{z}_1}^*\mathds{1}_m
-
\frac{1}{n\lambda}\Phi_{\mathbf{z}_1}^*L_{\mathbf{z}_1\!\mathbf{z}_2,\lambda}^{-1}\Phi_{\mathbf{z}_1}\Phi_{\mathbf{z}_2}^*\mathds{1}_n
+
\frac{1}{m\lambda}\Phi_{\mathbf{z}_1}^*L_{\mathbf{z}_1\!\mathbf{z}_2,\lambda}^{-1}\Phi_{\mathbf{z}_1}\Phi_{\mathbf{z}_1}^*\mathds{1}_m
\right\|_\mathcal{H}\\
\lesssim~&
\frac{1}{\lambda}\left\|
\frac{1}{n} \Phi_{\mathbf{z}_2}^*\mathds{1}_n
-
\frac{1}{m} \Phi_{\mathbf{z}_1}^*\mathds{1}_m
\right\|_\mathcal{H}
+
\frac{1}{\lambda}\left\|
\Phi_{\mathbf{z}_1}^*L_{\mathbf{z}_1\!\mathbf{z}_2,\lambda}^{-1}\Phi_{\mathbf{z}_1}
\!\left(
\frac{1}{n}\Phi_{\mathbf{z}_2}^*\mathds{1}_n
-
\frac{1}{m}\Phi_{\mathbf{z}_1}^*\mathds{1}_m
\right)
\right\|_\mathcal{H}\\
\lesssim~&
\frac{1}{\lambda}\left\|
\frac{1}{n} \Phi_{\mathbf{z}_2}^*\mathds{1}_n
-
\frac{1}{m} \Phi_{\mathbf{z}_1}^*\mathds{1}_m
\right\|_\mathcal{H}
+
\frac{1}{\lambda}\left\|
\Phi_{\mathbf{z}_1}^*
\right\|_\mathrm{op}
\left\|
L_{\mathbf{z}_1\!\mathbf{z}_2,\lambda}^{-1}\Phi_{\mathbf{z}_1}
\!\left(
\frac{1}{n}\Phi_{\mathbf{z}_2}^*\mathds{1}_n
-
\frac{1}{m}\Phi_{\mathbf{z}_1}^*\mathds{1}_m
\right)
\right\|_2\\
\lesssim~&
\frac{1}{\lambda}\left\|
\frac{1}{n} \Phi_{\mathbf{z}_2}^*\mathds{1}_n
-
\frac{1}{m} \Phi_{\mathbf{z}_1}^*\mathds{1}_m
\right\|_\mathcal{H}
+
\frac{1}{\lambda}
\left\|
\Phi_{\mathbf{z}_1}^*
\right\|_\mathrm{op}
\left\|
L_{\mathbf{z}_1\!\mathbf{z}_2,\lambda}^{-1}
\right\|_\mathrm{op}
\left\|
\Phi_{\mathbf{z}_1}
\!\left(
\frac{1}{n}\Phi_{\mathbf{z}_2}^*\mathds{1}_n
-
\frac{1}{m}\Phi_{\mathbf{z}_1}^*\mathds{1}_m
\right)
\right\|_2\\
\lesssim~&
\frac{1}{\lambda}\left\|
\frac{1}{n} \Phi_{\mathbf{z}_2}^*\mathds{1}_n
-
\frac{1}{m} \Phi_{\mathbf{z}_1}^*\mathds{1}_m
\right\|_\mathcal{H}
+
\frac{1}{\lambda}
\left\|
\Phi_{\mathbf{z}_1}^*
\right\|_\mathrm{op}
\left\|
L_{\mathbf{z}_1\!\mathbf{z}_2,\lambda}^{-1}
\right\|_\mathrm{op}
\big\|
\Phi_{\mathbf{z}_1}
\big\|_\mathrm{op}
\left\|
\frac{1}{n}\Phi_{\mathbf{z}_2}^*\mathds{1}_n
-
\frac{1}{m}\Phi_{\mathbf{z}_1}^*\mathds{1}_m
\right\|_\mathcal{H}\\
\lesssim~&\!\!\left(
\frac{1}{\lambda} + \frac{1}{\lambda}\frac{1}{m\lambda}\sqrt{m}
\right)
\left\|
\frac{1}{n}\Phi_{\mathbf{z}_2}^*\mathds{1}_n
-
\frac{1}{m}\Phi_{\mathbf{z}_1}^*\mathds{1}_m
\right\|_\mathcal{H}\\
\lesssim~&\frac{1}{\lambda^2\sqrt{N}}
\left\|
\frac{1}{n}\Phi_{\mathbf{z}_2}^*\mathds{1}_n
-
\frac{1}{m}\Phi_{\mathbf{z}_1}^*\mathds{1}_m
\right\|_\mathcal{H}\\
=~&\frac{1}{\lambda^2\sqrt{N}}
\left\|
\frac{1}{n}\sum_{j=1}^n \phi(z_{m+j})-\frac{1}{m}\sum_{i=1}^m \phi(z_{i})
\right\|_{\mathcal H}\\
=~&\frac{1}{\lambda^2\sqrt{N}}
\sqrt{
\frac{1}{n^2}\sum_{i=1}^n\sum_{j=1}^n k(z_{m+i},z_{m+j}) - \frac{2}{nm}\sum_{i=1}^n\sum_{j=1}^m k(z_{m+i},z_{j})+\frac{1}{m^2}\sum_{i=1}^m\sum_{j=1}^m k(z_{i},z_{j})}\\
=~&\frac{1}{\lambda^2\sqrt{N}}\,
\widehat{\mathrm{\mathrm{MMD}}}(\mathbf{z}_1,\mathbf{z}_2)\\
\lesssim~&\frac{1}{\lambda^2\sqrt{N}}
\sqrt{\frac{\ln(1/\eta)}{N}}\\
=~&\frac{1}{\lambda^2}
{\frac{\sqrt{\ln(1/\eta)}}{N}}
\end{aligned}
\end{equation}
holding with probability at least $1-\eta$ with respect to the permutation randomness (conditioned on the data). 
Here, $\widehat{\mathrm{\mathrm{MMD}}}(\mathbf{z}_1,\mathbf{z}_2)$ denotes the V-statistic estimator.
The last inequality holds by \citet[Theorem 5.1]{kim2021comparing} who provides an exponential concentration inequality for the square-rooted permuted MMD V-statistic (e.g., $\widehat{\mathrm{\mathrm{MMD}}}$).
Note that, when not working in an RKHS, similar concentration results for the permuted difference of means exist (e.g., see \citealp[Lemma 3.1]{massart1986rates} and \citealp[Equations 55 and 56]{kim2020minimax}).

The bound can be loosened to
$$
\big\|\widehat{h}_\lambda\big\|_\mathcal{H}
\lesssim
\frac{1}{\lambda^2}
{\frac{\sqrt{\ln(1/\eta)}}{N}}
\lesssim
\frac{1}{\lambda^2}
\sqrt{\frac{{\ln(1/\eta)}}{N}}.
$$
In particular, for the choice of interest $\lambda_{N\!,\theta}=N^{-1/2(\theta+1) }$ with $\theta\in(1,2]$, we get
\begin{equation*}
\big\|\widehat{h}_{\lambda_{N\!,\theta}}\big\|_{\mathcal H}
\lesssim 
N^{-\frac{\theta}{\theta+1}} 
\sqrt{\ln(1/\eta)}
\lesssim 
N^{-\frac{\theta-1}{2(\theta+1)}} 
\sqrt{\ln(1/\eta)},
\end{equation*}
equivalently
\begin{equation*}
\big\|\widehat{r}_{\lambda_{N\!,\theta}}-1\big\|_{\mathcal H}
\lesssim 
N^{-\frac{\theta}{\theta+1}} 
\sqrt{\ln(1/\eta)}
\lesssim 
N^{-\frac{\theta-1}{2(\theta+1)}} 
\sqrt{\ln(1/\eta)},
\end{equation*}
holding with probability at least $1-\eta$. 

\end{proof}

\subsection{Proof of \Cref{res:non_asymptotic_power}}
\label{proof4}

Recall that \Cref{res:non_asymptotic_power} states that, 
under \Cref{assump1,assump2},
the permutation test with level $\alpha$ and regularisation $\lambda_{N\!,\theta}=N^{-1/2(\theta+1) }$ achieves power at least $1-\beta$ for any distributions $P$ and $Q$ satisfying
$$
D_{f}(P \;\|\; Q) 
\gtrsim
\left(
N^{-\frac{\theta-1}{2(\theta+1)}} 
+ \widetilde{N}^{-1/2}
\right)
\sqrt{\max\{\ln(1/\alpha),\ln(1/\beta)\}}.
$$

\begin{proof}[Proof of \Cref{res:non_asymptotic_power}]
Following a similar argument to \citet[Appendix A.1]{schrab2025unified}, in order to prove such a result we need a concentration bound for the statistic (given in \Cref{res:fdiv_convergence}), as well as a concentration bound for the permuted statistic (derived here). 
From \Cref{lastprop}, we have that
\begin{equation}
\label{f1}
\big\|\widehat{r}_{\lambda_{N\!,\theta}}-1\big\|_{\mathcal H} \leq
\mathsf{C}_0
N^{-\frac{\theta-1}{2(\theta+1)}} 
\sqrt{\ln(1/\eta)}
\end{equation}
for some constant $\mathsf{C}_0>0$,
holding with probability at least $1-\eta$.
We then leverage the fact that 
$f'\!\big(1\big) - f^*\!\big(f'\!\big(1\big)\big)=f(1)=0$ from Fenchel–Young inequality.
Consider $f$ twice continuously differentiable on $(0,\infty)$, then $f'$ and $f^*\circ f'$ are Lipschitz by \Cref{lip,lip2}.
The permuted statistic, with permutation $\sigma$, is then
\begin{align*}
\big|
\widehat{D}_{f,{\lambda_{N\!,\theta}}}^{\sigma}
\big| 
&\leq
\Bigg|
\frac{1}{\widetilde m}\sum_{i=1}^{\widetilde m} f'\!\big(\widehat{r}_{\lambda_{N\!,\theta}}(\widetilde{Z}_i)\big) 
-
f'\!\big(1\big)
\Bigg|
+ 
\Bigg|
\frac{1}{\widetilde n}\sum_{j=1}^{\widetilde n} f^*\!\big(f'\!\big(\widehat{r}_{\lambda_{N\!,\theta}}(\widetilde{Z}_{j})\big)\big)
-
f^*\!\big(f'\!\big(1\big)\big)
\Bigg|\\
&\lesssim
\frac{1}{\widetilde m}\sum_{i=1}^{\widetilde m} 
\big|
\widehat{r}_{\lambda_{N\!,\theta}}(\widetilde{Z}_i)-1
\big|
+
\frac{1}{\widetilde n}\sum_{j=1}^{\widetilde n} 
\big|
\widehat{r}_{\lambda_{N\!,\theta}}(\widetilde{Z}_{  j})-1
\big|
\\
&\lesssim
N^{-\frac{\theta-1}{2(\theta+1)}} 
\sqrt{\ln(1/\eta)}
.
\end{align*}
For the Total-Variation and Hockey-Stick divergences, adapting the reasoning of \Cref{c1} we have
\begin{align*}
\big|
\widehat{D}_{f,{\lambda_{N\!,\theta}}}^{\sigma}
\big| 
&\leq
\Bigg|
\frac{1}{\widetilde m}\sum_{i=1}^{\widetilde m} f'\!\big(\widehat{r}_{\lambda_{N\!,\theta}}(\widetilde{Z}_i)\big) 
-
f'\!\big(1\big)
\Bigg|
+ 
\Bigg|
\frac{1}{\widetilde n}\sum_{j=1}^{\widetilde n} f^*\!\big(f'\!\big(\widehat{r}_{\lambda_{N\!,\theta}}(\widetilde{Z}_{j})\big)\big)
-
f^*\!\big(f'\!\big(1\big)\big)
\Bigg|\\
&\leq
\frac{1}{\widetilde m}\sum_{i=1}^{\widetilde m} 
\Bigg|
f'\!\big(\widehat{r}_{\lambda_{N\!,\theta}}(\widetilde{Z}_i)\big) 
-
f'\!\big(1\big)
\Bigg|
+ 
\frac{1}{\widetilde n}\sum_{j=1}^{\widetilde n} 
\Bigg|
f^*\!\big(f'\!\big(\widehat{r}_{\lambda_{N\!,\theta}}(\widetilde{Z}_{j})\big)\big)
-
f^*\!\big(f'\!\big(1\big)\big)
\Bigg|.
\end{align*}
Since, for the Total-Variation, we have $f^*(u)=u$, and for the Hockey-Stick, we have $f^*(u)=\gamma u$, both terms can be treated similarly.
With a similar reasoning to \Cref{c1}, and using \Cref{f1}, we obtain
\begin{align*}
\frac{1}{\widetilde m}\sum_{i=1}^{\widetilde m} 
\Bigg|
f'\!\big(\widehat{r}_{\lambda_{N\!,\theta}}(\widetilde{Z}_i)\big) 
-
f'\!\left(1\right)
\Bigg|
&= 
\frac{1}{\widetilde m}\sum_{i=1}^{\widetilde m} 
\mathds{1}\!\left(
    f'\!\big(\widehat{r}_{\lambda_{N\!,\theta}}(\widetilde{Z}_i)\big) 
\neq
f'\!\left(1\right)
\right)\\ 
&=
\frac{1}{\widetilde m}\sum_{i=1}^{\widetilde m} 
\mathds{1}\!\left(
\gamma\in\big(
    \min\left\{\widehat{r}_{\lambda_{N\!,\theta}}(\widetilde{Z}_i),1\right\}
,
\max\left\{\widehat{r}_{\lambda_{N\!,\theta}}(\widetilde{Z}_i),1\right\}
\big]
\right)\\ 
&\leq
\frac{1}{\widetilde m}\sum_{i=1}^{\widetilde m} 
\mathds{1}\!\left(
\big|\widehat{r}_{\lambda_{N\!,\theta}}(\widetilde{Z}_i)-\gamma\big|
\leq 
\mathsf{C}_0
N^{-\frac{\theta-1}{2(\theta+1)}} 
\sqrt{\ln(1/\eta)}
\right)\\ 
&\leq
\mathbb{P}_{\sigma,P,Q}\!\left(
    \big|\widehat{r}_{\lambda_{N\!,\theta}}(\widetilde{Z}_{\sigma(1)})-\gamma\big|
\leq 
\mathsf{C}_0
N^{-\frac{\theta-1}{2(\theta+1)}} 
\sqrt{\ln(1/\eta)}
\right) + \sqrt{\frac{\ln(1/\eta)}{\widetilde N}}\\
&\lesssim 2G
\mathsf{C}_0
N^{-\frac{\theta-1}{2(\theta+1)}} 
\sqrt{\ln(1/\eta)}
 + \sqrt{\frac{\ln(1/\eta)}{\widetilde N}}
\end{align*}
since from the fact that, for $X\sim P$ and $Y\sim Q$, the densities of $\widehat{r}_{\lambda_{N\!,\theta}}(X)$ and of $\widehat{r}_{\lambda_{N\!,\theta}}(Y)$ exist and are bounded on $[\gamma/2,3\gamma/2]$ by some $G>0$, we can deduce that the density of $\widehat{r}_{\lambda_{N\!,\theta}}(\widetilde{z}_{\sigma(1)})$ also exists and is bounded by $G$.
We conclude that
\begin{equation}
\label{phd}
\big|
\widehat{D}_{f,\lambda_{N\!,\theta}}^{\sigma}
\big| 
\lesssim
\left(
N^{-\frac{\theta-1}{2(\theta+1)}} 
+ \widetilde{N}^{-1/2}
\right)
\sqrt{\ln(1/\eta)}
\end{equation}
holding with probability at least $1-\eta$,
from which we can deduce bound on the quantile obtained via permutations
\begin{equation}
\label{g1}
\widehat{q}_{1-\alpha}
\lesssim
\left(
N^{-\frac{\theta-1}{2(\theta+1)}} 
+ \widetilde{N}^{-1/2}
\right)
\sqrt{\ln(1/\alpha)},
\end{equation}
holding with probability $1-\beta/2$ provided that the number of permutations is greater than $6\ln(2/\beta)/\alpha$ as shown by \citet[Lemma 21 and Equation 63]{kim2023differentially}. 
We also recall from \Cref{res:fdiv_convergence} that 
\begin{equation}
\label{g2}
\big|
\widehat{D}_{f,\lambda_{N\!,\theta}}-D_f
\big|
\lesssim 
\left(
N^{-\frac{\theta-1}{2(\theta+1)}} 
+ \widetilde{N}^{-1/2}
\right)
\sqrt{\ln(1/\beta)}
\end{equation}
 holding with probability at least $1-\beta/2$. 
Now, following \citet[Appendix A.1]{schrab2025unified}, and using \Cref{g1,g2}, the type II error of the test is
\begin{align*}
&\mathbb{P}\left(
\widehat{D}_{f,\lambda_{N\!,\theta}}
\leq
\widehat{q}_{1-\alpha}
\right)\\ 
\leq~
&\mathbb{P}\left(
{D}_{f}
\lesssim
\widehat{q}_{1-\alpha}
+
\left(
N^{-\frac{\theta-1}{2(\theta+1)}} 
+ \widetilde{N}^{-1/2}
\right)
\sqrt{\ln(1/\beta)}
\right)+\frac{\beta}{2}\\ 
\leq~
&\mathbb{P}\left(
{D}_{f}
\lesssim
\left(
N^{-\frac{\theta-1}{2(\theta+1)}} 
+ \widetilde{N}^{-1/2}
\right)
\sqrt{\ln(1/\alpha)}
+
\left(
N^{-\frac{\theta-1}{2(\theta+1)}} 
+ \widetilde{N}^{-1/2}
\right)
\sqrt{\ln(1/\beta)}
\right)+\beta\\ 
\leq~
&\mathbb{P}\left(
{D}_{f}
\lesssim
\left(
N^{-\frac{\theta-1}{2(\theta+1)}} 
+ \widetilde{N}^{-1/2}
\right)
\sqrt{\max\{\ln(1/\alpha),\ln(1/\beta)\}}\right)+\beta.
\end{align*}
We conclude that if 
$$
{D}_{f}
\gtrsim
\left(
N^{-\frac{\theta-1}{2(\theta+1)}} 
+ \widetilde{N}^{-1/2}
\right)
\sqrt{\max\{\ln(1/\alpha),\ln(1/\beta)\}}
$$
then the type II error is controlled as 
$
\mathbb{P}\left(
\widehat{D}_{f,\lambda_{N\!,\theta}}
\leq
\widehat{q}_{1-\alpha}
\right)
\leq \beta.
$
This concludes the proof since
$\sqrt{\max\{\ln(1/\alpha),\ln(1/\beta)\}}=
\sqrt{\ln({1}/{\min\{\alpha,\beta\}})}$.

\end{proof}

\section{Power Guarantees for the DrMMD Permutation Test}
\label{drmmd_power}

Recall that, for the Pearson $\chi^2$-divergence, we have $f(t)=(t-1)^2$, $f^*(u) = u+\frac{u^2}{4}$, and $f'\!(r) = 2(r-1)$, it is defined as
\begin{align*}
    D_{\chi^2}(P \;\|\; Q) 
    = \sup_{h\colon\!\mathcal{Z}\to\mathbb{R}} \ 
\int h \,\mathrm{d}P
-\int \left(h+\frac{h^2}{4}\right) \,\mathrm{d}Q.
\end{align*}
The DrMMD \citep{chen2024regularized} is a regularised $\chi^2$-divergence estimator in some RKHS $\mathcal H$, defined as
$$
\sup_{h\in\mathcal H}\ 
\int h \,\mathrm{d}P
-\int \left(h+\frac{h^2}{4}\right) \,\mathrm{d}Q
- \frac{\lambda}{4}\|h\|_{\mathcal H}^2.
$$
with optimal witness function
$$
h_\lambda = 2\big(\Sigma_Q+\lambda I\big)^{-1}(\mu_P-\mu_Q).
$$
The sample-based version of the witness function is
$$
\widehat{h}_\lambda 
= 2\big(\Sigma_{\widehat Q}+\lambda I\big)^{-1}(\mu_{\widehat P}-\mu_{\widehat Q}) 
= f'(\widehat{r}_\lambda)
= 2 (\widehat{r}_\lambda - 1)
$$
where a closed-form expression \citep[Proposition 6.1]{chen2024regularized} for $\widehat{r}_\lambda$ is presented in \Cref{eq:r_lambda}.

Then, our $\chi^2$ estimator of \Cref{sec:consistent-estimation} is 
\begin{equation*}
\widehat{D}_{\chi^2\!,\lambda} =
\frac{1}{\widetilde m}\sum_{i=1}^{\widetilde m} \widehat{h}_\lambda(\widetilde{X}_i) - \frac{1}{\widetilde n}\sum_{j=1}^{\widetilde n} \left(\widehat{h}_\lambda(\widetilde{Y}_j)+\frac{\widehat{h}_\lambda(\widetilde{Y}_j)^2}{4}\right)
\end{equation*}
while a DrMMD estimator would be\footnote{Note that sample splitting is not necessary to construct a DrMMD estimator, see \citet{chen2024regularized}.}
\begin{equation*}
    \widehat{D}_{\chi^2\!,\lambda}^{\mathrm{reg}} =
\frac{1}{\widetilde m}\sum_{i=1}^{\widetilde m} \widehat{h}_\lambda(\widetilde{X}_i) - \frac{1}{\widetilde n}\sum_{j=1}^{\widetilde n} \left(\widehat{h}_\lambda(\widetilde{Y}_j)+\frac{\widehat{h}_\lambda(\widetilde{Y}_j)^2}{4}\right) - \frac{\lambda}{4}\big\|\widehat{h}_\lambda\big\|_{\mathcal H}^2
\end{equation*}
where $\big\|\widehat{h}_\lambda\big\|_{\mathcal H}$ can be computed explicitly in closed-form from \Cref{h_lambda expression a} \citep[Proposition 6.1]{chen2024regularized}.
In particular, we have
\begin{equation}
\label{reg_relation}
\widehat{D}_{\chi^2\!,\lambda}^{\mathrm{reg}}
=
\widehat{D}_{\chi^2\!,\lambda}
- \frac{\lambda}{4}\big\|\widehat{h}_\lambda\big\|_{\mathcal H}^2.
\end{equation}
We now show that the permutation DrMMD test satisfies the same non-asymptotic power guarantee of \Cref{res:non_asymptotic_power} as our proposed $\chi^2$-divergence test.

\begin{theorem}[Non-asymptotic Power of DrMMD Permutation Test]
\label{res:non_asymptotic_power_drmmd}
Under \Cref{assump1,assump2}, 
the permutation DrMMD test, based on the estimator $\widehat{D}_{\chi^2\!,\lambda_{N\!,\theta}}^{\mathrm{reg}}\!$ with regularisation $\lambda_{N\!,\theta}=N^{-1/2(\theta+1) }$ for $\theta\in(1,2]$ and level $\alpha$, achieves power at least $1-\beta$ for any distributions $P$ and $Q$ satisfying
$$
D_{\chi^2}(P \;\|\; Q) 
\gtrsim
\left(
N^{-\frac{\theta-1}{2(\theta+1)}} 
+ \widetilde{N}^{-1/2}
\right)
\sqrt{\max\!\left\{\ln\!\left(\frac{1}{\alpha}\right),\ln\!\left(\frac{1}{\beta}\right)\right\}}.
$$
\end{theorem}

\begin{proof}
Following the proof of \Cref{res:non_asymptotic_power} in \Cref{proof4},
we see that in order for the same reasoning to hold, we need to have a concentration bound for the regularised statistic
\begin{equation}
\label{need1}
\big|
\widehat{D}_{{\chi^2}\!,\lambda_{N\!,\theta}}^\mathrm{reg}-D_{\chi^2}
\big|
\lesssim 
\left(
N^{-\frac{\theta-1}{2(\theta+1)}} 
+ \widetilde{N}^{-1/2}
\right)
\sqrt{\ln(1/\eta)},
\end{equation}
as well as a concentration bound for the permuted regularised statistic
\begin{equation}
\label{need2}
\big|
\widehat{D}_{{\chi^2}\!,\lambda_{N\!,\theta}}^{\mathrm{reg}, \sigma}
\big| 
\lesssim
\left(
N^{-\frac{\theta-1}{2(\theta+1)}} 
+ \widetilde{N}^{-1/2}
\right)
\sqrt{\ln(1/\eta)},
\end{equation}
each of these holding with probability at least $1-\eta$.
Provided that both \Cref{need1,need2} hold, then the same reasoning as in \Cref{proof4} leads to the desired power guarantee.
Using the non-regularised results of \Cref{res:fdiv_convergence} and of \Cref{phd} in \Cref{proof4}, together with the relation of \Cref{reg_relation}, we deduce that it is sufficient to prove that
\begin{equation}
\label{need3}
\lambda_{N\!,\theta}
\big\|\widehat{h}_{\lambda_{N\!,\theta}}^\sigma\big\|_{\mathcal H}^2
\lesssim
N^{-\frac{\theta-1}{2(\theta+1)}} 
\end{equation}
for the case of $\sigma=\mathrm{Id}$, solving \Cref{need1}, and for the case of any permutation $\sigma$, solving \Cref{need2}.

From \Cref{h_lambda norm bound}, we have
$$
\big\|\widehat{h}_\lambda^\sigma\big\|_\mathcal{H}
\lesssim
\frac{1}{\lambda^2\sqrt{N}}\,
\widehat{\mathrm{\mathrm{MMD}}}_\sigma
\lesssim
\frac{1}{\lambda^2\sqrt{N}}.
$$
Since we want the bound to hold both for $\sigma=\mathrm{Id}$ and for any permutation $\sigma$, we cannot use a concentration inequality of the permuted MMD around 0 as in \Cref{h_lambda norm bound}, since when $\sigma=\mathrm{Id}$ the MMD does not necessarily concentrate around 0. Hence, instead, we directly bound the empirical (permuted) MMD term by $\sqrt{2K}$ assuming the kernel is bounded everywhere by $K>0$.
We deduce that
$$
\lambda\,\big\|\widehat{h}_\lambda^\sigma\big\|_\mathcal{H}^2
\lesssim
\frac{1}{\lambda^3{N}}.
$$
Substituting $\lambda_{N\!,\theta}=N^{-1/2(\theta+1) }$ for $\theta\in(1,2]$, we obtain
$$
\lambda_{N\!,\theta}
\big\|\widehat{h}_{\lambda_{N\!,\theta}}^\sigma\big\|_\mathcal{H}^2
\lesssim
N^{-\frac{2\theta-1}{2(\theta+1)}}
\lesssim
N^{-\frac{\theta-1}{2(\theta+1)}}.
$$
Hence, \Cref{need3} is satisfied, which concludes the proof.
\end{proof}

Similarly, the asymptotic power guarantee of \Cref{res:asymptotic_power} also holds for the permutation DrMMD test.

\begin{theorem}[Asymptotic Power]
\label{res:asymptotic_power2}
Under \Cref{assump1,assump2}, 
the permutation DrMMD test, based on the estimator $\widehat{D}_{\chi^2\!,\lambda_{N\!,\theta}}^{\mathrm{reg}}\!$ with regularisation $\lambda_{N\!,\theta}=N^{-1/2(\theta+1) }$ for $\theta\in(1,2]$,
is consistent in that its power converges to 1 as the sample sizes tend to infinity.
For any fixed $P\neq Q$, we have
$$
\mathbb{P}(\mathrm{reject }~ H_0)\to 1
\textrm{ as } N,\widetilde{N}\to\infty.
$$
\end{theorem}

\begin{proof}
The reasoning of the proof of \Cref{res:asymptotic_power} in \Cref{proof3} applies in this setting too. 
For clarity, in this proof, we simply write $\lambda=\lambda_{N\!,\theta}=N^{-1/2(\theta+1) }$.
As explained in \Cref{proof3}, to prove consistency \citep[Lemma 8 and Appendix E.6]{kim2023differentially}, it suffices to show that
$
\widehat{D}_{\chi^2\!,\lambda}^{\mathrm{reg}}
=
\widehat{D}_{\chi^2\!,\lambda}
- \frac{\lambda}{4}\big\|\widehat{h}_\lambda\big\|_{\mathcal H}^2
$
converges to 
${D}_{\chi^2}(P\;\|\;Q)>0$ when $P\neq Q$, and that the permuted statistic 
$
\widehat{D}_{\chi^2\!,\lambda}^{\mathrm{reg},\sigma}
=
\widehat{D}_{\chi^2\!,\lambda}^\sigma
- \frac{\lambda}{4}\big\|\widehat{h}_\lambda^\sigma\big\|_{\mathcal H}^2
$
converges to zero.

In \Cref{proof3}, we have have shown that
$
\widehat{D}_{\chi^2\!,\lambda}
$
converges to 
${D}_{\chi^2}(P\;\|\;Q)>0$ when $P\neq Q$, and that the permuted statistic 
$
\widehat{D}_{\chi^2\!,\lambda}^\sigma
$
converges to zero.
From the derivation of \Cref{need3} in the proof of \Cref{res:non_asymptotic_power_drmmd}, we also know that 
$
\lambda\,\big\|\widehat{h}_\lambda^\sigma\big\|_\mathcal{H}^2
$
converges to zero.
Combining these results concludes the proof.

\end{proof}

\end{document}